\documentclass{article}

\usepackage{arxiv}

\usepackage{array}
\newcolumntype{C}[1]{>{\centering\arraybackslash}p{#1}}

\usepackage[utf8]{inputenc} % allow utf-8 input
\usepackage[T1]{fontenc}    % use 8-bit T1 fonts
\usepackage{hyperref}       % hyperlinks
\usepackage{url}            % simple URL typesetting
\usepackage{booktabs}       % professional-quality tables
\usepackage{amsfonts}       % blackboard math symbols
\usepackage{nicefrac}       % compact symbols for 1/2, etc.
\usepackage{graphicx}
\usepackage{natbib}
\usepackage{doi}
\usepackage{tabularx}
\usepackage{wrapfig}
\usepackage{needspace}

\usepackage{enumitem}
\usepackage{amsmath}
\usepackage{amssymb}
\usepackage{bm}
\usepackage{amsthm}
\usepackage{multirow}
\usepackage{subcaption}
\usepackage{algorithm}
\usepackage{algorithmic}

\newtheorem{assumption}{Assumption}
\newtheorem{theorem}{Theorem}
\newtheorem{lemma}{Lemma}
\newtheorem{proposition}{Proposition}
\newtheorem{remark}{Remark}

\title{FedPA-LoRA: Product-Aligned Framework for Mitigating Aggregation and Initialization Errors in Heterogeneous Federated LoRA}

\author{%
\begin{center}
\begin{tabular}{@{}C{0.30\textwidth}C{0.30\textwidth}C{0.30\textwidth}@{}}
\begin{tabular}[t]{c}
\textbf{Juseok~Jeon}\\
DGIST\\
Daegu, Republic of Korea\\
\texttt{jooseuk@dgist.ac.kr}
\end{tabular}
&
\begin{tabular}[t]{c}
\textbf{Ramy E.~Ali}\textsuperscript{\(\ddagger\)}\\
Samsung\\
San Diego, USA\\
\texttt{ramy.ali@ieee.org}
\end{tabular}
&
\begin{tabular}[t]{c}
\textbf{Doyun~Kwon}\\
DGIST\\
Daegu, Republic of Korea\\
\texttt{ehdbs810@dgist.ac.kr}
\end{tabular}
\\
\noalign{\vskip 1.5em}
\begin{tabular}[t]{c}
\textbf{Myungbeom~Her}\\
DGIST\\
Daegu, Republic of Korea\\
\texttt{myungbeom.her@dgist.ac.kr}
\end{tabular}
&
\begin{tabular}[t]{c}
\textbf{Jinhwi~Kim}\\
DGIST\\
Daegu, Republic of Korea\\
\texttt{kjh2159@dgist.ac.kr}
\end{tabular}
&
\begin{tabular}[t]{c}
\textbf{Jinhyun~So}\textsuperscript{\(\dagger\)}\\
DGIST\\
Daegu, Republic of Korea\\
\texttt{jinhyun@dgist.ac.kr}
\end{tabular}
\end{tabular}
\end{center}%
\thanks{\textsuperscript{\(\dagger\)} Corresponding author.\protect\\
\textsuperscript{\(\ddagger\)} The opinions expressed in this work are solely those of the author and do not represent the views of Samsung.}%
}

\date{}

\renewcommand{\shorttitle}{FedPA-LoRA}

\hypersetup{
pdftitle={FedPA-LoRA: Product-Aligned Framework for Mitigating Aggregation and Initialization Errors in Heterogeneous Federated LoRA},
pdfsubject={Machine Learning, Federated Learning},
pdfauthor={Juseok Jeon, Ramy E. Ali, Doyun Kwon, Myungbeom Her, Jinhwi Kim, Jinhyun So},
pdfkeywords={Federated Learning, Large Language Models, LoRA, Parameter-Efficient Fine-Tuning},
}

\begin{document}
\maketitle
\setcounter{footnote}{0}

\begin{abstract}
Low-Rank Adaptation (LoRA) enables efficient federated fine-tuning of large language models, but its factorization creates a tension between aggregation fidelity and local optimization continuity. Factor-wise aggregation incurs aggregation mismatch but better preserves factor continuity, whereas product-space reconstruction reduces this mismatch at the cost of greater factor-level initialization mismatch. We propose FedPA-LoRA, a product-aligned framework that addresses both mismatches and establishes convergence bounds for homogeneous and heterogeneous ranks. Clients preserve local factors across rounds to eliminate initialization mismatch and align their products with rank-specific global references to reduce client drift. The server aggregates heterogeneous-rank updates in product space to mitigate aggregation mismatch and efficiently reconstructs a rank-constrained global adapter without forming the dense aggregate. By decoupling local and reference ranks, FedPA-LoRA supports client-specific computation and communication budgets. Experiments on natural language understanding and generation tasks show consistent gains under data and resource heterogeneity, including a 6.82 percentage-point improvement in average GLUE accuracy under heterogeneous ranks. The code is available at \href{https://github.com/Juseok-Jeon/FedPA-LoRA.git}{FedPA-LoRA-Code}.
\end{abstract}

% keywords can be removed

\section{Introduction}

Large language models (LLMs) owe much of their downstream performance to the scale and diversity of the data used to adapt them \citep{kaplan2020scaling}, yet a large share of that data sits on individual devices or within institutions that cannot release it externally for privacy or regulatory reasons \citep{rieke2020future}. Federated learning (FL) addresses this constraint by coordinating training across participants without transferring raw data \citep{mcmahan2017communication}. Even so, full-parameter fine-tuning remains costly in federated settings, as resource-constrained clients cannot feasibly compute and communicate model-scale updates in every round \citep{zhang2023fedpetuning,xu2024fwdllm}. Parameter-efficient fine-tuning (PEFT) alleviates these costs by training only a compact set of additional parameters while keeping the pretrained backbone frozen \citep{lialin2023scaling,han2024parameterefficient}. A prominent example is Low-Rank Adaptation (LoRA), which parameterizes each weight update as $\Delta \bm W=\bm B\bm A$, where $\bm B$ and $\bm A$ are low-rank factors \citep{hu2022lora}. By substantially reducing the number of trainable and communicated parameters, LoRA has become a practical approach for federated fine-tuning of large models \citep{babakniya2023slora,kuang2024federatedscope}.

Despite these advantages, LoRA's factorized parameterization complicates federated aggregation. Factor-wise averaging preserves a compact low-rank adapter, but the product of the averaged factors generally differs from the average of the locally optimized updates, resulting in an aggregation mismatch \citep{bai2024federated,chen2025robust,zhang2026fedrotlora}. Product-space aggregation avoids the cross-client terms introduced by factor-wise averaging by aggregating the local updates before low-rank reconstruction. However, newly reconstructed global factors may differ substantially from locally optimized factors, and repeatedly replacing local factors with them can disrupt factor-level continuity across rounds. Existing approaches, therefore, face a tension between global aggregation fidelity and local optimization continuity.

Resource heterogeneity presents another challenge for federated LoRA. Clients with different system capabilities may require different LoRA ranks, yielding incompatible factor dimensions and necessitating rank-aware aggregation and client-specific redistribution \citep{cho-etal-2024-heterogeneous,bai2024federated,zhang2026heterogeneous}. Federated LoRA must, therefore, support heterogeneous adaptation capacities while maintaining global knowledge sharing and stable local optimization.

\textbf{Contributions}. We propose FedPA-LoRA to jointly address initialization and aggregation mismatches through local factor preservation and product-space aggregation. Product-guided alignment further mitigates client drift under data heterogeneity, while decoupled ranks $r_i$ and $R_i$ allow independent control of computation/uplink and downlink budgets. Our contributions are as follows.

\begin{itemize}
    \item We propose FedPA-LoRA to mitigate aggregation and factor-level initialization mismatches under data heterogeneity through product-space aggregation and local factor preservation, while supporting client-specific computation and communication budgets.
    
    \item We establish a standard non-convex convergence rate of $\mathcal O((T\tau)^{-1/2})$ in terms of averaged factor-space stationarity under homogeneous ranks and extend the analysis to heterogeneous ranks with an explicit reference-truncation error term.

    \item We demonstrate consistent gains across LLM fine-tuning tasks under data and resource heterogeneity, with up to a $6.82$ percentage-point improvement over state-of-the-art baselines.
\end{itemize}
\section{Related Work}
% We review representative federated LoRA methods, with additional work discussed in Appendix~\ref{appendix:additional_related_work}.
We now review the federated LoRA methods most closely related to FedPA-LoRA, and defer discussion of additional related methods to  Appendix \ref{appendix:additional_related_work}.

\subsection{Federated LoRA}

Federated LoRA methods primarily differ in how they aggregate and redistribute the two low-rank factors. \textbf{\emph{(i)} Factor-wise aggregation} preserves a fixed-rank global adapter by aggregating the factors independently. FedIT \citep{10447454} directly averages $\bm B$ and $\bm A$. LoRA-FAIR \citep{Bian_2025_ICCV} mitigates aggregation bias and initialization drift through iterative server-side residual correction, while FedRot-LoRA \citep{zhang2026fedrotlora} aligns client factors before averaging to reduce cross-client subspace misalignment. Despite these refinements, neither method guarantees an optimal rank-constrained approximation of the aggregated local updates or preserves client-specific factors across rounds. \textbf{\emph{(ii)} Product-space reconstruction} approximates aggregated local updates with a low-rank global adapter. FlexLoRA \citep{bai2024federated} applies truncated singular value decomposition (SVD) to the aggregated product, while FedSRD \citep{10.1145/3774904.3792144} reconstructs the global update from sparsified client updates before decomposition. Although these methods improve aggregation fidelity, they require dense reconstruction and high-dimensional decomposition, and redistributing the reconstructed factors can disrupt local optimization continuity. FedPA-LoRA instead uses reduced QR and core SVD on the concatenated factors, and further supports randomized reconstruction that trades exact rank-constrained reconstruction for lower server-side complexity. \textbf{\emph{(iii)} Partial training or sharing} simplifies aggregation by restricting which factors are optimized or globally shared. FFA-LoRA \citep{ICLR2024_4e243e95} freezes $\bm A$ and trains only $\bm B$, whereas RoLoRA \citep{chen2025robust} alternates the trainable factor across communication rounds. Although these approaches make aggregation linear, freezing one factor restricts joint optimization of the two factors and may slow convergence or limit adaptation capacity. FedSA-LoRA \citep{guo2025selective} instead trains both factors locally but aggregates only $\bm A$, leading to an incomplete global adapter.

\subsection{Heterogeneous Federated LoRA}

Resource heterogeneity motivates assigning client-specific LoRA ranks according to local computation and communication budgets. HetLoRA \citep{cho-etal-2024-heterogeneous} zero-pads heterogeneous factors to a common rank for factor-wise aggregation and truncates the global factors for each client. However, independently averaging the padded factors still introduces aggregation mismatch. FlexLoRA \citep{bai2024federated} aggregates local products and applies SVD-based reconstruction to redistribute rank-specific factors. FLoRA \citep{3737916.3738624} instead exactly aggregates heterogeneous-rank updates by stacking local factors. However, its global rank grows with the sum of client ranks, increasing communication and adapter size, and its per-round factor reinitialization disrupts optimization continuity. Ravan \citep{raje2025ravan} employs multiple fixed-basis heads and trains resource-dependent intermediate matrices, restricting local adaptation to subspaces determined at initialization. Fed-PLoRA \citep{zhang2026heterogeneous} decomposes LoRA into parallel rank-one modules and allows each client to train a resource-dependent subset through Select-N-Fold, but distributes the complete collection of global modules to every client.

% \textbf{Our work}. FedPA-LoRA preserves local factors across rounds to maintain optimization continuity and aggregates their products to mitigate aggregation mismatch. In contrast to FedProx \citep{MLSYS2020_1f5fe839}, it regularizes the LoRA product toward the global reference rather than constraining parameter deviation. This product-space design further decouples the local computation/uplink rank from the downlink reference rank, enabling their independent control unlike existing federated LoRA methods.
\textbf{Our work}. FedPA-LoRA preserves local factors and aggregates their products to mitigate both mismatches. While FedProx \citep{MLSYS2020_1f5fe839} reinitializes local models from the global model and penalizes parameter deviations, FedPA-LoRA regularizes local LoRA products toward rank-specific global references without overwriting their factors. Unlike existing federated LoRA methods, this design independently controls local computation/uplink and downlink reference ranks.
\section{Motivating Example}
\label{sec:motivating_example}

Federated LoRA faces a fundamental tension between accurate aggregation of local updates and continuity of locally optimized factors. We illustrate this trade-off using FedIT, FlexLoRA, and FedRot-LoRA, and show how FedPA-LoRA addresses both aggregation and factor-level initialization mismatches.

\subsection{Aggregation Mismatch}

At round $t$, client $i$ obtains a local LoRA update $\Delta\bm W_i^{(t)}=\bm B_i^{(t)}\bm A_i^{(t)}$, where $\bm B_i^{(t)}\in\mathbb R^{d\times r}$, $\bm A_i^{(t)}\in\mathbb R^{r\times d}$, and $r\ll d$. The ideal global update is the average of these local products, $\Delta\bm W_{\mathrm{ideal}}^{(t)}=\frac{1}{N}\sum_{i=1}^{N}\bm B_i^{(t)}\bm A_i^{(t)}$, with rank up to $\min(Nr,d)$.

FedIT independently averages the local factors to obtain a rank-$r$ global adapter. In general, however, $\bm B_g^{(t)}\bm A_g^{(t)}\neq\Delta\bm W_{\mathrm{ideal}}^{(t)}$ because factor-wise averaging introduces cross-client terms $\bm B_i^{(t)}\bm A_j^{(t)}$ for $i\neq j$. We refer to this discrepancy as the \emph{aggregation mismatch}.

FedRot-LoRA reduces aggregation mismatch by aligning local factors with an orthogonal rotation $\bm R_i^{(t)}$. It transforms $\widetilde{\bm B}_i^{(t)}=\bm B_i^{(t)}\bm R_i^{(t)}$ and $\widetilde{\bm A}_i^{(t)}=(\bm R_i^{(t)})^\top\bm A_i^{(t)}$, where $(\bm R_i^{(t)})^\top\bm R_i^{(t)}=\bm I_r$. This preserves $\widetilde{\bm B}_i^{(t)}\widetilde{\bm A}_i^{(t)}=\bm B_i^{(t)}\bm A_i^{(t)}$, but averaging the aligned factors still does not generally recover $\Delta\bm W_{\mathrm{ideal}}^{(t)}$ exactly.

\subsection{Factor-Level Initialization Mismatch}

\begin{figure}[t]
\centering

\begin{subfigure}[t]{0.325\linewidth}
    \centering
    \includegraphics[
        width=\linewidth,
        trim=5 10 5 15,
        clip
    ]{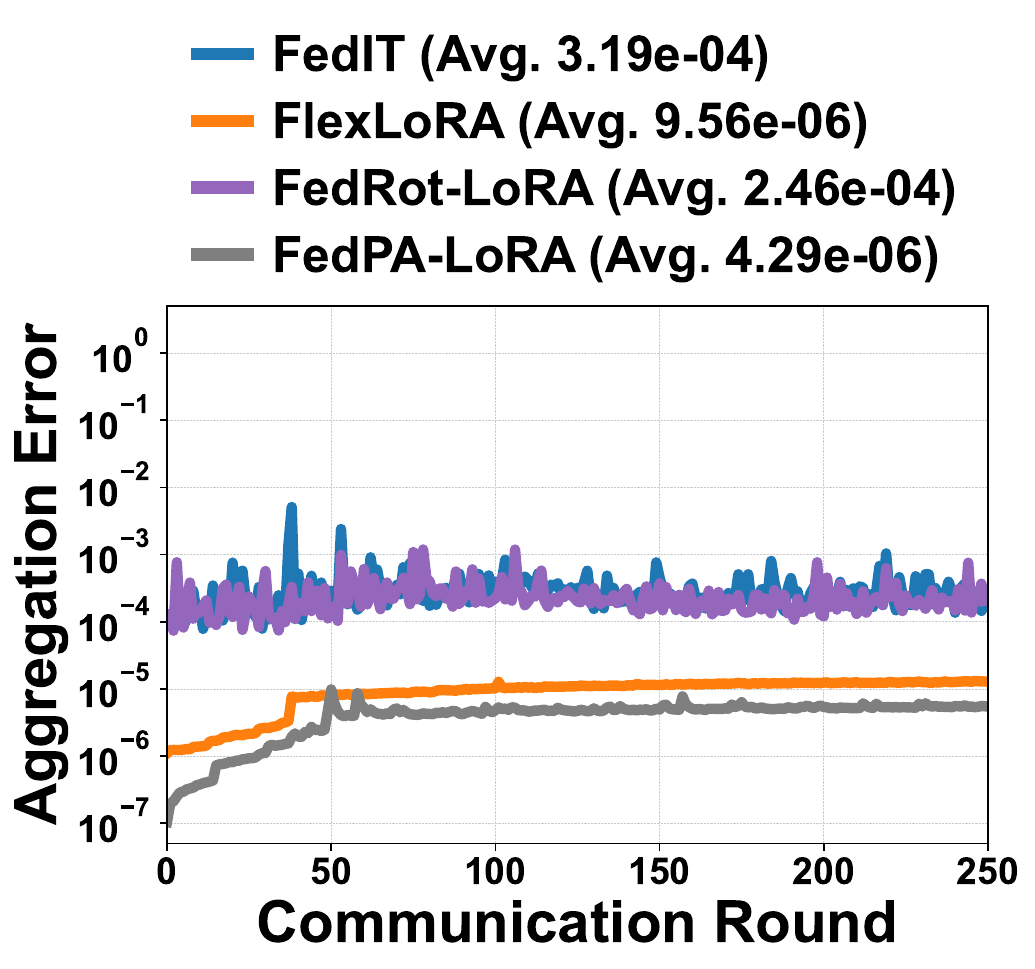}
    \captionsetup{justification=centering}
    \caption{Aggregation error}
    \label{fig:agg_error}
\end{subfigure}
\hfill
\begin{subfigure}[t]{0.325\linewidth}
    \centering
    \includegraphics[
        width=\linewidth,
        trim=5 10 5 15,
        clip
    ]{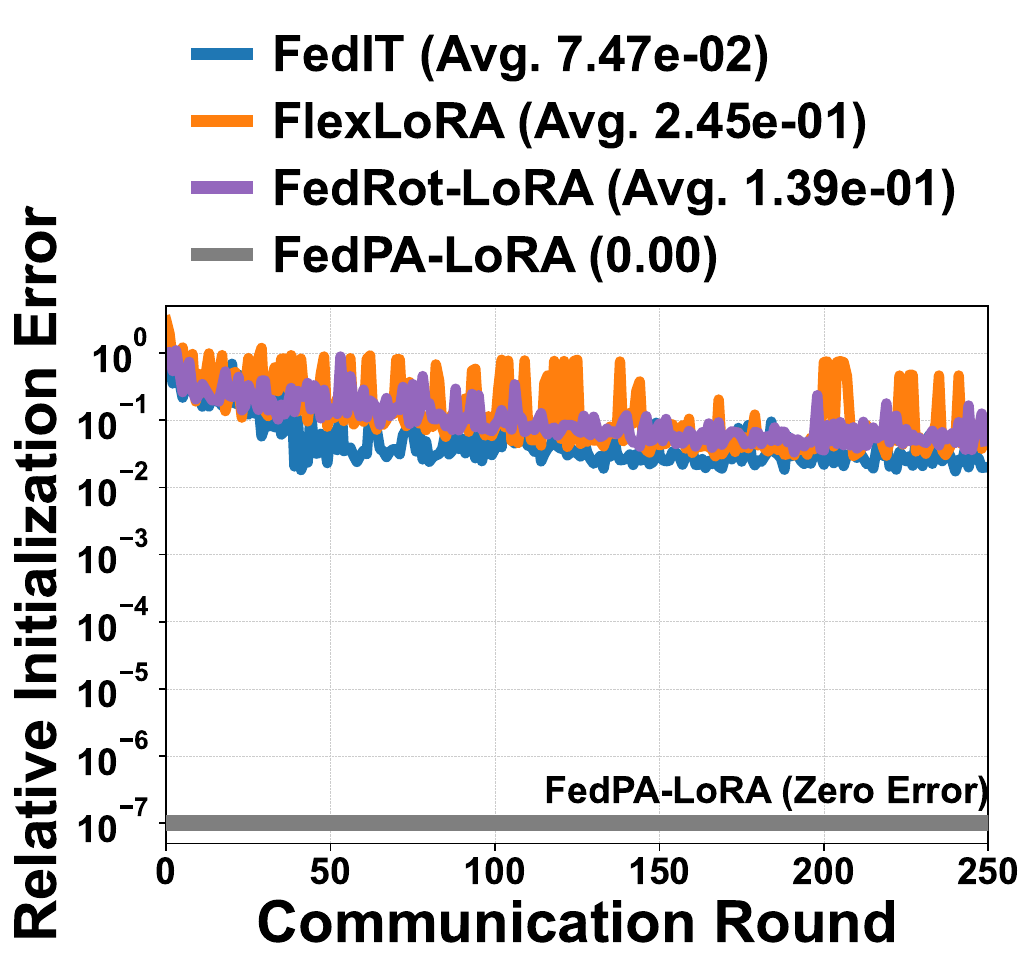}
    \captionsetup{justification=centering}
    \caption{Initialization error of $\bm B$}
    \label{fig:b_error}
\end{subfigure}
\hfill
\begin{subfigure}[t]{0.325\linewidth}
    \centering
    \includegraphics[
        width=\linewidth,
        trim=5 10 5 15,
        clip
    ]{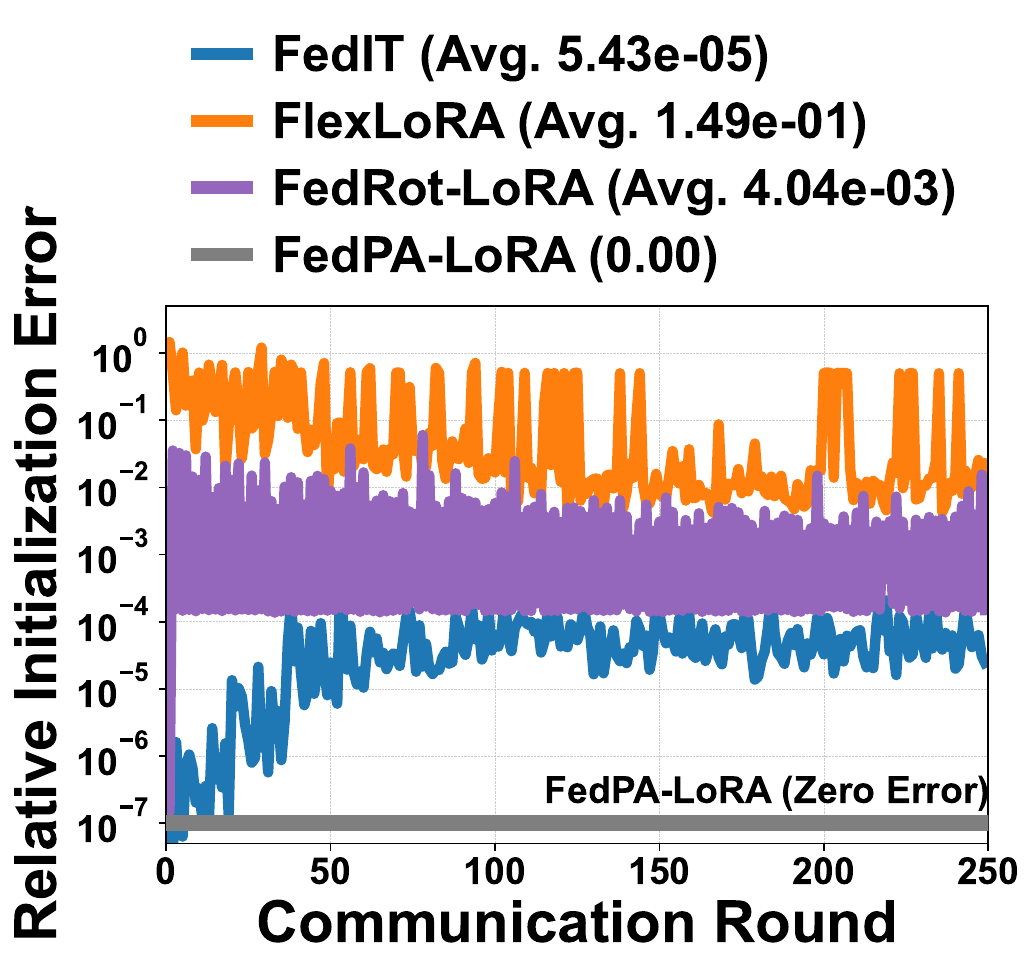}
    \captionsetup{justification=centering}
    \caption{Initialization error of $\bm A$}
    \label{fig:a_error}
\end{subfigure}
\caption{
Aggregation error (a) and relative initialization errors of $\bm B$ (b) and $\bm A$ (c) for RoBERTa-Large's layer-0 query projection on MNLI (three clients). All y-axes are logarithmic, with FedPA-LoRA's zero initialization errors shown at the plotting floor.
}
\label{fig:error_results}
\end{figure}

FlexLoRA avoids factor-wise aggregation by applying rank-$r$ truncated SVD to the ideal update, $\Delta\bm W_{\mathrm{ideal}}^{(t)}\approx\bm U_r^{(t)}\bm\Sigma_r^{(t)}(\bm V_r^{(t)})^\top$. It reconstructs the global factors as $\bm B_g^{(t)}=\bm U_r^{(t)}\bm\Sigma_r^{(t)}$ and $\bm A_g^{(t)}=(\bm V_r^{(t)})^\top$, minimizing aggregation mismatch among rank-$r$ updates under the Frobenius norm. However, low-rank factorization is non-unique: for any invertible $\bm Q$, $(\bm B_g^{(t)}\bm Q)(\bm Q^{-1}\bm A_g^{(t)})=\bm B_g^{(t)}\bm A_g^{(t)}$. Thus, the reconstructed global factors can differ from the locally optimized factors, resulting in a \emph{factor-level initialization mismatch}.

FedRot-LoRA provides an intermediate compromise. Its rotational alignment reduces aggregation mismatch but alters the local factors, while clients are reinitialized with aggregated global factors in the next round. It therefore preserves neither the exact average of local products nor the locally optimized factors.

\subsection{Empirical Observation}

We evaluate this trade-off using RoBERTa-Large \citep{liu2019roberta} on MNLI with $N=3$, $r=4$, and Dirichlet concentration $\beta=0.5$. We measure aggregation mismatch as
\begin{equation}
E_{\mathrm{agg}}^{(t)}
=
\left\lVert
\bm B_g^{(t)}\bm A_g^{(t)}
-
\Delta\bm W_{\mathrm{ideal}}^{(t)}
\right\rVert_{\mathrm F}.
\label{eq:aggregation_error}
\end{equation}

To account for different scales of $\bm B$ and $\bm A$, we measure factor-level initialization mismatch for each $\bm Z\in\{\bm B,\bm A\}$ using
\begin{equation}
E_{\mathrm{init}}^{(t)}(\bm Z)
=
\frac{1}{N}
\sum_{i=1}^{N}
\frac{
\left\lVert
\bm Z_i^{(t+1,0)}-\bm Z_i^{(t)}
\right\rVert_{\mathrm F}
}{
\left\lVert
\bm Z_i^{(t+1,0)}
\right\rVert_{\mathrm F}
+\varepsilon
}.
\label{eq:initialization_error}
\end{equation}
where $\varepsilon>0$ avoids division by zero, and $\bm Z_i^{(t)}$ and $\bm Z_i^{(t+1,0)}$ denote the locally optimized factor and its next-round initialization, respectively. This metric measures the factor-level change at next-round initialization rather than differences caused by data heterogeneity.

Fig.~\ref{fig:error_results} shows a clear trade-off between aggregation fidelity and factor-level continuity. FedRot-LoRA reduces FedIT's aggregation error through factor alignment, while FlexLoRA further lowers it through optimal rank-$r$ reconstruction. FedPA-LoRA achieves the lowest observed aggregation error by combining optimal server-side reconstruction with product-guided alignment, which keeps local products close to the global reference.

Initialization errors show the opposite trend. FedIT has the smallest mismatch, FedRot-LoRA remains intermediate, and FlexLoRA shows the largest error due to SVD-based factor replacement. The larger error for $\bm B$ than for $\bm A$ is consistent with prior findings that $\bm B$ captures more client-specific adaptation under non-IID data \citep{guo2025selective}. FedPA-LoRA instead preserves local factors across rounds, yielding zero initialization error while maintaining the lowest aggregation error.
\section{Proposed Method: FedPA-LoRA}

\begin{figure*}[t]
    \centering
    \includegraphics[
        width=\linewidth
    ]{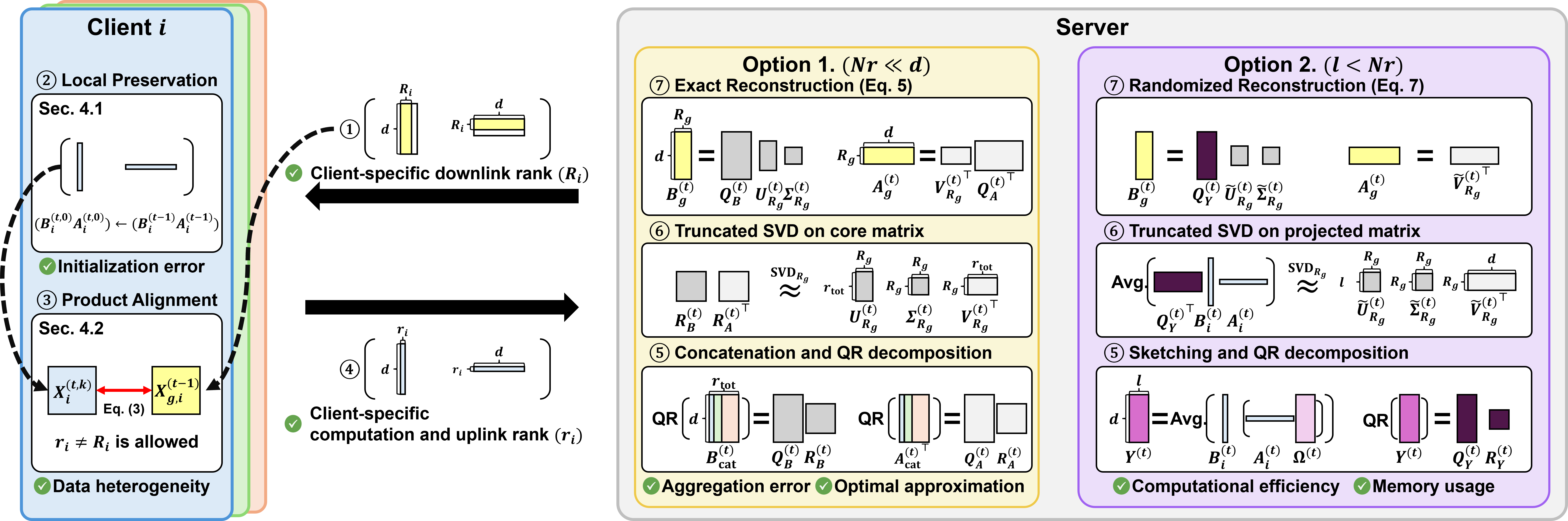}
    \caption{
    Overview of FedPA-LoRA. Clients preserve local factors and align their products with rank-$R_i$ global references. The server reconstructs a rank-$R_g$ adapter via exact or randomized SVD without dense aggregation.
    }
    \label{fig:framework}
\end{figure*}

We consider an FL system with $N$ clients that collaboratively learn a global LoRA adapter over $T$ communication rounds, where $\bm W_0$ denotes the frozen weight matrix of each adapted layer. At round $t\in\{1,\ldots,T\}$, client $i$ performs $\tau$ local updates on its rank-$r_i$ factors, indexed by $k=0,\ldots,\tau-1$. The factors before the $k$-th update are denoted by $(\bm B_i^{(t,k)},\bm A_i^{(t,k)})$, and the update produces $(\bm B_i^{(t,k+1)},\bm A_i^{(t,k+1)})$. We denote the factors after $\tau$ local updates by $(\bm B_i^{(t)},\bm A_i^{(t)}):=(\bm B_i^{(t,\tau)},\bm A_i^{(t,\tau)})$. The server provides client $i$ with a rank-$R_i$ global LoRA reference that guides its locally preserved rank-$r_i$ factors through product-level alignment. Here, $r_i$ determines computation and uplink costs, whereas $R_i$ determines downlink cost. The server then aggregates the local updates directly in the product space and efficiently reconstructs a low-rank global adapter without forming the dense aggregate, as illustrated in Fig.~\ref{fig:framework}.

\subsection{Local Factor Preservation}

Conventional FL initializes local training from the global model at each round. In federated LoRA, however, replacing local factors with newly reconstructed global factors can introduce initialization mismatch and disrupt optimization continuity. FedPA-LoRA avoids this mismatch by preserving local factors across rounds.

At initialization, the server constructs global LoRA factors $(\bm B_g^{(0)},\bm A_g^{(0)})$ with rank $\max_j\{r_j,R_j\}$ and sends the first $\max\{r_i,R_i\}$ components to client $i$. Each client initializes its rank-$r_i$ local factors $(\bm B_i^{(0)},\bm A_i^{(0)})$ from the first $r_i$ components. At each round $t\geq1$, client $i$ preserves its local factors by setting $(\bm B_i^{(t,0)},\bm A_i^{(t,0)})\leftarrow(\bm B_i^{(t-1)},\bm A_i^{(t-1)})$.

\subsection{Global Product-Guided Alignment}

While local factor preservation avoids initialization mismatch, it does not explicitly address client drift under heterogeneous data or incorporate updated global information. FedPA-LoRA therefore introduces product-level regularization to guide local adaptation without replacing the preserved factors.

In the first round, client $i$ extracts its rank-$R_i$ global reference $(\bm B_{g,i}^{(0)},\bm A_{g,i}^{(0)})$ from the initially received factors. For subsequent rounds $t\geq2$, the server sends $(\bm B_{g,i}^{(t-1)},\bm A_{g,i}^{(t-1)})$ by retaining the first $R_i$ columns of $\bm B_g^{(t-1)}$ and rows of $\bm A_g^{(t-1)}$. At each round $t\geq1$, client $i$ performs $\tau$ local updates from $(\bm B_i^{(t,0)},\bm A_i^{(t,0)})$ to minimize
\begin{equation}
\label{eq:product_guided_objective}
\mathcal L_i^{(t)}(\bm B_i,\bm A_i)
= f_i(\bm W_0+\bm B_i\bm A_i)
+ \frac{\lambda}{2}
\|\bm B_i\bm A_i-\bm B_{g,i}^{(t-1)}\bm A_{g,i}^{(t-1)}\|_{\mathrm F}^{2}.
\end{equation}
where $f_i$ denotes the local task objective and $\lambda$ controls the strength of regularization. By operating in the product space, FedPA-LoRA allows local and reference ranks $r_i$ and $R_i$ to be adjusted independently. An efficient implementation of the regularization term and its complexity analysis are provided in Appendix~\ref{app:client_wallclock_evaluation}.

\subsection{Product-Space Aggregation}

After local training, clients send $(\bm B_i^{(t)},\bm A_i^{(t)})$ to the server for product-space aggregation, avoiding factor-wise aggregation mismatch. Direct truncated SVD yields the optimal low-rank approximation but requires forming and decomposing the dense product aggregate. We therefore exploit the factorized structure for exact QR-based reconstruction and a randomized variant. By balancing client-specific knowledge with globally shared information, FedPA LoRA maintains strong performance from near-IID to highly heterogeneous settings.

\textbf{Exact reconstruction.}
We concatenate the local factors as
\begin{equation}
\begin{aligned}
\bm B_{\mathrm{cat}}^{(t)}
&=
\frac{1}{\sqrt N}
[\bm B_1^{(t)},\ldots,\bm B_N^{(t)}]
\in\mathbb R^{d\times r_{\mathrm{tot}}},\\[-2pt]
\bm A_{\mathrm{cat}}^{(t)}
&=
\frac{1}{\sqrt N}
[(\bm A_1^{(t)})^\top,\ldots,(\bm A_N^{(t)})^\top]^\top
\in\mathbb R^{r_{\mathrm{tot}}\times d},
\end{aligned}
\label{eq:concatenated_local_factors}
\end{equation}
so that
$\bm B_{\mathrm{cat}}^{(t)}\bm A_{\mathrm{cat}}^{(t)}
=
\Delta\bm W_{\mathrm{ideal}}^{(t)}$.

With $r_{\mathrm{tot}}=\sum_{i=1}^{N}r_i$ and $R_g=\max_{1\le i\le N}R_i$, the server computes the reduced QR factorizations $\bm B_{\mathrm{cat}}^{(t)}=\bm Q_B^{(t)}\bm R_B^{(t)}$ and $(\bm A_{\mathrm{cat}}^{(t)})^\top=\bm Q_A^{(t)}\bm R_A^{(t)}$. It then applies rank-$R_g$ truncated SVD to the $r_{\mathrm{tot}}\times r_{\mathrm{tot}}$ core matrix $\bm H^{(t)}=\bm R_B^{(t)}(\bm R_A^{(t)})^\top$, yielding $\bm U_{R_g}^{(t)}$, $\bm\Sigma_{R_g}^{(t)}$, and $\bm V_{R_g}^{(t)}$. The global factors are reconstructed as
\begin{equation}
\bm B_g^{(t)}
=
\bm Q_B^{(t)}
\bm U_{R_g}^{(t)}
\bm\Sigma_{R_g}^{(t)},
\qquad
\bm A_g^{(t)}
=
(\bm Q_A^{(t)}\bm V_{R_g}^{(t)})^\top.
\label{eq:efficient_global_lora_reconstruction}
\end{equation}
By truncated-SVD optimality, the resulting update satisfies
\begin{equation}
\Delta\bm W_g^{(t)}
\in
\operatorname*{arg\,min}_{\operatorname{rank}(\bm X)\leq R_g}
\left\|
\bm X-\Delta\bm W_{\mathrm{ideal}}^{(t)}
\right\|_{\mathrm F}.
\label{eq:optimal_rank_projection}
\end{equation}
For homogeneous ranks with $Nr\ll d$, this reduces the dominant server-side complexity from $\mathcal O(d^3)$ for dense SVD to $\mathcal O(N^2dr^2)$.

\textbf{Randomized reconstruction.}
To further reduce server-side cost, we use a sketch dimension $\ell=R_g+p$ and sample a Gaussian matrix $\bm\Omega^{(t)}\in\mathbb R^{d\times\ell}$. The server directly accumulates
\[
\bm Y^{(t)}
=
\frac{1}{N}
\sum_{i=1}^{N}
\bm B_i^{(t)}
\left(
\bm A_i^{(t)}\bm\Omega^{(t)}
\right)
\in\mathbb R^{d\times\ell}
\]
and computes its reduced QR factorization $\bm Y^{(t)}=\bm Q_Y^{(t)}\bm R_Y^{(t)}$. It then applies rank-$R_g$ truncated SVD to the projected aggregate $\frac{1}{N}\sum_{i=1}^{N}\left((\bm Q_Y^{(t)})^\top\bm B_i^{(t)}\right)\bm A_i^{(t)}$, obtaining $\widetilde{\bm U}_{R_g}^{(t)}$, $\widetilde{\bm\Sigma}_{R_g}^{(t)}$, and $\widetilde{\bm V}_{R_g}^{(t)}$. The global factors are reconstructed as
\begin{equation}
\bm B_g^{(t)}
=
\bm Q_Y^{(t)}
\widetilde{\bm U}_{R_g}^{(t)}
\widetilde{\bm\Sigma}_{R_g}^{(t)},
\qquad
\bm A_g^{(t)}
=
(\widetilde{\bm V}_{R_g}^{(t)})^\top.
\label{eq:randomized_reconstruction}
\end{equation}
For homogeneous ranks with $\ell=\mathcal O(r)$, this reduces the dominant server-side complexity from $\mathcal O(N^2dr^2)$ to $\mathcal O(Ndr^2)$ while avoiding explicit construction of the concatenated factors. The probabilistic approximation guarantee is provided in Remark~\ref{remark:randomized_reconstruction}, with detailed complexity analysis in Appendix~\ref{app:server_complexity}.

\subsection{Convergence Analysis}
\label{subsec:convergence_analysis}

We analyze FedPA-LoRA with exact reconstruction under homogeneous local and reference ranks, $r_i=R_i=r$, and full client participation. For client $i$ at local step $k$ of communication round $t$, we define $\bm X_i^{(t,k)}=\bm B_i^{(t,k)}\bm A_i^{(t,k)}$ and $\bm W_i^{(t,k)}=\bm W_0+\bm X_i^{(t,k)}$. For $t\in\{1,\ldots,T\}$, we denote the end-of-round iterates by $\bm X_i^{(t)}=\bm X_i^{(t,\tau)}$ and $\bm W_i^{(t)}=\bm W_i^{(t,\tau)}$. The round-wise joint objective is then defined as
\begin{equation}
\Psi^{(t)}
=
\frac{1}{N}
\sum_{i=1}^{N}
\left[
f_i\left(\bm W_i^{(t)}\right)
+
\frac{\lambda}{2}
\left\|
\bm X_i^{(t)}-\bm X_g^{(t)}
\right\|_{\mathrm F}^{2}
\right].
\label{eq:conv_joint_objective}
\end{equation}

Our convergence analysis is based on the following four assumptions.

\begin{assumption}[$L_s$-smoothness]
\label{assumption:conv_smoothness}
For every client $i$ and any $\bm W,\bm W'$,
\begin{equation}
\|\nabla_{\bm W}f_i(\bm W)-\nabla_{\bm W}f_i(\bm W')\|_{\mathrm F}
\leq L_s\|\bm W-\bm W'\|_{\mathrm F}.
\label{eq:conv_task_smoothness}
\end{equation}
\end{assumption}

\begin{assumption}[Stochastic task gradients]
\label{assumption:conv_stochastic_gradients}
For every client $i$, communication round $t$, and local step $k$, the stochastic task gradient computed using mini-batch $\xi_i^{(t,k)}$ is unbiased:
\begin{equation}
\mathbb E_{\xi_i^{(t,k)}}\left[\nabla_{\bm W}f_i(\bm W_i^{(t,k)};\xi_i^{(t,k)})\right]
=
\nabla_{\bm W}f_i(\bm W_i^{(t,k)}).
\label{eq:conv_unbiased_task_gradient}
\end{equation}
Moreover, there exists a constant $G_f>0$ such that
\begin{equation}
\left\|\nabla_{\bm W}f_i(\bm W_i^{(t,k)};\xi_i^{(t,k)})\right\|_{\mathrm F}
\leq G_f.
\label{eq:conv_bounded_task_gradient}
\end{equation}
\end{assumption}

\begin{assumption}[LoRA factor regularity]
\label{assumption:conv_factor_regularity}
There exist constants $C_A,C_B>0$ such that, for every client $i$, communication round $t$, and local step $k$,
\begin{equation}
\|\bm A_i^{(t,k)}\|_{\mathrm F}\leq C_A,\qquad
\|\bm B_i^{(t,k)}\|_{\mathrm F}\leq C_B.
\label{eq:conv_bounded_factors}
\end{equation}
\end{assumption}

\begin{assumption}[$L_c$-Lipschitz continuity]
\label{assumption:conv_lipschitz}
For every client $i$ and any $\bm W,\bm W'$,
\begin{equation}
|f_i(\bm W)-f_i(\bm W')|
\leq L_c\|\bm W-\bm W'\|_{\mathrm F}.
\label{eq:conv_task_lipschitz}
\end{equation}
\end{assumption}

These assumptions yield a descent bound in terms of the factor-space gradient norms. The constants $L_s,\lambda,G_f,C_A,C_B$ determine $M_\lambda$ in Theorem~\ref{theorem:conv_fedpa}.

\begin{theorem}[Convergence of FedPA-LoRA]
\label{theorem:conv_fedpa}
Under these assumptions, suppose that $\frac{1}{N}\sum_{i=1}^{N}f_i(\bm W)\geq\Psi^\star$ for all $\bm W$ and $\mathbb E[\Psi^{(0)}]-\Psi^\star\leq D$ for some $D>0$. Defining $\widetilde D:=D+L_cC_AC_B$, there exists a constant $M_\lambda>0$, independent of $T$ and $\tau$, such that for any $0<\eta\leq1$,
\begin{equation}
\frac{1}{NT\tau}\sum_{t=1}^{T}\sum_{i=1}^{N}\sum_{k=0}^{\tau-1}\mathbb E\left[\|\nabla_{\bm B}\mathcal L_i^{(t)}(\bm W_i^{(t,k)})\|_{\mathrm F}^{2}+\|\nabla_{\bm A}\mathcal L_i^{(t)}(\bm W_i^{(t,k)})\|_{\mathrm F}^{2}\right]\leq\frac{\widetilde D}{\eta T\tau}+M_\lambda\eta.
\label{eq:conv_average_gradient_bound}
\end{equation}
Consequently, when $\widetilde D\leq M_\lambda T\tau$, choosing $\eta=\sqrt{\widetilde D/(M_\lambda T\tau)}$ yields an averaged factor-space stationarity rate of $\mathcal O((T\tau)^{-1/2})$.
\end{theorem}

\begin{proposition}[Global LoRA Loss]
\label{proposition:conv_global_loss}
Under Assumption~\ref{assumption:conv_lipschitz}, define $F(\bm W)=N^{-1}\sum_{i=1}^{N}f_i(\bm W)$ and $\bm W_g^{(t)}=\bm W_0+\bm X_g^{(t)}$. Then,
\begin{equation}
F(\bm W_g^{(t)})
\leq
\Psi^{(t)}+\frac{L_c^2}{2\lambda}.
\label{eq:conv_global_loss_control}
\end{equation}
Thus, the global LoRA task loss is controlled by the joint objective up to a $\lambda$-dependent term.
\end{proposition}

The complete proofs for the homogeneous setting appear in Appendix~\ref{appendix:convergence_analysis}. Appendix~\ref{app:convergence_heterogeneous} extends the analysis to heterogeneous local and reference ranks with an additional reference-truncation error term, while Appendix~\ref{app:partial_participation} considers partial participation with $K$ active clients per round.
\section{Experiments}

We conduct experiments to answer four questions: \emph{(i)} Does FedPA-LoRA consistently improve performance across various natural language tasks? \emph{(ii)} Is it robust to data and resource heterogeneity, client scale and participation, and rank variation? \emph{(iii)} Do its core components contribute as intended? \emph{(iv)} What efficiency trade-offs does FedPA-LoRA introduce? Appendix~\ref{app:exp_details} provides complete experimental configurations and additional results, while Appendix~\ref{app:complexity} presents detailed theoretical and empirical efficiency analyses.

\subsection{Experimental Setup}

\subsubsection{Datasets and Models.}

We evaluate RoBERTa-Large \citep{liu2019roberta} on SST-2, QNLI, QQP, RTE, and MNLI from GLUE \citep{wang-etal-2018-glue}. For generation, we evaluate Llama 3-8B \citep{grattafiori2024llama} on GSM8K \citep{cobbe2021training} and HumanEval \citep{chen2021evaluating}, using CodeSearchNet \citep{husain2019codesearchnet} for HumanEval fine-tuning.

\subsubsection{Baselines.}

For homogeneous ranks, we compare with FedIT \citep{10447454}, FlexLoRA \citep{bai2024federated}, FFA-LoRA \citep{ICLR2024_4e243e95}, RoLoRA \citep{chen2025robust}, and FedRot-LoRA \citep{zhang2026fedrotlora}. FedPA-LoRA preserves local parameters across rounds and uploads only the LoRA factors. Accordingly, on GLUE, each client retains its local classifier head and evaluates the reconstructed global LoRA with that head. We therefore include the personalized methods FedDPA-LoRA \citep{yang2024dualpersonalizing} and FedSA-LoRA \citep{guo2025selective}. For heterogeneous ranks, we compare with FlexLoRA, HetLoRA \citep{cho-etal-2024-heterogeneous}, Ravan \citep{raje2025ravan}, and Fed-PLoRA \citep{zhang2026heterogeneous}.

\subsubsection{Implementation Details.}

We implement all methods using FederatedScope-LLM \citep{kuang2024federatedscope} and report means and standard deviations over three seeds. We use 20 local updates for 250 rounds on GLUE and 30 updates for 200 rounds on generation tasks. Unless otherwise stated, FedPA-LoRA uses the exact reduced-QR and core-SVD reconstruction. Implementation details are provided in Appendix~\ref{app:experimental_details}.

\subsection{Natural Language Understanding}

\begin{table}[t]
\centering
\caption{
GLUE accuracy under homogeneous- and heterogeneous-rank settings with $\beta=0.5$. The homogeneous setting uses $N=10$ and $r=4$, while the heterogeneous setting uses $N=20$ across three resource groups.
}
\label{tab:nlu_results}

{\scriptsize
\renewcommand{\arraystretch}{0.90}
\begin{tabular*}{\textwidth}{
@{\extracolsep{\fill}}
clcccccc
@{}
}
\toprule
\textbf{Setting}
& \textbf{Method}
& \textbf{SST-2}
& \textbf{QNLI}
& \textbf{QQP}
& \textbf{RTE}
& \textbf{MNLI}
& \textbf{Average} \\
\midrule

\multirow{8}{*}{\textbf{Homogeneous}}
& FedIT
& $0.956{\pm}0.000$
& $0.891{\pm}0.046$
& \underline{$0.861{\pm}0.001$}
& $0.696{\pm}0.018$
& $0.858{\pm}0.001$
& $0.8524$ \\

& FlexLoRA
& $0.956{\pm}0.001$
& $0.642{\pm}0.100$
& $0.848{\pm}0.010$
& $0.742{\pm}0.061$
& $0.858{\pm}0.002$
& $0.8092$ \\

& FFA-LoRA
& $0.946{\pm}0.001$
& $0.890{\pm}0.010$
& $0.848{\pm}0.002$
& $0.611{\pm}0.021$
& $0.840{\pm}0.001$
& $0.8270$ \\

& RoLoRA
& $0.957{\pm}0.002$
& $0.882{\pm}0.046$
& $0.845{\pm}0.024$
& \underline{$0.794{\pm}0.022$}
& $0.852{\pm}0.002$
& \underline{$0.8660$} \\

& FedRot-LoRA
& $0.954{\pm}0.002$
& \underline{$0.908{\pm}0.004$}
& $0.849{\pm}0.002$
& $0.714{\pm}0.038$
& $0.858{\pm}0.002$
& $0.8566$ \\

& FedDPA-LoRA
& $0.957{\pm}0.001$
& $0.851{\pm}0.062$
& $0.811{\pm}0.045$
& $0.612{\pm}0.105$
& $0.879{\pm}0.002$
& $0.8220$ \\

& FedSA-LoRA
& \underline{$0.959{\pm}0.002$}
& $0.905{\pm}0.006$
& $0.814{\pm}0.025$
& $0.766{\pm}0.009$
& \underline{$0.880{\pm}0.001$}
& $0.8648$ \\

& FedPA-LoRA
& $\mathbf{0.969{\pm}0.001}$
& $\mathbf{0.942{\pm}0.003}$
& $\mathbf{0.884{\pm}0.001}$
& $\mathbf{0.871{\pm}0.005}$
& $\mathbf{0.907{\pm}0.002}$
& $\mathbf{0.9146}$ \\

\midrule

\multirow{5}{*}{\textbf{Heterogeneous}}
& FlexLoRA
& $0.850{\pm}0.015$
& $0.798{\pm}0.098$
& $0.761{\pm}0.033$
& $0.570{\pm}0.011$
& $0.832{\pm}0.031$
& $0.7622$ \\

& HetLoRA
& \underline{$0.957{\pm}0.001$}
& \underline{$0.919{\pm}0.003$}
& $0.850{\pm}0.002$
& $0.671{\pm}0.039$
& \underline{$0.875{\pm}0.000$}
& \underline{$0.8544$} \\

& Ravan
& $0.953{\pm}0.001$
& $0.883{\pm}0.040$
& \underline{$0.851{\pm}0.003$}
& \underline{$0.712{\pm}0.016$}
& $0.866{\pm}0.003$
& $0.8530$ \\

& Fed-PLoRA
& $0.932{\pm}0.027$
& $0.712{\pm}0.137$
& $0.811{\pm}0.006$
& $0.588{\pm}0.013$
& $0.839{\pm}0.004$
& $0.7764$ \\

& FedPA-LoRA
& $\mathbf{0.966{\pm}0.003}$
& $\mathbf{0.959{\pm}0.001}$
& $\mathbf{0.917{\pm}0.009}$
& $\mathbf{0.860{\pm}0.010}$
& $\mathbf{0.911{\pm}0.003}$
& $\mathbf{0.9226}$ \\

\bottomrule
\end{tabular*}
}
\end{table}

In the homogeneous-rank setting, FedPA-LoRA achieves the best performance on all five GLUE tasks, averaging $0.9146$ and outperforming RoLoRA by $4.86$ percentage points (Table~\ref{tab:nlu_results}). These results highlight the effectiveness of combining local factor preservation, product-guided alignment, and product-space aggregation.

\subsection{Natural Language Generation}

GSM8K uses three IID clients, while HumanEval uses six clients partitioned by programming language. FedPA-LoRA achieves the best performance on both tasks, reaching $0.4672$ accuracy on GSM8K and $0.4121$ pass@1 on HumanEval. Since neither task uses client-specific classifier heads, these results further demonstrate the effectiveness of FedPA-LoRA beyond classifier-head personalization. Detailed results are provided in Table~\ref{tab:nlg_results} of Appendix~\ref{app:extra_results}.

\subsection{In-Depth Analyses}

\subsubsection{Robustness Analyses.}

We further evaluate FedPA-LoRA across data heterogeneity, client scale, and LoRA ranks.

% data heterogeneity
\begin{table}[t]
\centering
\caption{
MNLI accuracy across data heterogeneity levels $\beta\in\{100,1,0.5,0.1\}$ with $N=3$ and $r=4$.
}
\label{tab:hetero_results}

\begin{tabular}{@{}lcccc@{}}
\toprule
\textbf{Method}
& $\bm{\beta}=\mathbf{100}$
& $\bm{\beta}=\mathbf{1}$
& $\bm{\beta}=\mathbf{0.5}$
& $\bm{\beta}=\mathbf{0.1}$ \\
\midrule

FedIT
& $0.873{\pm}0.002$
& $0.859{\pm}0.001$
& $0.869{\pm}0.001$
& $0.858{\pm}0.004$ \\

FlexLoRA
& $0.868{\pm}0.001$
& \underline{$0.877{\pm}0.002$}
& $0.869{\pm}0.002$
& $0.859{\pm}0.011$ \\

FFA-LoRA
& $0.867{\pm}0.002$
& $0.856{\pm}0.002$
& $0.856{\pm}0.002$
& $0.824{\pm}0.003$ \\

RoLoRA
& $0.867{\pm}0.008$
& $0.676{\pm}0.252$
& $0.862{\pm}0.000$
& $0.822{\pm}0.008$ \\

FedRot-LoRA
& \underline{$0.874{\pm}0.001$}
& $0.868{\pm}0.002$
& $0.868{\pm}0.002$
& $0.857{\pm}0.005$ \\

FedDPA-LoRA
& $0.873{\pm}0.003$
& $0.702{\pm}0.265$
& \underline{$0.879{\pm}0.002$}
& $0.686{\pm}0.257$ \\

FedSA-LoRA
& $0.862{\pm}0.001$
& $0.865{\pm}0.001$
& $0.874{\pm}0.002$
& \underline{$0.863{\pm}0.006$} \\

FedPA-LoRA
& $\bm{0.880{\pm}0.002}$
& $\bm{0.914{\pm}0.001}$
& $\bm{0.906{\pm}0.003}$
& $\bm{0.892{\pm}0.002}$ \\

\bottomrule
\end{tabular}
\end{table}

\textbf{(1) Data Heterogeneity.} Table~\ref{tab:hetero_results} reports MNLI accuracy from near-IID ($\beta=100$) to increasingly heterogeneous client distributions. Global federated LoRA methods remain competitive under relatively homogeneous data, where client updates are better aligned, while personalized methods become more effective as heterogeneity increases. FedDPA-LoRA and FedSA-LoRA outperform all global baselines at $\beta=0.5$ and $\beta=0.1$, respectively, although FedDPA-LoRA varies substantially across $\beta$. FedPA-LoRA achieves the highest accuracy across all heterogeneity levels, with local factor preservation retaining client-specific knowledge and product-guided regularization maintaining consistency with the global reference under heterogeneous data. By balancing client-specific knowledge with globally shared information, FedPA LoRA maintains strong performance from near-IID to highly heterogeneous settings.

% number of clients
\begin{table}[t]
\centering
\caption{
MNLI accuracy under different numbers of clients $N\in\{3,10,50\}$ with $r=4$ and $\beta=0.5$.
}
\label{tab:client_results}

\begin{tabular}{@{}lccc@{}}
\toprule
\textbf{Method}
& $\bm{N}=\mathbf{3}$
& $\bm{N}=\mathbf{10}$
& $\bm{N}=\mathbf{50}$ \\
\midrule

FedIT
& $0.869{\pm}0.001$
& $0.858{\pm}0.001$
& $0.877{\pm}0.001$ \\

FlexLoRA
& $0.869{\pm}0.002$
& $0.858{\pm}0.002$
& $0.828{\pm}0.025$ \\

FFA-LoRA
& $0.856{\pm}0.002$
& $0.840{\pm}0.001$
& $0.860{\pm}0.001$ \\

RoLoRA
& $0.862{\pm}0.000$
& $0.852{\pm}0.002$
& $0.882{\pm}0.002$ \\

FedRot-LoRA
& $0.868{\pm}0.002$
& $0.858{\pm}0.002$
& $0.878{\pm}0.001$ \\

FedDPA-LoRA
& \underline{$0.879{\pm}0.002$}
& $0.879{\pm}0.002$
& $0.883{\pm}0.014$ \\

FedSA-LoRA
& $0.874{\pm}0.002$
& \underline{$0.880{\pm}0.001$}
& \underline{$0.891{\pm}0.000$} \\

FedPA-LoRA
& $\bm{0.906{\pm}0.003}$
& $\bm{0.907{\pm}0.002}$
& $\bm{0.916{\pm}0.001}$ \\

\bottomrule
\end{tabular}
\end{table}

\textbf{(2) Client Scale.} FedPA-LoRA consistently outperforms all baselines as $N$ increases from $3$ to $50$, maintaining strong accuracy from $0.906$ to $0.916$ (Table~\ref{tab:client_results}). Additional results under varying participation rates are provided in Appendix~\ref{app:client_participation_rate}.

\begin{figure}[t]
\centering
\includegraphics[scale=0.5]{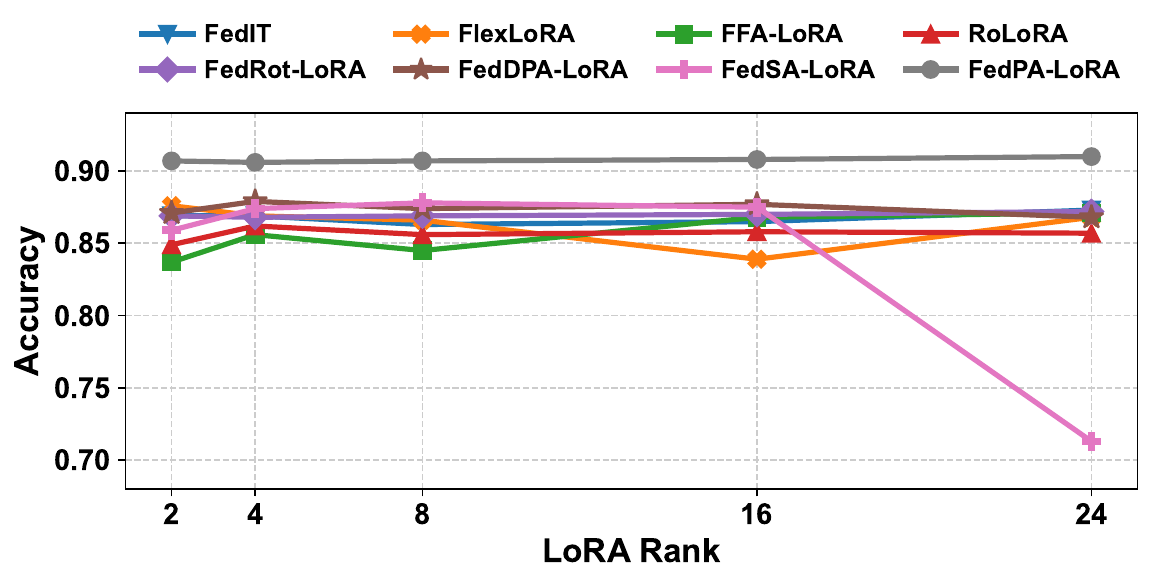}
\caption{MNLI accuracy under different LoRA ranks $r\in\{2,4,8,16,24\}$ with $\beta=0.5$ and $N=3$.}
\label{fig:rank_scale}
\end{figure}

\textbf{(3) LoRA Rank.} FedPA-LoRA achieves the best performance across all evaluated ranks, remaining stable between $0.906$ and $0.910$, while several baselines fluctuate substantially, particularly FedSA-LoRA at $r=24$ (Figure~\ref{fig:rank_scale}). Detailed results are provided in Table~\ref{tab:rank_results} of Appendix~\ref{app:extra_results}.

\subsubsection{Effect of Rank Heterogeneity.}

We evaluate natural language understanding and generation under heterogeneous client resources using three resource groups. 

\textbf{(1) Natural Language Understanding.} FedPA-LoRA achieves the best performance on all five GLUE tasks, averaging $0.9226$ and outperforming HetLoRA by $6.82$ percentage points (Table~\ref{tab:nlu_results}). This demonstrates its effectiveness under heterogeneous client capacities.

\textbf{(2) Natural Language Generation.} FedPA-LoRA also performs best on GSM8K and HumanEval, reaching $0.5018$ accuracy and $0.4048$ pass@1, respectively. Unlike Ravan and Fed-PLoRA, which distribute their full sets of global modules, FedPA-LoRA transmits only the reference corresponding to each client's rank. Detailed results are provided in Table~\ref{tab:hetero_nlg_results} of Appendix~\ref{app:extra_results}.

\textbf{(3) Rank Decoupling.} With fixed heterogeneous local ranks $r_i\in\{1,2,4\}$, we independently vary the reference ranks $R_i$. Accuracy generally improves with larger reference budgets, while assigning reference ranks in reverse order of local ranks slightly reduces performance. Nevertheless, FedPA-LoRA maintains stable MNLI accuracy between $0.905$ and $0.912$ across configurations, enabling flexible control of computation and downlink communication, as detailed in Appendix~\ref{app:rank_decoupling}.

\subsubsection{Effect of the Regularization Strength \texorpdfstring{$\lambda$}{lambda}.}

\begin{figure}[t]
\centering
\begin{minipage}[t]{0.49\textwidth}
\centering
\begin{subfigure}[t]{0.48\linewidth}
\centering
\includegraphics[width=\linewidth, trim=0 0 0 0, clip]{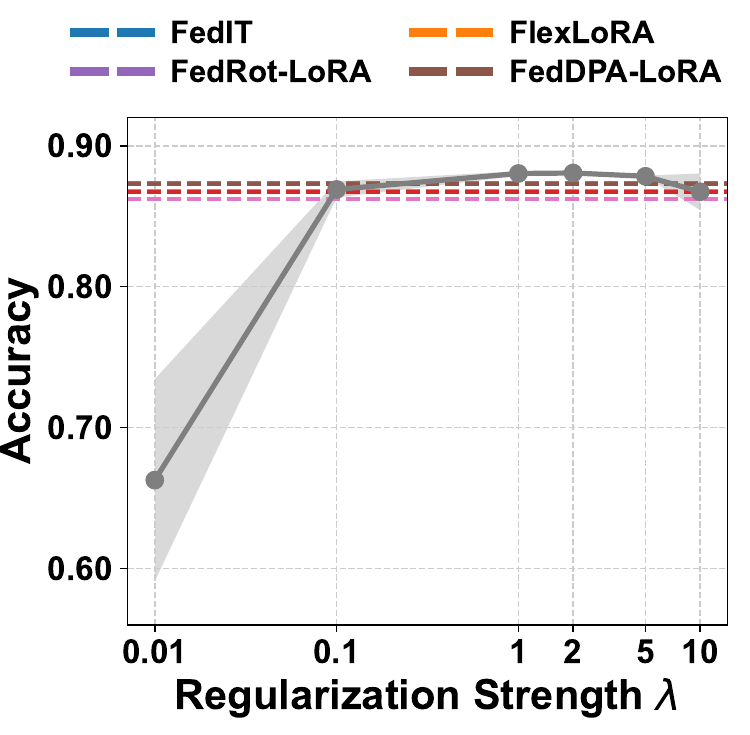}
\caption{IID}
\label{fig:lam_iid}
\end{subfigure}
\hfill
\begin{subfigure}[t]{0.48\linewidth}
\centering
\includegraphics[width=\linewidth, trim=0 0 0 0, clip]{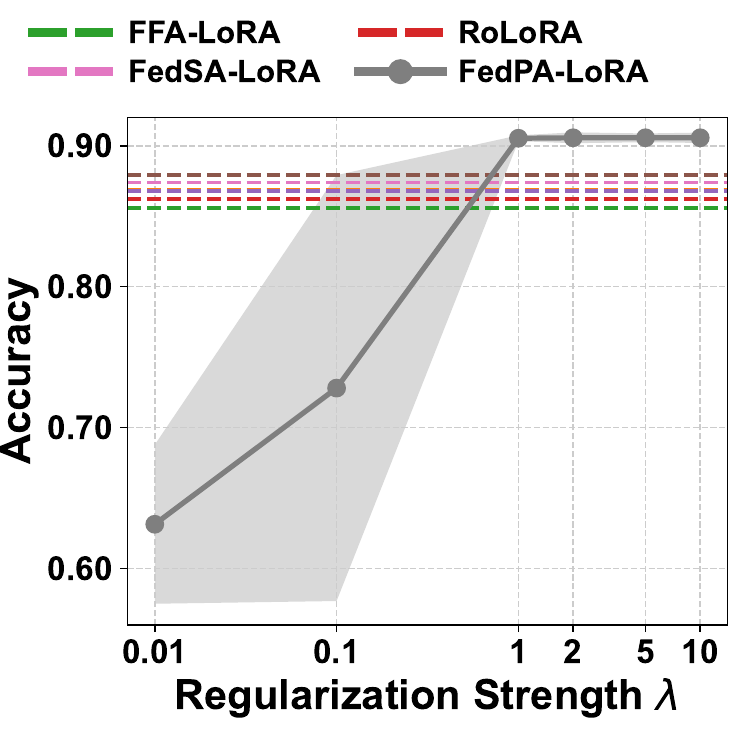}
\caption{Non-IID}
\label{fig:lam_niid}
\end{subfigure}
\caption{MNLI accuracy under different regularization strengths $\lambda$ with $N=3$ clients and $r=R=4$. (a) Near-IID ($\beta=100$). (b) Non-IID ($\beta=0.5$). Horizontal lines denote baseline means over three random seeds.}
\label{fig:lambda_results}
\end{minipage}
\hfill
\begin{minipage}[t]{0.49\textwidth}
\centering
\begin{subfigure}[t]{0.48\linewidth}
\centering
\includegraphics[width=\linewidth, trim=0 0 0 0, clip]{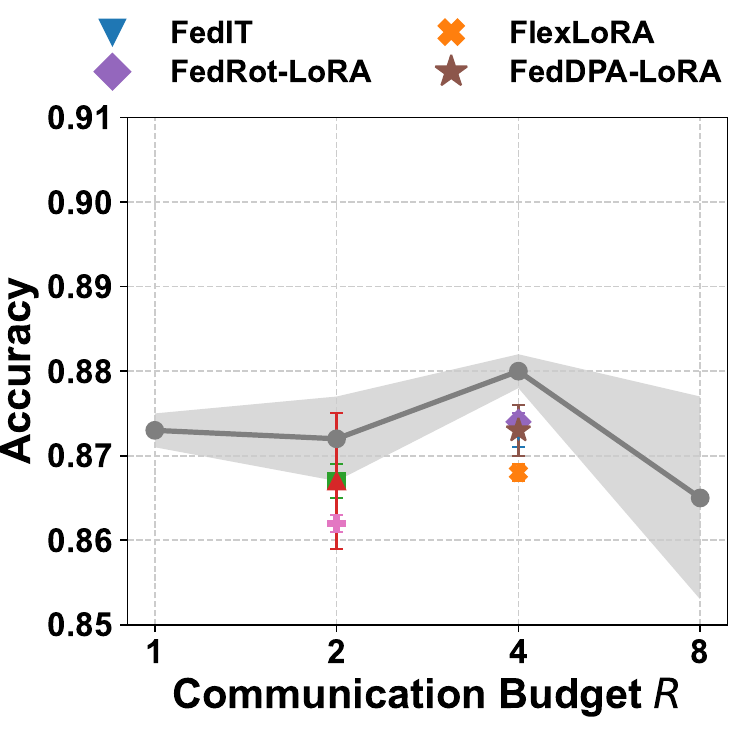}
\caption{IID}
\label{fig:R_iid}
\end{subfigure}
\hfill
\begin{subfigure}[t]{0.48\linewidth}
\centering
\includegraphics[width=\linewidth, trim=0 0 0 0, clip]{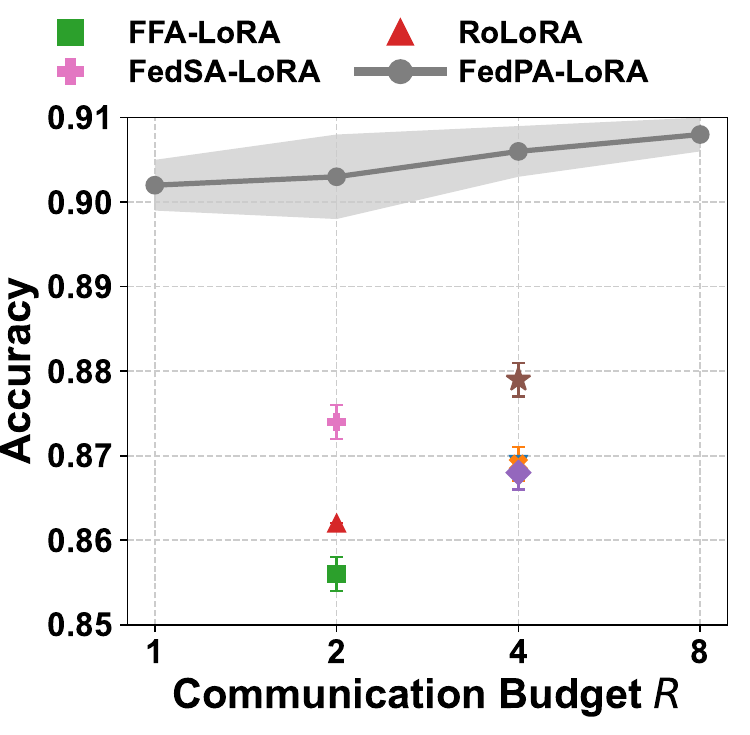}
\caption{Non-IID}
\label{fig:R_niid}
\end{subfigure}
\caption{MNLI accuracy across communication ranks $R\in\{1,2,4,8\}$ with local rank $r=4$ and $N=3$ clients. (a) Near-IID ($\beta=100$). (b) Non-IID ($\beta=0.5$).}
\label{fig:R_results}
\end{minipage}
\end{figure}

Figure~\ref{fig:lambda_results} examines sensitivity to the product-guided regularization strength $\lambda$. Under IID data, accuracy is low at $\lambda=0.01$ but remains high and stable at moderate values. It declines at $\lambda=10$, suggesting that overly strong regularization can restrict local adaptation when client updates are already aligned.

Under non-IID data, small $\lambda\in\{0.01,0.1\}$ yields lower accuracy and higher variance due to insufficient global guidance. Increasing $\lambda$ improves both accuracy and stability, with consistently strong performance for $\lambda\in\{1,2,5,10\}$. This highlights the role of product-guided regularization in mitigating client drift under data heterogeneity. Appendix~\ref{app:adaptive_regularization} discusses client-specific adaptive weights $\lambda_i$ for varying degrees of client drift.

\subsubsection{Effect of the Communication Budget \texorpdfstring{$R$}{R}.}

Figure~\ref{fig:R_results} shows that FedPA-LoRA performs well across downlink ranks $R\in\{1,2,4,8\}$. Increasing $R$ offers no consistent gains under IID data but improves accuracy under non-IID data. This suggests that smaller reference ranks can reduce downlink communication, while larger ranks provide more effective global guidance under data heterogeneity.

\subsection{Effect of Sketch Dimension \texorpdfstring{$\ell$}{ell}}

\begin{figure}[t]
\centering
\includegraphics[scale=0.5]{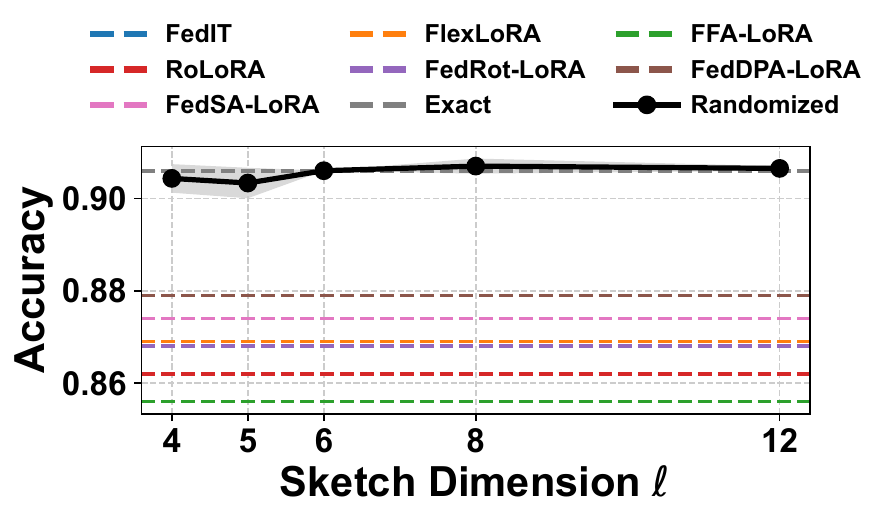}
\caption{MNLI accuracy under randomized reconstruction with $\ell=R_g+p\in\{4,5,6,8,12\}$, $N=3$, $r=R_g=4$, and $\beta=0.5$.}
\label{fig:randomized_accuracy}
\end{figure}

Figure~\ref{fig:randomized_accuracy} examines the effect of sketch dimension $\ell$ on randomized reconstruction. Smaller sketches show slightly lower accuracy and greater variability, while larger sketches achieve accuracy comparable to exact reconstruction. These results demonstrate the trade-off between server-side efficiency and approximation accuracy, with smaller sketches offering reduced computation at a modest accuracy cost. The guarantee in Remark~\ref{remark:randomized_reconstruction} holds for $p\geq2$, while $p\in\{0,1\}$ are included as empirical minimum-sketch settings.

\subsection{Ablation Study.}

\begin{table}[t]
\centering
\caption{Component ablation results on MNLI with $N=3$ clients and $\beta=0.5$.}
\label{tab:ablation_results}

\begin{small}
\begin{tabular}{@{}lcccc@{}}
\toprule
\multirow{2}{*}{\textbf{Method}}
& \textbf{Local}
& \textbf{Local}
& \multirow{2}{*}{\textbf{Regularization}}
& \multirow{2}{*}{\textbf{Accuracy}} \\
& \textbf{Head}
& \textbf{Preservation}
& & \\
\midrule

FedIT
& $\times$
& $\times$
& $\times$
& $0.869{\pm}0.001$ \\

FedIT
& $\checkmark$
& $\times$
& $\times$
& $0.869{\pm}0.000$ \\

FlexLoRA
& $\times$
& $\times$
& $\times$
& $0.869{\pm}0.002$ \\

FlexLoRA
& $\checkmark$
& $\times$
& $\times$
& $0.870{\pm}0.001$ \\

\midrule

Ours
& $\times$
& $\times$
& $\bm B\bm A$
& $0.867{\pm}0.003$ \\

Ours
& $\checkmark$
& $\times$
& $\bm B\bm A$
& $0.877{\pm}0.002$ \\

Ours
& $\times$
& $\checkmark$
& $\bm B\bm A$
& $0.871{\pm}0.002$ \\

Ours
& $\checkmark$
& $\checkmark$
& $\bm B$
& $0.875{\pm}0.021$ \\

Ours
& $\checkmark$
& $\checkmark$
& $\bm A$
& $0.879{\pm}0.010$ \\

Ours
& $\checkmark$
& $\checkmark$
& $\bm B,\bm A$
& \underline{$0.896{\pm}0.008$} \\

Ours
& $\checkmark$
& $\checkmark$
& $\bm B\bm A$
& $\mathbf{0.906{\pm}0.003}$ \\

\bottomrule
\end{tabular}
\end{small}
\end{table}

Table~\ref{tab:ablation_results} shows that applying either a local head or regularization alone provides little or no improvement. Combining the local head with product-space regularization raises accuracy to $0.877$, while additionally preserving the local factors further improves it to $0.906$. Product-space regularization performs best among the variants and, unlike factor-space regularization, supports $R_i\neq r_i$, enabling independent control of client-specific computation and communication budgets. The corresponding objectives are provided in Appendix~\ref{app:ablation_objectives}.
\section{Conclusion}

We proposed FedPA-LoRA, a federated LoRA framework that preserves local factors across rounds, aligns their products with global references to mitigate client drift, and aggregates local updates in the product space. This design jointly mitigates aggregation and factor-level initialization mismatches while supporting client-specific computation and communication budgets. We establish convergence bounds for homogeneous and heterogeneous ranks and demonstrate consistent gains on natural language understanding and generation tasks under data and resource heterogeneity.

\bibliographystyle{unsrtnat}
\bibliography{references}  %%% Uncomment this line and comment out the ``thebibliography'' section below to use the external .bib file (using bibtex) .





% \end{thebibliography}

\appendix
\onecolumn
\section*{Appendices}

The appendices are organized as follows.
\begin{itemize}
    \item[\textbf{\ref{appendix:additional_related_work}}] Additional discussion of proximal and personalized federated learning and federated LoRA aggregation methods.
    \item[\textbf{\ref{app:algorithm}}] Detailed pseudocode of FedPA-LoRA.
    \item[\textbf{\ref{app:exp_details}}] Experimental details, hyperparameters, and additional results across ranks, generation tasks, client participation, rank decoupling, and computer vision.
    \item[\textbf{\ref{appendix:convergence_analysis_all}}] Full convergence analyses for homogeneous and heterogeneous ranks.
    \item[\textbf{\ref{app:partial_participation}}] Convergence analysis under partial client participation.
    \item[\textbf{\ref{app:complexity}}] Server- and client-side complexity analyses, including randomized reconstruction, efficient product regularization, and wall-clock evaluations.
    \item[\textbf{\ref{app:adaptive_regularization}}] Extensions to adaptive regularization.
\end{itemize}

\section{Additional Discussion of Related Work}
\label{appendix:additional_related_work}

\subsection{Proximal and Personalized Federated Learning}

FedProx \citep{MLSYS2020_1f5fe839} initializes local optimization from the global model each round and constrains parameter deviations through a proximal term, $\frac{\mu}{2}\|w-w_g\|_2^2$. A direct extension to LoRA would penalize deviations in the LoRA factor parameters as $\|\bm B_i-\bm B_g\|_{\mathrm F}^2+\|\bm A_i-\bm A_g\|_{\mathrm F}^2$. Ditto \citep{li2021ditto} instead trains separate global and personalized models, with the former learned through federated optimization and the latter regularized toward it in parameter space.

In contrast, FedPA-LoRA preserves locally optimized LoRA factors across rounds rather than reinitializing them from the global adapter. The global adapter serves only as a product-space regularization reference, without overwriting local factors. These local factors are uploaded for product-space aggregation to reconstruct the next global adapter, without a separate global-training branch.

\subsection{Personalized Federated LoRA}

FedDPA-LoRA \citep{yang2024dualpersonalizing} maintains a globally aggregated LoRA adapter alongside a personalized local adapter. FedDPA-F initializes the personalized adapter from the final global adapter and subsequently fine-tunes it locally, whereas FedDPA-T preserves its personalized adapter across communication rounds. However, only the global adapter participates in federation and is updated through conventional factor-wise averaging, leaving its aggregation mismatch unresolved.

FedSA-LoRA \citep{guo2025selective} instead exploits the asymmetry between LoRA factors by treating $\bm A$ as task-general and $\bm B$ as client-specific. It aggregates only $\bm A$ while preserving $\bm B$ locally, avoiding factor-wise aggregation mismatch while retaining client-specific information. In contrast, FedPA-LoRA preserves both local factors while aggregating their joint product.

\subsection{High-Communication Aggregation}

FedEx-LoRA \citep{singhal2024fedex} achieves exact aggregation by absorbing the residual between the average local product and the product of the factor-wise averages into the shared frozen pre-trained weights, requiring additional residual communication and shared-weight updates.

FLoRA \citep{3737916.3738624} eliminates aggregation mismatch by stacking local factors and naturally supports heterogeneous ranks. However, the global rank grows with the sum of participating client ranks, increasing communication cost, and the local factors are reinitialized in each round. FedPA-LoRA instead preserves client-specific factors and aggregates heterogeneous-rank updates without additional communication or modifying the shared backbone.

\section{FedPA-LoRA: Detailed Algorithm}
\label{app:algorithm}

The FedPA-LoRA framework is summarized in Algorithm~\ref{alg:fedpa_lora}.

\begin{algorithm}[H]
\small
\setlength{\abovedisplayskip}{3pt}
\setlength{\belowdisplayskip}{3pt}

\caption{FedPA-LoRA}
\label{alg:fedpa_lora}

\textbf{Input}: Number of clients $N$, learning rate $\eta$, number of local updates $\tau$, computation ranks $\{r_i\}_{i=1}^{N}$, communication ranks $\{R_i\}_{i=1}^{N}$, rounds $T$, regularization weight $\lambda$, and oversampling parameter $p$ for randomized reconstruction\\
\textbf{Output}: Global LoRA factors $(\bm B_g^{(T)},\bm A_g^{(T)})$

\begin{algorithmic}[1]

\STATE
$r_g\leftarrow\max_{1\leq i\leq N}\{r_i,R_i\}$,
$\quad R_g\leftarrow\max_{1\leq i\leq N}R_i$

\STATE Initialize
$(\bm B_g^{(0)},\bm A_g^{(0)})$
with rank $r_g$

\FOR{$i=1,\ldots,N$}

    \STATE Send the first $\max\{r_i,R_i\}$
    components of the global factors to client $i$

    \STATE Initialize the local factors
    $(\bm B_i^{(0)},\bm A_i^{(0)})$
    and global reference
    $(\bm B_{g,i}^{(0)},\bm A_{g,i}^{(0)})$
    using the first $r_i$ and $R_i$
    components, respectively

\ENDFOR

\FOR{$t=1,\ldots,T$}

    \FOR{$i=1,\ldots,N$ \textbf{in parallel}}

        \STATE Preserve the local factors
        $(\bm B_i^{(t,0)},\bm A_i^{(t,0)})
        \leftarrow
        (\bm B_i^{(t-1)},\bm A_i^{(t-1)})$

        \STATE Obtain the rank-$R_i$ global reference
        $(\bm B_{g,i}^{(t-1)},\bm A_{g,i}^{(t-1)})$
        by retaining the first $R_i$ columns of
        $\bm B_g^{(t-1)}$ and the first $R_i$ rows of
        $\bm A_g^{(t-1)}$

        \FOR{$k=0,\ldots,\tau-1$}

            \STATE Update
            $(\bm B_i^{(t,k+1)},\bm A_i^{(t,k+1)})$
            from
            $(\bm B_i^{(t,k)},\bm A_i^{(t,k)})$
            with learning rate $\eta$ by one optimization step on
            \begin{equation*}
            \mathcal L_i^{(t)}(\bm B_i,\bm A_i)
            =
            f_i\left(\bm W_0+\bm B_i\bm A_i\right)
            +
            \frac{\lambda}{2}
            \left\|
            \bm B_i\bm A_i
            -
            \bm B_{g,i}^{(t-1)}
            \bm A_{g,i}^{(t-1)}
            \right\|_{\mathrm F}^{2}.
            \end{equation*}

        \ENDFOR

        \STATE Set
        $(\bm B_i^{(t)},\bm A_i^{(t)})
        \leftarrow
        (\bm B_i^{(t,\tau)},\bm A_i^{(t,\tau)})$

        \STATE Send
        $(\bm B_i^{(t)},\bm A_i^{(t)})$
        to the server

    \ENDFOR

    \IF{exact reconstruction is used}
    
        \STATE Construct the concatenated factors
        \begin{equation*}
        \bm B_{\mathrm{cat}}^{(t)}
        \leftarrow
        \frac{1}{\sqrt{N}}
        [\bm B_1^{(t)},\ldots,\bm B_N^{(t)}],
        \qquad
        \bm A_{\mathrm{cat}}^{(t)}
        \leftarrow
        \frac{1}{\sqrt{N}}
        [(\bm A_1^{(t)})^\top,\ldots,(\bm A_N^{(t)})^\top]^\top.
        \end{equation*}
    
        \STATE Compute the reduced QR factorizations
        \begin{equation*}
        (\bm Q_B^{(t)},\bm R_B^{(t)})
        \leftarrow
        \operatorname{QR}(\bm B_{\mathrm{cat}}^{(t)}),
        \qquad
        (\bm Q_A^{(t)},\bm R_A^{(t)})
        \leftarrow
        \operatorname{QR}((\bm A_{\mathrm{cat}}^{(t)})^\top).
        \end{equation*}
    
        \STATE Form
        $\bm H^{(t)}
        =
        \bm R_B^{(t)}(\bm R_A^{(t)})^\top$
        and compute its rank-$R_g$ truncated SVD
        \begin{equation*}
        (\bm U_{R_g}^{(t)},
        \bm\Sigma_{R_g}^{(t)},
        \bm V_{R_g}^{(t)})
        \leftarrow
        \operatorname{SVD}_{R_g}(\bm H^{(t)}).
        \end{equation*}
    
        \STATE Reconstruct the global LoRA factors
        \begin{equation*}
        \bm B_g^{(t)}
        \leftarrow
        \bm Q_B^{(t)}
        \bm U_{R_g}^{(t)}
        \bm\Sigma_{R_g}^{(t)},
        \qquad
        \bm A_g^{(t)}
        \leftarrow
        \left(
        \bm Q_A^{(t)}
        \bm V_{R_g}^{(t)}
        \right)^\top.
        \end{equation*}
    
    \ELSE
    
        \STATE Set $\ell\leftarrow R_g+p$ and sample a Gaussian matrix
        $\bm\Omega^{(t)}\in\mathbb R^{d\times\ell}$
    
        \STATE Form the randomized sketch and compute its reduced QR factorization
        \begin{equation*}
        \bm Y^{(t)}
        \leftarrow
        \frac{1}{N}
        \sum_{i=1}^{N}
        \bm B_i^{(t)}
        \left(
        \bm A_i^{(t)}
        \bm\Omega^{(t)}
        \right),
        \qquad
        (\bm Q_Y^{(t)},\bm R_Y^{(t)})
        \leftarrow
        \operatorname{QR}(\bm Y^{(t)}).
        \end{equation*}
    
        \STATE Compute the rank-$R_g$ truncated SVD in the sampled subspace
        \begin{equation*}
            (\widetilde{\bm U}_{R_g}^{(t)},
            \widetilde{\bm\Sigma}_{R_g}^{(t)},
            \widetilde{\bm V}_{R_g}^{(t)})
            \leftarrow
            \operatorname{SVD}_{R_g}
            \left(
            \frac{1}{N}
            \sum_{i=1}^{N}
            \left(
            (\bm Q_Y^{(t)})^\top
            \bm B_i^{(t)}
            \right)
            \bm A_i^{(t)}
            \right).
        \end{equation*}
    
        \STATE Reconstruct the global LoRA factors
        \begin{equation*}
        \bm B_g^{(t)}
        \leftarrow
        \bm Q_Y^{(t)}
        \widetilde{\bm U}_{R_g}^{(t)}
        \widetilde{\bm\Sigma}_{R_g}^{(t)},
        \qquad
        \bm A_g^{(t)}
        \leftarrow
        (\widetilde{\bm V}_{R_g}^{(t)})^\top.
        \end{equation*}
    
    \ENDIF
    
\ENDFOR

\end{algorithmic}
\end{algorithm}

\section{Experimental Details and Extra Results}
\label{app:exp_details}
\subsection{Experimental Details}
\label{app:experimental_details}

\begin{table*}[t]
\centering
\caption{Selected hyperparameters for experiments under homogeneous-rank settings.}
\label{tab:hyperparameters}

\renewcommand{\arraystretch}{1.15}
\setlength{\tabcolsep}{2.5pt}

\resizebox{\textwidth}{!}{%
\begin{tabular}{@{}ccccc*{8}{c}cc@{}}
\toprule
\multirow{3}{*}{\textbf{Experiment}}
& \multirow{3}{*}{\textbf{Dataset}}
& \multirow{3}{*}{\shortstack{\textbf{Clients}\\$N$ ($\rho$)}}
& \multirow{3}{*}{\shortstack{\textbf{Rank}\\$r$}}
& \multirow{3}{*}{\shortstack{\textbf{Data}\\\textbf{Distribution}\\$\beta$}}
& \multicolumn{8}{c}{\textbf{Selected Learning Rate }$\boldsymbol{\eta}$}
& \multicolumn{2}{c}{\textbf{Selected }$\boldsymbol{\lambda}$} \\
\cmidrule(lr){6-13}
\cmidrule(lr){14-15}

&
&
&
&
&
\raisebox{0.5ex}{\textbf{FedIT}}
&
\raisebox{0.5ex}{\textbf{FlexLoRA}}
& \shortstack{\textbf{FFA-}\\\textbf{LoRA}}
&
\raisebox{0.5ex}{\textbf{RoLoRA}}
& \shortstack{\textbf{FedRot-}\\\textbf{LoRA}}
& \shortstack{\textbf{FedDPA-}\\\textbf{LoRA}}
& \shortstack{\textbf{FedSA-}\\\textbf{LoRA}}
& \shortstack{\textbf{FedPA-}\\\textbf{LoRA}}
& \shortstack{\textbf{FedRot-}\\\textbf{LoRA}}
& \shortstack{\textbf{FedPA-}\\\textbf{LoRA}} \\
\midrule

\multirow{5}{*}{\shortstack{Natural Language\\Understanding}}
& SST-2
& \multirow{5}{*}{$10\,(1.0)$}
& \multirow{5}{*}{$4$}
& \multirow{5}{*}{$0.5$}
& \texttt{1e-2}
& \texttt{1e-2}
& \texttt{2e-2}
& \texttt{1e-2}
& \texttt{1e-2}
& \texttt{5e-3}
& \texttt{2e-2}
& \texttt{2e-2}
& $0.6$
& $1.0$ \\

& QNLI
& & &
& \texttt{1e-2}
& \texttt{1e-2}
& \texttt{2e-2}
& \texttt{5e-3}
& \texttt{5e-3}
& \texttt{5e-3}
& \texttt{1e-3}
& \texttt{2e-2}
& $0.2$
& $2.0$ \\

& QQP
& & &
& \texttt{1e-2}
& \texttt{5e-3}
& \texttt{2e-2}
& \texttt{2e-2}
& \texttt{5e-3}
& \texttt{5e-3}
& \texttt{2e-2}
& \texttt{2e-2}
& $0.6$
& $2.0$ \\

& RTE
& & &
& \texttt{1e-2}
& \texttt{1e-2}
& \texttt{2e-2}
& \texttt{5e-3}
& \texttt{1e-2}
& \texttt{2e-2}
& \texttt{1e-3}
& \texttt{2e-2}
& $0.4$
& $1.0$ \\

& MNLI
& & &
& \texttt{1e-2}
& \texttt{5e-3}
& \texttt{2e-2}
& \texttt{5e-3}
& \texttt{1e-2}
& \texttt{1e-2}
& \texttt{1e-2}
& \texttt{2e-2}
& $0.4$
& $2.0$ \\

\midrule

\multirow{3}{*}{\shortstack{Data\\Heterogeneity}}
& \multirow{3}{*}{MNLI}
& \multirow{3}{*}{$3\,(1.0)$}
& \multirow{3}{*}{$4$}
& $100$
& \texttt{1e-2}
& \texttt{5e-3}
& \texttt{2e-2}
& \texttt{1e-2}
& \texttt{1e-2}
& \texttt{1e-2}
& \texttt{5e-3}
& \texttt{2e-2}
& $0.8$
& $1.0$ \\

&
&
&
&
$1$
& \texttt{5e-3}
& \texttt{2e-2}
& \texttt{2e-2}
& \texttt{5e-3}
& \texttt{1e-2}
& \texttt{2e-2}
& \texttt{5e-3}
& \texttt{2e-2}
& $0.8$
& $1.0$ \\

&
&
&
&
$0.1$
& \texttt{1e-2}
& \texttt{1e-2}
& \texttt{2e-2}
& \texttt{5e-3}
& \texttt{1e-2}
& \texttt{1e-2}
& \texttt{1e-2}
& \texttt{2e-2}
& $0.4$
& $1.0$ \\

\midrule

\multirow{2}{*}{\shortstack{Number of\\Clients}}
& \multirow{2}{*}{MNLI}
& $3\,(1.0)$
& \multirow{2}{*}{$4$}
& \multirow{2}{*}{$0.5$}
& \texttt{1e-2}
& \texttt{1e-2}
& \texttt{2e-2}
& \texttt{1e-2}
& \texttt{1e-2}
& \texttt{1e-2}
& \texttt{1e-2}
& \texttt{2e-2}
& $0.4$
& $5.0$ \\

&
&
$50\,(1.0)$
&
&
& \texttt{2e-2}
& \texttt{5e-3}
& \texttt{2e-2}
& \texttt{5e-3}
& \texttt{2e-2}
& \texttt{1e-2}
& \texttt{1e-2}
& \texttt{1e-2}
& $0.4$
& $5.0$ \\

\midrule

\multirow{4}{*}{\shortstack{Client\\Participation}}
& \multirow{4}{*}{MNLI}
& $50\,(0.1)$
& \multirow{4}{*}{$4$}
& \multirow{4}{*}{$0.5$}
& \texttt{5e-3}
& \texttt{5e-3}
& \texttt{2e-2}
& \texttt{1e-2}
& \texttt{2e-2}
& \texttt{5e-3}
& \texttt{1e-2}
& \texttt{2e-2}
& $0.2$
& $10.0$ \\

&
&
$50\,(0.3)$
&
&
& \texttt{2e-2}
& \texttt{5e-3}
& \texttt{2e-2}
& \texttt{1e-2}
& \texttt{2e-2}
& \texttt{1e-2}
& \texttt{2e-2}
& \texttt{2e-2}
& $0.4$
& $5.0$ \\

&
&
$50\,(0.5)$
&
&
& \texttt{2e-2}
& \texttt{5e-3}
& \texttt{2e-2}
& \texttt{1e-2}
& \texttt{2e-2}
& \texttt{5e-3}
& \texttt{1e-2}
& \texttt{1e-2}
& $0.2$
& $5.0$ \\

&
&
$100\,(0.3)$
&
&
& \texttt{2e-2}
& \texttt{5e-3}
& \texttt{2e-2}
& \texttt{1e-2}
& \texttt{2e-2}
& \texttt{1e-2}
& \texttt{2e-2}
& \texttt{2e-2}
& $0.4$
& $10.0$ \\

\midrule

\multirow{4}{*}{LoRA Rank}
& \multirow{4}{*}{MNLI}
& \multirow{4}{*}{$3\,(1.0)$}
& $2$
& \multirow{4}{*}{$0.5$}
& \texttt{1e-2}
& \texttt{1e-2}
& \texttt{2e-2}
& \texttt{5e-3}
& \texttt{1e-2}
& \texttt{5e-3}
& \texttt{5e-3}
& \texttt{2e-2}
& $0.4$
& $5.0$ \\

&
&
&
$8$
&
& \texttt{5e-3}
& \texttt{5e-3}
& \texttt{1e-2}
& \texttt{5e-3}
& \texttt{1e-2}
& \texttt{1e-2}
& \texttt{1e-2}
& \texttt{2e-2}
& $0.4$
& $2.0$ \\

&
&
&
$16$
&
& \texttt{5e-3}
& \texttt{1e-2}
& \texttt{2e-2}
& \texttt{5e-3}
& \texttt{1e-2}
& \texttt{5e-3}
& \texttt{5e-3}
& \texttt{2e-2}
& $0.4$
& $1.0$ \\

&
&
&
$24$
&
& \texttt{1e-2}
& \texttt{5e-3}
& \texttt{2e-2}
& \texttt{5e-3}
& \texttt{1e-2}
& \texttt{5e-3}
& \texttt{5e-3}
& \texttt{2e-2}
& $0.4$
& $2.0$ \\

\midrule

\multirow{2}{*}{\shortstack{Natural Language\\Generation}}
& GSM8K
& $3\,(1.0)$
& $8$
& IID
& \texttt{5e-3}
& \texttt{5e-3}
& \texttt{5e-3}
& \texttt{1e-3}
& \texttt{5e-3}
& --
& --
& \texttt{5e-3}
& $0.2$
& $1.0$ \\

& HumanEval
& $6\,(1.0)$
& $8$
& Non-IID
& \texttt{5e-3}
& \texttt{5e-3}
& \texttt{5e-3}
& \texttt{5e-3}
& \texttt{5e-3}
& --
& --
& \texttt{5e-3}
& $0.2$
& $1.0$ \\

\midrule

\multirow{2}{*}{\shortstack{Computer Vision}}
& CIFAR-10
& $10\,(1.0)$
& $4$
& $0.1$
& \texttt{2e-2}
& \texttt{2e-2}
& \texttt{2e-2}
& \texttt{2e-2}
& \texttt{2e-2}
& \texttt{2e-2}
& \texttt{2e-2}
& \texttt{2e-2}
& $0.4$
& $0.1$ \\

& CIFAR-100
& $10\,(1.0)$
& $4$
& $0.1$
& \texttt{2e-2}
& \texttt{2e-2}
& \texttt{2e-2}
& \texttt{2e-2}
& \texttt{2e-2}
& \texttt{2e-2}
& \texttt{2e-2}
& \texttt{2e-2}
& $0.4$
& $0.1$ \\

\bottomrule
\end{tabular}%
}
\end{table*}

\begin{table*}[t]
\centering
\caption{Selected hyperparameters for experiments under heterogeneous-rank settings.}
\label{tab:heterogeneous_hyperparameters}

\setlength{\tabcolsep}{2.5pt}
\renewcommand{\arraystretch}{1.15}

\resizebox{\textwidth}{!}{%
\begin{tabular}{@{}ccccc*{5}{c}c@{}}
\toprule
\multirow{3}{*}{\textbf{Experiment}}
& \multirow{3}{*}{\textbf{Dataset}}
& \multirow{3}{*}{\shortstack{\textbf{Clients}\\$N$ ($\rho$)}}
& \multirow{3}{*}{\shortstack{\textbf{Rank}\\\textbf{Configuration}}}
& \multirow{3}{*}{\shortstack{\textbf{Data}\\\textbf{Distribution}\\$\beta$}}
& \multicolumn{5}{c}{\textbf{Selected Learning Rate }$\boldsymbol{\eta}$}
& \multicolumn{1}{c}{\textbf{Selected }$\boldsymbol{\lambda}$} \\
\cmidrule(lr){6-10}
\cmidrule(lr){11-11}

&
&
&
&
&
\raisebox{0.5ex}{\textbf{FlexLoRA}}
&
\raisebox{0.5ex}{\textbf{HetLoRA}}
&
\raisebox{0.5ex}{\textbf{Ravan}}
& \shortstack{\textbf{Fed-}\\\textbf{PLoRA}}
& \shortstack{\textbf{FedPA-}\\\textbf{LoRA}}
& \shortstack{\textbf{FedPA-}\\\textbf{LoRA}} \\
\midrule

\multirow{5}{*}{\shortstack{Natural Language\\Understanding}}
& SST-2
& \multirow{5}{*}{$20\,(1.0)$}
& \multirow{5}{*}{\shortstack[l]{Other methods: $r_i\in\{2,4,16\}$\\Ravan: $h_i\in\{1,2,8\}$, $r_h=64$}}
& \multirow{5}{*}{$0.5$}
& \texttt{1e-2}
& \texttt{2e-2}
& \texttt{2e-2}
& \texttt{1e-2}
& \texttt{1e-2}
& $2.0$ \\

& QNLI
&
&
&
& \texttt{1e-2}
& \texttt{1e-2}
& \texttt{2e-2}
& \texttt{1e-2}
& \texttt{1e-2}
& $2.0$ \\

& QQP
&
&
&
& \texttt{1e-2}
& \texttt{1e-2}
& \texttt{1e-2}
& \texttt{5e-3}
& \texttt{2e-2}
& $1.0$ \\

& RTE
&
&
&
& \texttt{1e-2}
& \texttt{1e-2}
& \texttt{1e-2}
& \texttt{1e-2}
& \texttt{1e-2}
& $1.0$ \\

& MNLI
&
&
&
& \texttt{5e-3}
& \texttt{2e-2}
& \texttt{1e-2}
& \texttt{5e-3}
& \texttt{2e-2}
& $1.0$ \\

\midrule

\multirow{2}{*}{\shortstack{Natural Language\\Generation}}
& GSM8K
& $3\,(1.0)$
& \multirow{2}{*}{\shortstack[l]{Other methods: $r_i\in\{2,4,16\}$\\Ravan: $h_i\in\{1,2,8\}$, $r_h=115$}}
& IID
& \texttt{1e-3}
& \texttt{1e-3}
& \texttt{5e-3}
& \texttt{5e-3}
& \texttt{5e-3}
& $0.1$ \\

& HumanEval
& $6\,(1.0)$
&
& Non-IID
& \texttt{1e-3}
& \texttt{5e-3}
& \texttt{5e-3}
& \texttt{5e-3}
& \texttt{5e-3}
& $0.1$ \\

\bottomrule
\end{tabular}%
}
\end{table*}

All experiments are implemented using FederatedScope-LLM \citep{kuang2024federatedscope} and PyTorch, and conducted on Ubuntu 18.04.5 LTS with AMD EPYC 7543 processors, 1 TiB of system memory, and NVIDIA RTX A6000 or RTX 6000 Ada GPUs. We report the mean and standard deviation over three random seeds. Tables~\ref{tab:hyperparameters} and \ref{tab:heterogeneous_hyperparameters} report the configurations selected for the homogeneous- and heterogeneous-rank settings, respectively, based on validation performance.

We search the learning rate over $\eta\in\{\texttt{5e-4},\texttt{1e-3},\texttt{5e-3},\texttt{1e-2},\texttt{2e-2}\}$. For FedRot-LoRA, we use the alignment strength reported in the original paper when available and otherwise search $\lambda\in\{0.2,0.4,0.6,0.8,1.0\}$. For FedPA-LoRA, we search the product-guided regularization strength over $\lambda\in\{0.1,1.0,2.0,5.0,10.0\}$.

In the heterogeneous-rank experiments, FlexLoRA, HetLoRA, and FedPA-LoRA use client ranks $r_i\in\{2,4,16\}$, while Fed-PLoRA uses the corresponding numbers of parallel rank-one modules. Ravan uses $h_i\in\{1,2,8\}$ active heads, with a per-head rank of $r_h=64$ for RoBERTa-Large and $r_h=115$ for Llama 3-8B. These configurations are chosen to approximately match the trainable parameter budgets of the corresponding resource groups across methods. All remaining training settings follow those described in the main text.

\subsection{Local Objectives of the Ablation Variants}
\label{app:ablation_objectives}

The ablation variants use the same task loss and server-side aggregation as FedPA-LoRA, but differ in the regularization space. Their local objectives are defined as
\begin{equation}
\begin{aligned}
\mathcal L_{i,\bm B}^{(t)} &= f_i(\bm W_0+\bm B_i\bm A_i) + \frac{\lambda}{2}\|\bm B_i-\bm B_{g,i}^{(t-1)}\|_{\mathrm F}^{2},\\
\mathcal L_{i,\bm A}^{(t)} &= f_i(\bm W_0+\bm B_i\bm A_i) + \frac{\lambda}{2}\|\bm A_i-\bm A_{g,i}^{(t-1)}\|_{\mathrm F}^{2},\\
\mathcal L_{i,\bm B,\bm A}^{(t)} &= f_i(\bm W_0+\bm B_i\bm A_i) + \frac{\lambda}{2}\left(\|\bm B_i-\bm B_{g,i}^{(t-1)}\|_{\mathrm F}^{2}+\|\bm A_i-\bm A_{g,i}^{(t-1)}\|_{\mathrm F}^{2}\right),\\
\mathcal L_{i,\bm B\bm A}^{(t)} &= f_i(\bm W_0+\bm B_i\bm A_i) + \frac{\lambda}{2}\left\|\bm B_i\bm A_i-\bm B_{g,i}^{(t-1)}\bm A_{g,i}^{(t-1)}\right\|_{\mathrm F}^{2}.
\end{aligned}
\label{eq:ablation_objectives}
\end{equation}

The first two variants regularize $\bm B$ or $\bm A$ individually, while the third applies FedProx-style proximal regularization jointly to both factors. FedPA-LoRA instead regularizes their product toward the global reference.

\subsection{Extra Results}
\label{app:extra_results}
We report additional quantitative results, including MNLI performance across LoRA ranks and generation results on GSM8K and HumanEval under homogeneous- and heterogeneous-rank settings.

% number of rank
\begin{table}[t]
\centering
\caption{MNLI accuracy across LoRA ranks $r\in\{2,4,8,16,24\}$ with $\beta=0.5$ and $N=3$ clients.}
\label{tab:rank_results}
{\small
\setlength{\tabcolsep}{2.0mm}
\begin{tabular}{@{}lccccc@{}}
\toprule
\textbf{Method}
& $\bm r=\mathbf{2}$
& $\bm r=\mathbf{4}$
& $\bm r=\mathbf{8}$
& $\bm r=\mathbf{16}$
& $\bm r=\mathbf{24}$ \\
\midrule

FedIT
& $0.869{\pm}0.002$
& $0.869{\pm}0.001$
& $0.863{\pm}0.001$
& $0.865{\pm}0.001$
& \underline{$0.873{\pm}0.000$} \\

FlexLoRA
& \underline{$0.876{\pm}0.004$}
& $0.869{\pm}0.002$
& $0.866{\pm}0.001$
& $0.839{\pm}0.032$
& $0.868{\pm}0.001$ \\

FFA-LoRA
& $0.837{\pm}0.013$
& $0.856{\pm}0.002$
& $0.845{\pm}0.002$
& $0.868{\pm}0.001$
& $0.871{\pm}0.002$ \\

RoLoRA
& $0.849{\pm}0.005$
& $0.862{\pm}0.000$
& $0.856{\pm}0.002$
& $0.858{\pm}0.002$
& $0.857{\pm}0.004$ \\

FedRot-LoRA
& $0.869{\pm}0.001$
& $0.868{\pm}0.002$
& $0.869{\pm}0.002$
& $0.870{\pm}0.001$
& $0.872{\pm}0.002$ \\

FedDPA-LoRA
& $0.871{\pm}0.001$
& \underline{$0.879{\pm}0.002$}
& $0.874{\pm}0.005$
& \underline{$0.877{\pm}0.001$}
& $0.868{\pm}0.012$ \\

FedSA-LoRA
& $0.859{\pm}0.007$
& $0.874{\pm}0.002$
& \underline{$0.878{\pm}0.001$}
& $0.875{\pm}0.002$
& $0.713{\pm}0.232$ \\

FedPA-LoRA
& $\mathbf{0.907{\pm}0.002}$
& $\mathbf{0.906{\pm}0.003}$
& $\mathbf{0.907{\pm}0.001}$
& $\mathbf{0.908{\pm}0.001}$
& $\mathbf{0.910{\pm}0.003}$ \\

\bottomrule
\end{tabular}
}
\end{table}

\begin{table}[t]
\centering
\begin{minipage}[t]{0.49\textwidth}
\centering
\captionof{table}{Natural language generation performance under the homogeneous-rank setting with $r=8$. GSM8K uses $N=3$ IID clients, while HumanEval uses $N=6$ clients partitioned by programming language.}
\label{tab:nlg_results}
{\small
\renewcommand{\arraystretch}{0.95}
\setlength{\tabcolsep}{2.2mm}
\begin{tabular}{@{}lcc@{}}
\toprule
\multirow{2}{*}{\textbf{Method}}
& \textbf{GSM8K}
& \textbf{HumanEval} \\
& \textbf{(Acc.)}
& \textbf{(pass@1)} \\
\midrule
FedIT
& $0.4316{\pm}0.003$
& \underline{$0.4020{\pm}0.002$} \\
FlexLoRA
& \underline{$0.4460{\pm}0.004$}
& $0.3873{\pm}0.010$ \\
FFA-LoRA
& $0.4316{\pm}0.003$
& $0.3902{\pm}0.008$ \\
RoLoRA
& $0.4233{\pm}0.016$
& $0.3751{\pm}0.003$ \\
FedRot-LoRA
& $0.4412{\pm}0.014$
& $0.4016{\pm}0.007$ \\
FedPA-LoRA
& $\mathbf{0.4672{\pm}0.006}$
& $\mathbf{0.4121{\pm}0.002}$ \\
\bottomrule
\end{tabular}
}
\end{minipage}
\hfill
\begin{minipage}[t]{0.49\textwidth}
\centering
\captionof{table}{Natural language generation performance under the heterogeneous-rank setting. GSM8K uses $N=3$ IID clients, while HumanEval uses $N=6$ clients partitioned by programming language. Clients are divided into three resource groups.}
\label{tab:hetero_nlg_results}
{\small
\renewcommand{\arraystretch}{0.95}
\setlength{\tabcolsep}{2.2mm}
\begin{tabular}{@{}lcc@{}}
\toprule
\multirow{2}{*}{\textbf{Method}}
& \textbf{GSM8K}
& \textbf{HumanEval} \\
& \textbf{(Acc.)}
& \textbf{(pass@1)} \\
\midrule
FlexLoRA
& $0.4488{\pm}0.013$
& $0.4012{\pm}0.070$ \\
HetLoRA
& $0.4230{\pm}0.001$
& $0.3894{\pm}0.007$ \\
Ravan
& \underline{$0.4735{\pm}0.005$}
& $0.3959{\pm}0.009$ \\
Fed-PLoRA
& $0.4319{\pm}0.018$
& \underline{$0.4016{\pm}0.021$} \\
FedPA-LoRA
& $\mathbf{0.5018{\pm}0.017}$
& $\mathbf{0.4048{\pm}0.001}$ \\
\bottomrule
\end{tabular}
}
\end{minipage}
\end{table}

\subsubsection{Effect of Client Participation Rate}
\label{app:client_participation_rate}

\begin{table}[t]
\centering
\caption{MNLI accuracy across different numbers of clients $N$ and participation rates $\rho$, with $r=4$ and $\beta=0.5$.}
\label{tab:participation_rate_results}
{\small
\setlength{\tabcolsep}{7pt}
\renewcommand{\arraystretch}{1.15}
\begin{tabular}{@{}lccccc@{}}
\toprule
\multirow{2}{*}{\textbf{Method}}
& \multicolumn{4}{c}{$\bm{N=50}$}
& \multicolumn{1}{c}{$\bm{N=100}$} \\
\cmidrule(lr){2-5}
\cmidrule(lr){6-6}
& $\bm{\rho=0.1}$
& $\bm{\rho=0.3}$
& $\bm{\rho=0.5}$
& $\bm{\rho=1.0}$
& $\bm{\rho=0.3}$ \\
\midrule

FedIT
& $0.863{\pm}0.002$
& $0.879{\pm}0.001$
& $0.878{\pm}0.001$
& $0.877{\pm}0.001$
& $0.878{\pm}0.001$ \\

FlexLoRA
& $0.737{\pm}0.176$
& $0.841{\pm}0.012$
& $0.840{\pm}0.015$
& $0.828{\pm}0.025$
& $0.835{\pm}0.029$ \\

FFA-LoRA
& $0.858{\pm}0.003$
& $0.861{\pm}0.002$
& $0.861{\pm}0.001$
& $0.860{\pm}0.001$
& $0.860{\pm}0.002$ \\

RoLoRA
& $0.858{\pm}0.020$
& $0.878{\pm}0.001$
& $0.883{\pm}0.002$
& $0.882{\pm}0.002$
& $0.881{\pm}0.007$ \\

FedRot-LoRA
& $0.876{\pm}0.003$
& $0.878{\pm}0.001$
& $0.879{\pm}0.002$
& $0.878{\pm}0.001$
& $0.878{\pm}0.001$ \\

FedDPA-LoRA
& $0.875{\pm}0.001$
& \underline{$0.889{\pm}0.002$}
& \underline{$0.890{\pm}0.002$}
& $0.883{\pm}0.014$
& $0.879{\pm}0.011$ \\

FedSA-LoRA
& \underline{$0.877{\pm}0.003$}
& \underline{$0.889{\pm}0.001$}
& $0.889{\pm}0.002$
& \underline{$0.891{\pm}0.000$}
& \underline{$0.891{\pm}0.001$} \\

FedPA-LoRA
& $\mathbf{0.884{\pm}0.008}$
& $\mathbf{0.914{\pm}0.001}$
& $\mathbf{0.913{\pm}0.001}$
& $\mathbf{0.916{\pm}0.001}$
& $\mathbf{0.905{\pm}0.001}$ \\

\bottomrule
\end{tabular}
}
\end{table}

Table~\ref{tab:participation_rate_results} reports MNLI performance for $N=50$ clients across $\rho\in\{0.1,0.3,0.5,1.0\}$, with an additional evaluation at $N=100$ and $\rho=0.3$. FedPA-LoRA consistently performs best across all settings, demonstrating robustness to both partial participation and larger client populations.

\subsubsection{Effect of Decoupling Local and Reference Ranks.}
\label{app:rank_decoupling}

\begin{table}[t]
\centering
\begin{minipage}[t]{0.49\textwidth}
\vspace{0pt}
\centering
\captionof{table}{MNLI accuracy across reference-rank configurations $R_i$, with fixed local ranks $r_i=\{1,2,4\}$, $N=3$, and $\beta=0.5$. Downlink cost is based on reference parameter counts and normalized to $R_i=\{4,4,4\}$ (100\%).}
\label{tab:rank_decoupling}
{\small
\setlength{\tabcolsep}{6pt}
\renewcommand{\arraystretch}{1.0}
\begin{tabular}{@{}lccc@{}}
\toprule
\shortstack{\textbf{Reference}\\\textbf{Budget}}
& \shortstack{\textbf{Reference}\\\textbf{Ranks} $R_i$}
& \shortstack{\textbf{Downlink}\\\textbf{Cost (\%)}}
& \textbf{Accuracy} \\
\midrule
\multirow{2}{*}{Low}
& $\{1,1,2\}$ & 33.3 & 0.907 \\
& $\{2,1,1\}$ & 33.3 & 0.905 \\
\midrule
\multirow{3}{*}{Medium}
& $\{1,2,4\}$ & 58.3 & 0.908 \\
& $\{4,2,1\}$ & 58.3 & 0.906 \\
& $\{2,2,2\}$ & 50.0 & 0.910 \\
\midrule
\multirow{3}{*}{High}
& $\{2,4,4\}$ & 83.3 & 0.912 \\
& $\{4,4,2\}$ & 83.3 & 0.910 \\
& $\{4,4,4\}$ & 100.0 & 0.910 \\
\bottomrule
\end{tabular}
}
\end{minipage}
\hfill
\begin{minipage}[t]{0.49\textwidth}
\vspace{0pt}
\centering
\captionof{table}{Image classification accuracy on CIFAR-10 and CIFAR-100 under the homogeneous-rank setting with $N=10$, $r=4$, and $\beta=0.1$.}
\label{tab:cv_results}
{\small
\setlength{\tabcolsep}{6pt}
\renewcommand{\arraystretch}{1.0}
\begin{tabular}{@{}lcc@{}}
\toprule
\textbf{Method}
& \textbf{CIFAR-10}
& \textbf{CIFAR-100} \\
\midrule
FedIT
& $0.977{\pm}0.000$
& $0.855{\pm}0.001$ \\
FlexLoRA
& $0.977{\pm}0.001$
& $0.852{\pm}0.001$ \\
FFA-LoRA
& $0.977{\pm}0.000$
& $0.855{\pm}0.001$ \\
RoLoRA
& $0.978{\pm}0.001$
& $0.866{\pm}0.001$ \\
FedRot-LoRA
& $0.978{\pm}0.001$
& $0.855{\pm}0.001$ \\
FedDPA-LoRA
& $0.977{\pm}0.000$
& $0.853{\pm}0.002$ \\
FedSA-LoRA
& \underline{$0.984{\pm}0.001$}
& \underline{$0.909{\pm}0.001$} \\
FedPA-LoRA
& $\mathbf{0.987{\pm}0.001}$
& $\mathbf{0.916{\pm}0.001}$ \\
\bottomrule
\end{tabular}
}
\end{minipage}
\end{table}

Table~\ref{tab:rank_decoupling} evaluates different reference-rank configurations $R_i$ with fixed heterogeneous local ranks $r_i\in\{1,2,4\}$. Overall, accuracy generally improves with larger reference-rank budgets. Allocating reference ranks in reverse order of local ranks results in slightly lower accuracy than maintaining a similar rank ordering. Nevertheless, FedPA-LoRA maintains stable performance across diverse configurations, demonstrating robustness under decoupled local and reference ranks and enabling flexible adjustment of computation and reference-rank budgets to accommodate heterogeneous client resources.

\subsubsection{Computer Vision}
\label{app:computer_vision}

We further evaluate FedPA-LoRA on CIFAR-10 and CIFAR-100 using ViT-B/16 \citep{dosovitskiy2021an}. We use $N=10$ clients with homogeneous rank $r=4$, Dirichlet parameter $\beta=0.1$, and $50$ local updates per round. As shown in Table~\ref{tab:cv_results}, FedPA-LoRA consistently outperforms the baselines on both datasets, demonstrating that its effectiveness extends beyond language tasks to computer vision.

\section{Convergence Analysis}
\label{appendix:convergence_analysis_all}

\subsection{Convergence Analysis: Homogeneous Setting}
\label{appendix:convergence_analysis}

We provide the complete proofs of Theorem~\ref{theorem:conv_fedpa} and Proposition~\ref{proposition:conv_global_loss} under the four assumptions stated in Section~\ref{subsec:convergence_analysis}. Throughout this section, we consider homogeneous local and reference ranks $r_i=R_i=r$, full client participation, and $\lambda>0$. Unless stated otherwise, expectations at each step are taken conditionally on the current iterates.

\subsubsection{Preliminaries}
\label{appendix:convergence_preliminaries}

We write $\bm X_i^{(t,k)}=\bm B_i^{(t,k)}\bm A_i^{(t,k)}$ and $\bm W_i^{(t,k)}=\bm W_0+\bm X_i^{(t,k)}$ for the local LoRA update and model at local step $k$ of communication round $t$, respectively.

During communication round $t$, the global reference $\bm X_g^{(t-1)}=\bm B_g^{(t-1)}\bm A_g^{(t-1)}$, reconstructed by the server in Eq.~\eqref{eq:efficient_global_lora_reconstruction}, remains fixed. The stochastic weight-space objective is
\begin{equation}
\mathcal L_i^{(t)}(\bm W;\xi_i^{(t,k)})
=
f_i(\bm W;\xi_i^{(t,k)})
+
\frac{\lambda}{2}
\left\|
\bm W-\bm W_0-\bm X_g^{(t-1)}
\right\|_{\mathrm F}^{2},
\label{eq:app_stochastic_objective}
\end{equation}
and its deterministic counterpart is
\begin{equation}
\mathcal L_i^{(t)}(\bm W)
=
f_i(\bm W)
+
\frac{\lambda}{2}
\left\|
\bm W-\bm W_0-\bm X_g^{(t-1)}
\right\|_{\mathrm F}^{2}.
\label{eq:app_deterministic_objective}
\end{equation}

At the end of communication round $t$, $\bm X_i^{(t)}=\bm X_i^{(t,\tau)}$ and $\bm W_i^{(t)}=\bm W_i^{(t,\tau)}$. We define the round-wise joint objective as
\begin{equation}
\Psi^{(t)}
=
\frac{1}{N}
\sum_{i=1}^{N}
\left[
f_i\left(\bm W_i^{(t)}\right)
+
\frac{\lambda}{2}
\left\|
\bm X_i^{(t)}-\bm X_g^{(t)}
\right\|_{\mathrm F}^{2}
\right].
\label{eq:app_joint_objective}
\end{equation}

We assume that $\Psi^\star$ is a uniform lower bound on the average task loss, i.e.,
$\frac{1}{N}\sum_{i=1}^{N} f_i(\bm W)\geq\Psi^\star$ for all $\bm W$.
We further assume that, for some constant $D>0$,
\begin{equation}
\mathbb E\left[\Psi^{(0)}\right]-\Psi^\star
\leq
D.
\label{eq:app_initial_objective_gap}
\end{equation}

Since $\bm W_i^{(t,k)}-\bm W_0=\bm X_i^{(t,k)}$, the stochastic weight-space gradient is
\begin{equation}
\nabla_{\bm W}\mathcal L_i^{(t)}
\left(\bm W_i^{(t,k)};\xi_i^{(t,k)}\right)
=
\nabla_{\bm W}f_i
\left(\bm W_i^{(t,k)};\xi_i^{(t,k)}\right)
+
\lambda\left(\bm X_i^{(t,k)}-\bm X_g^{(t-1)}\right).
\label{eq:app_regularized_gradient}
\end{equation}
Because the regularization gradient is deterministic conditioned on the current iterate, Assumption~\ref{assumption:conv_stochastic_gradients} implies
\begin{equation}
\mathbb E_{\xi_i^{(t,k)}}\left[
\nabla_{\bm W}\mathcal L_i^{(t)}
\left(\bm W_i^{(t,k)};\xi_i^{(t,k)}\right)
\right]
=
\nabla_{\bm W}\mathcal L_i^{(t)}
\left(\bm W_i^{(t,k)}\right).
\label{eq:app_unbiased_regularized_gradient}
\end{equation}

\subsubsection{Smoothness}

For any $\bm W$ and $\bm W'$, Assumption~\ref{assumption:conv_smoothness} gives
\begin{equation}
\begin{aligned}
\left\|
\nabla_{\bm W}\mathcal L_i^{(t)}(\bm W)
-
\nabla_{\bm W}\mathcal L_i^{(t)}(\bm W')
\right\|_{\mathrm F}
&\leq
\left\|
\nabla_{\bm W}f_i(\bm W)
-
\nabla_{\bm W}f_i(\bm W')
\right\|_{\mathrm F}
+
\lambda
\left\|
\bm W-\bm W'
\right\|_{\mathrm F}
\\
&\leq
(L_s+\lambda)
\left\|
\bm W-\bm W'
\right\|_{\mathrm F}.
\end{aligned}
\label{eq:app_regularized_smoothness}
\end{equation}
Thus, $\mathcal L_i^{(t)}$ is $(L_s+\lambda)$-smooth and satisfies
\begin{equation}
\mathcal L_i^{(t)}(\bm W')
\leq
\mathcal L_i^{(t)}(\bm W)
+
\left\langle
\nabla_{\bm W}\mathcal L_i^{(t)}(\bm W),
\bm W'-\bm W
\right\rangle_{\mathrm F}
+
\frac{L_s+\lambda}{2}
\left\|
\bm W'-\bm W
\right\|_{\mathrm F}^{2}.
\label{eq:app_descent_lemma}
\end{equation}

\subsubsection{Bounded Regularized Gradients}

Assumption~\ref{assumption:conv_factor_regularity} implies
\begin{equation}
\left\|
\bm X_i^{(t,k)}
\right\|_{\mathrm F}
=
\left\|
\bm B_i^{(t,k)}\bm A_i^{(t,k)}
\right\|_{\mathrm F}
\leq
\left\|
\bm B_i^{(t,k)}
\right\|_{\mathrm F}
\left\|
\bm A_i^{(t,k)}
\right\|_{\mathrm F}
\leq
C_A C_B.
\label{eq:app_local_product_bound}
\end{equation}

At initialization, all clients receive the same rank-$r$ global factors, so that $\bm X_g^{(0)}=\bm X_i^{(1,0)}$ for every client $i$. Therefore, Assumption~\ref{assumption:conv_factor_regularity} gives
\begin{equation}
\left\|
\bm X_g^{(0)}
\right\|_{\mathrm F}
=
\left\|
\bm B_i^{(1,0)}\bm A_i^{(1,0)}
\right\|_{\mathrm F}
\leq
\left\|
\bm B_i^{(1,0)}
\right\|_{\mathrm F}
\left\|
\bm A_i^{(1,0)}
\right\|_{\mathrm F}
\leq
C_A C_B.
\label{eq:app_initial_global_bound}
\end{equation}

We define $\overline{\bm X}^{(t)}=N^{-1}\sum_{i=1}^{N}\bm X_i^{(t)}$. Since $\bm X_g^{(t)}$ is the rank-$r$ truncated-SVD approximation of $\overline{\bm X}^{(t)}$,
\begin{equation}
\left\|
\bm X_g^{(t)}
\right\|_{\mathrm F}
\leq
\left\|
\overline{\bm X}^{(t)}
\right\|_{\mathrm F}
\leq
\frac{1}{N}
\sum_{i=1}^{N}
\left\|
\bm X_i^{(t)}
\right\|_{\mathrm F}
\leq
C_A C_B.
\label{eq:app_global_product_bound}
\end{equation}
Consequently,
\begin{equation}
\left\|
\bm X_i^{(t,k)}-\bm X_g^{(t-1)}
\right\|_{\mathrm F}
\leq
\left\|
\bm X_i^{(t,k)}
\right\|_{\mathrm F}
+
\left\|
\bm X_g^{(t-1)}
\right\|_{\mathrm F}
\leq
2C_A C_B.
\label{eq:app_reference_difference_bound}
\end{equation}

Under Assumption~\ref{assumption:conv_stochastic_gradients}, the stochastic task gradient satisfies
\begin{equation}
\left\|
\nabla_{\bm W}
f_i\left(
\bm W_i^{(t,k)};\xi_i^{(t,k)}
\right)
\right\|_{\mathrm F}
\leq
G_f.
\label{eq:app_task_gradient_bound}
\end{equation}
Combining Eqs.~\eqref{eq:app_regularized_gradient}, \eqref{eq:app_reference_difference_bound}, and \eqref{eq:app_task_gradient_bound} gives
\begin{equation}
\left\|
\nabla_{\bm W}\mathcal L_i^{(t)}
\left(
\bm W_i^{(t,k)};\xi_i^{(t,k)}
\right)
\right\|_{\mathrm F}
\leq
G_f+2\lambda C_A C_B.
\label{eq:app_stochastic_gradient_bound}
\end{equation}

We define
\begin{equation}
G_\lambda
=
G_f+2\lambda C_A C_B.
\label{eq:app_G_lambda}
\end{equation}
Then,
\begin{equation}
\left\|
\nabla_{\bm W}\mathcal L_i^{(t)}
\left(
\bm W_i^{(t,k)};\xi_i^{(t,k)}
\right)
\right\|_{\mathrm F}
\leq
G_\lambda,
\qquad
\left\|
\nabla_{\bm W}\mathcal L_i^{(t)}
\left(
\bm W_i^{(t,k)}
\right)
\right\|_{\mathrm F}
\leq
G_\lambda.
\label{eq:app_regularized_gradient_bounds}
\end{equation}
The second inequality follows from Eq.~\eqref{eq:app_unbiased_regularized_gradient} and Jensen's inequality.

\subsubsection{Proof of Theorem~\ref{theorem:conv_fedpa}}
\label{appendix:proof_convergence} 

\textbf{One-Step Descent.} Because $\bm W_i^{(t,k)}=\bm W_0+\bm B_i^{(t,k)}\bm A_i^{(t,k)}$, the chain rule gives
\begin{equation}
\begin{aligned}
\nabla_{\bm B}\mathcal L_i^{(t)}
(\bm W_i^{(t,k)};\xi_i^{(t,k)})
&=
\nabla_{\bm W}\mathcal L_i^{(t)}
(\bm W_i^{(t,k)};\xi_i^{(t,k)})
(\bm A_i^{(t,k)})^\top, \\
\nabla_{\bm A}\mathcal L_i^{(t)}
(\bm W_i^{(t,k)};\xi_i^{(t,k)})
&=
(\bm B_i^{(t,k)})^\top
\nabla_{\bm W}\mathcal L_i^{(t)}
(\bm W_i^{(t,k)};\xi_i^{(t,k)}).
\end{aligned}
\label{eq:app_factor_gradients}
\end{equation}
The simultaneous stochastic-gradient updates are therefore
\begin{equation}
\begin{aligned}
\bm B_i^{(t,k+1)}
&=
\bm B_i^{(t,k)}
-
\eta
\nabla_{\bm W}\mathcal L_i^{(t)}
(\bm W_i^{(t,k)};\xi_i^{(t,k)})
(\bm A_i^{(t,k)})^\top, \\
\bm A_i^{(t,k+1)}
&=
\bm A_i^{(t,k)}
-
\eta
(\bm B_i^{(t,k)})^\top
\nabla_{\bm W}\mathcal L_i^{(t)}
(\bm W_i^{(t,k)};\xi_i^{(t,k)}).
\end{aligned}
\label{eq:app_factor_updates}
\end{equation}

Since $\bm W_0$ is fixed,
\begin{equation}
\bm W_i^{(t,k+1)}-\bm W_i^{(t,k)}
=
\bm X_i^{(t,k+1)}-\bm X_i^{(t,k)}.
\label{eq:app_weight_product_increment}
\end{equation}
Expanding the product of the updated factors gives
\begin{equation}
\begin{aligned}
\bm X_i^{(t,k+1)}-\bm X_i^{(t,k)}
={}&
-\eta
\nabla_{\bm W}\mathcal L_i^{(t)}
(\bm W_i^{(t,k)};\xi_i^{(t,k)})
(\bm A_i^{(t,k)})^\top\bm A_i^{(t,k)} \\
&-
\eta
\bm B_i^{(t,k)}(\bm B_i^{(t,k)})^\top
\nabla_{\bm W}\mathcal L_i^{(t)}
(\bm W_i^{(t,k)};\xi_i^{(t,k)}) \\
&+
\eta^2
\nabla_{\bm W}\mathcal L_i^{(t)}
(\bm W_i^{(t,k)};\xi_i^{(t,k)})
(\bm A_i^{(t,k)})^\top
(\bm B_i^{(t,k)})^\top
\nabla_{\bm W}\mathcal L_i^{(t)}
(\bm W_i^{(t,k)};\xi_i^{(t,k)}).
\end{aligned}
\label{eq:app_product_increment}
\end{equation}

Applying Eq.~\eqref{eq:app_descent_lemma} with $\bm W=\bm W_i^{(t,k)}$ and $\bm W'=\bm W_i^{(t,k+1)}$ yields
\begin{equation}
\mathcal L_i^{(t)}(\bm W_i^{(t,k+1)})
\leq
\mathcal L_i^{(t)}(\bm W_i^{(t,k)})
+
\left\langle
\nabla_{\bm W}\mathcal L_i^{(t)}(\bm W_i^{(t,k)}),
\bm W_i^{(t,k+1)}-\bm W_i^{(t,k)}
\right\rangle_{\mathrm F}
+
\frac{L_s+\lambda}{2}
\left\|
\bm W_i^{(t,k+1)}-\bm W_i^{(t,k)}
\right\|_{\mathrm F}^{2}.
\label{eq:app_descent_application}
\end{equation}

Using Eq.~\eqref{eq:app_weight_product_increment}, we substitute Eq.~\eqref{eq:app_product_increment} into the inner-product and squared-increment terms of Eq.~\eqref{eq:app_descent_application}. Taking conditional expectation of the inner-product term gives
\begin{equation}
\begin{aligned}
&
\mathbb E_{\xi_i^{(t,k)}}
\left[
\left\langle
\nabla_{\bm W}\mathcal L_i^{(t)}(\bm W_i^{(t,k)}),
\bm W_i^{(t,k+1)}-\bm W_i^{(t,k)}
\right\rangle_{\mathrm F}
\right] \\
={}&
-\eta
\mathbb E_{\xi_i^{(t,k)}}
\left[
\left\langle
\nabla_{\bm W}\mathcal L_i^{(t)}(\bm W_i^{(t,k)}),
\nabla_{\bm W}\mathcal L_i^{(t)}
(\bm W_i^{(t,k)};\xi_i^{(t,k)})
(\bm A_i^{(t,k)})^\top\bm A_i^{(t,k)}
\right\rangle_{\mathrm F}
\right] \\
&-
\eta
\mathbb E_{\xi_i^{(t,k)}}
\left[
\left\langle
\nabla_{\bm W}\mathcal L_i^{(t)}(\bm W_i^{(t,k)}),
\bm B_i^{(t,k)}(\bm B_i^{(t,k)})^\top
\nabla_{\bm W}\mathcal L_i^{(t)}
(\bm W_i^{(t,k)};\xi_i^{(t,k)})
\right\rangle_{\mathrm F}
\right] \\
&+
\eta^2
\mathbb E_{\xi_i^{(t,k)}}
\left[
\left\langle
\nabla_{\bm W}\mathcal L_i^{(t)}(\bm W_i^{(t,k)}),
\nabla_{\bm W}\mathcal L_i^{(t)}
(\bm W_i^{(t,k)};\xi_i^{(t,k)})
(\bm A_i^{(t,k)})^\top
(\bm B_i^{(t,k)})^\top
\nabla_{\bm W}\mathcal L_i^{(t)}
(\bm W_i^{(t,k)};\xi_i^{(t,k)})
\right\rangle_{\mathrm F}
\right].
\end{aligned}
\label{eq:app_expected_inner_product_expansion}
\end{equation}

Using Eq.~\eqref{eq:app_unbiased_regularized_gradient} and standard properties of the Frobenius inner product, the first two terms reduce to
\begin{equation}
-\eta
\left\|
\nabla_{\bm W}\mathcal L_i^{(t)}(\bm W_i^{(t,k)})
(\bm A_i^{(t,k)})^\top
\right\|_{\mathrm F}^{2}
-
\eta
\left\|
(\bm B_i^{(t,k)})^\top
\nabla_{\bm W}\mathcal L_i^{(t)}(\bm W_i^{(t,k)})
\right\|_{\mathrm F}^{2}.
\label{eq:app_first_order_terms}
\end{equation}

For the remaining interaction term, the Cauchy--Schwarz inequality, submultiplicativity, and Eq.~\eqref{eq:app_regularized_gradient_bounds} give
\begin{equation}
\begin{aligned}
&
\left|
\left\langle
\nabla_{\bm W}\mathcal L_i^{(t)}(\bm W_i^{(t,k)}),
\nabla_{\bm W}\mathcal L_i^{(t)}
(\bm W_i^{(t,k)};\xi_i^{(t,k)})
(\bm A_i^{(t,k)})^\top
(\bm B_i^{(t,k)})^\top
\nabla_{\bm W}\mathcal L_i^{(t)}
(\bm W_i^{(t,k)};\xi_i^{(t,k)})
\right\rangle_{\mathrm F}
\right| \\
&\qquad\leq
\left\|
\nabla_{\bm W}\mathcal L_i^{(t)}(\bm W_i^{(t,k)})
\right\|_{\mathrm F}
\left\|
\nabla_{\bm W}\mathcal L_i^{(t)}
(\bm W_i^{(t,k)};\xi_i^{(t,k)})
\right\|_{\mathrm F}^{2}
\left\|
\bm A_i^{(t,k)}
\right\|_{\mathrm F}
\left\|
\bm B_i^{(t,k)}
\right\|_{\mathrm F}
\leq
C_A C_B G_\lambda^3.
\end{aligned}
\label{eq:app_interaction_bound}
\end{equation}
Combining Eqs.~\eqref{eq:app_expected_inner_product_expansion}, \eqref{eq:app_first_order_terms}, and \eqref{eq:app_interaction_bound} yields
\begin{equation}
\begin{aligned}
&\mathbb E_{\xi_i^{(t,k)}}\left[\left\langle\nabla_{\bm W}\mathcal L_i^{(t)}(\bm W_i^{(t,k)}),\bm W_i^{(t,k+1)}-\bm W_i^{(t,k)}\right\rangle_{\mathrm F}\right]\\
&\leq-\eta\left(\left\|\nabla_{\bm B}\mathcal L_i^{(t)}(\bm W_i^{(t,k)})\right\|_{\mathrm F}^{2}+\left\|\nabla_{\bm A}\mathcal L_i^{(t)}(\bm W_i^{(t,k)})\right\|_{\mathrm F}^{2}\right)+\eta^2C_AC_BG_\lambda^3.
\end{aligned}
\label{eq:app_expected_inner_product_bound}
\end{equation}

We next bound the squared weight increment. By Eq.~\eqref{eq:app_weight_product_increment},
\begin{equation}
\left\|
\bm W_i^{(t,k+1)}-\bm W_i^{(t,k)}
\right\|_{\mathrm F}^{2}
=
\left\|
\bm X_i^{(t,k+1)}-\bm X_i^{(t,k)}
\right\|_{\mathrm F}^{2}.
\label{eq:app_weight_increment_equivalence}
\end{equation}
Hence, it is sufficient to bound the three terms in Eq.~\eqref{eq:app_product_increment}.

For the first term, repeated application of Frobenius-norm submultiplicativity gives
\begin{equation}
\begin{aligned}
&
\left\|
\eta
\nabla_{\bm W}\mathcal L_i^{(t)}(\bm W_i^{(t,k)};\xi_i^{(t,k)})
(\bm A_i^{(t,k)})^\top\bm A_i^{(t,k)}
\right\|_{\mathrm F}^{2} \\
&\qquad\leq
\eta^2
\left\|
\nabla_{\bm W}\mathcal L_i^{(t)}(\bm W_i^{(t,k)};\xi_i^{(t,k)})
\right\|_{\mathrm F}^{2}
\left\|
(\bm A_i^{(t,k)})^\top
\right\|_{\mathrm F}^{2}
\left\|
\bm A_i^{(t,k)}
\right\|_{\mathrm F}^{2} \\
&\qquad=
\eta^2
\left\|
\nabla_{\bm W}\mathcal L_i^{(t)}(\bm W_i^{(t,k)};\xi_i^{(t,k)})
\right\|_{\mathrm F}^{2}
\left\|
\bm A_i^{(t,k)}
\right\|_{\mathrm F}^{4}
\leq
\eta^2C_A^4G_\lambda^2.
\end{aligned}
\label{eq:app_first_increment_term_bound}
\end{equation}

Similarly, for the second term,
\begin{equation}
\begin{aligned}
&
\left\|
\eta
\bm B_i^{(t,k)}(\bm B_i^{(t,k)})^\top
\nabla_{\bm W}\mathcal L_i^{(t)}(\bm W_i^{(t,k)};\xi_i^{(t,k)})
\right\|_{\mathrm F}^{2} \\
&\qquad\leq
\eta^2
\left\|
\bm B_i^{(t,k)}
\right\|_{\mathrm F}^{2}
\left\|
(\bm B_i^{(t,k)})^\top
\right\|_{\mathrm F}^{2}
\left\|
\nabla_{\bm W}\mathcal L_i^{(t)}(\bm W_i^{(t,k)};\xi_i^{(t,k)})
\right\|_{\mathrm F}^{2} \\
&\qquad=
\eta^2
\left\|
\bm B_i^{(t,k)}
\right\|_{\mathrm F}^{4}
\left\|
\nabla_{\bm W}\mathcal L_i^{(t)}(\bm W_i^{(t,k)};\xi_i^{(t,k)})
\right\|_{\mathrm F}^{2}
\leq
\eta^2C_B^4G_\lambda^2.
\end{aligned}
\label{eq:app_second_increment_term_bound}
\end{equation}

For the second-order interaction term,
\begin{equation}
\begin{aligned}
&
\left\|
\eta^2
\nabla_{\bm W}\mathcal L_i^{(t)}(\bm W_i^{(t,k)};\xi_i^{(t,k)})
(\bm A_i^{(t,k)})^\top
(\bm B_i^{(t,k)})^\top
\nabla_{\bm W}\mathcal L_i^{(t)}(\bm W_i^{(t,k)};\xi_i^{(t,k)})
\right\|_{\mathrm F}^{2} \\
&\qquad\leq
\eta^4
\left\|
\nabla_{\bm W}\mathcal L_i^{(t)}(\bm W_i^{(t,k)};\xi_i^{(t,k)})
\right\|_{\mathrm F}^{4}
\left\|
\bm A_i^{(t,k)}
\right\|_{\mathrm F}^{2}
\left\|
\bm B_i^{(t,k)}
\right\|_{\mathrm F}^{2}
\leq
\eta^4C_A^2C_B^2G_\lambda^4.
\end{aligned}
\label{eq:app_third_increment_term_bound}
\end{equation}

Applying $\|\bm P+\bm Q+\bm R\|_{\mathrm F}^{2}\leq3\|\bm P\|_{\mathrm F}^{2}+3\|\bm Q\|_{\mathrm F}^{2}+3\|\bm R\|_{\mathrm F}^{2}$ to Eq.~\eqref{eq:app_product_increment}, and then using Eqs.~\eqref{eq:app_first_increment_term_bound}, \eqref{eq:app_second_increment_term_bound}, and \eqref{eq:app_third_increment_term_bound}, gives
\begin{equation}
\mathbb E_{\xi_i^{(t,k)}}
\left[
\left\|
\bm W_i^{(t,k+1)}-\bm W_i^{(t,k)}
\right\|_{\mathrm F}^{2}
\right]
\leq
3\eta^2(C_A^4+C_B^4)G_\lambda^2
+
3\eta^4C_A^2C_B^2G_\lambda^4.
\label{eq:app_increment_squared_bound}
\end{equation}
For $0<\eta\leq1$, we have $\eta^4\leq\eta^2$, and hence
\begin{equation}
\mathbb E_{\xi_i^{(t,k)}}
\left[
\left\|
\bm W_i^{(t,k+1)}-\bm W_i^{(t,k)}
\right\|_{\mathrm F}^{2}
\right]
\leq
3\eta^2
\left[
(C_A^4+C_B^4)G_\lambda^2
+
C_A^2C_B^2G_\lambda^4
\right].
\label{eq:app_increment_squared_simplified}
\end{equation}

Define
\begin{equation}
M_\lambda
=
C_A C_B G_\lambda^3
+
\frac{3(L_s+\lambda)}{2}
\left[
(C_A^4+C_B^4)G_\lambda^2
+
C_A^2C_B^2G_\lambda^4
\right].
\label{eq:app_M_lambda}
\end{equation}
Substituting Eqs.~\eqref{eq:app_expected_inner_product_bound} and \eqref{eq:app_increment_squared_simplified} into Eq.~\eqref{eq:app_descent_application} gives
\begin{equation}
\mathbb E_{\xi_i^{(t,k)}}\left[\mathcal L_i^{(t)}(\bm W_i^{(t,k+1)})\right]\leq\mathcal L_i^{(t)}(\bm W_i^{(t,k)})-\eta\left(\left\|\nabla_{\bm B}\mathcal L_i^{(t)}(\bm W_i^{(t,k)})\right\|_{\mathrm F}^{2}+\left\|\nabla_{\bm A}\mathcal L_i^{(t)}(\bm W_i^{(t,k)})\right\|_{\mathrm F}^{2}\right)+M_\lambda\eta^2.
\label{eq:app_one_step_descent}
\end{equation}

\textbf{One-Round Descent.} Applying the law of total expectation and summing Eq.~\eqref{eq:app_one_step_descent} over $k=0,\ldots,\tau-1$ gives
\begin{equation}
\mathbb E\left[\mathcal L_i^{(t)}(\bm W_i^{(t,\tau)})\right]\leq\mathbb E\left[\mathcal L_i^{(t)}(\bm W_i^{(t,0)})\right]-\eta\sum_{k=0}^{\tau-1}\mathbb E\left[\left\|\nabla_{\bm B}\mathcal L_i^{(t)}(\bm W_i^{(t,k)})\right\|_{\mathrm F}^{2}+\left\|\nabla_{\bm A}\mathcal L_i^{(t)}(\bm W_i^{(t,k)})\right\|_{\mathrm F}^{2}\right]+\tau M_\lambda\eta^2.
\label{eq:app_local_round_descent}
\end{equation}

Local factor preservation implies
\begin{equation}
\bm X_i^{(t,0)}
=
\bm X_i^{(t-1)},
\qquad
\bm W_i^{(t,0)}
=
\bm W_i^{(t-1)}.
\label{eq:app_round_initialization}
\end{equation}
Using the definition of $\mathcal L_i^{(t)}$, we therefore have
\begin{equation}
\mathcal L_i^{(t)}(\bm W_i^{(t,0)})
=
f_i(\bm W_i^{(t-1)})
+
\frac{\lambda}{2}
\left\|
\bm X_i^{(t-1)}-\bm X_g^{(t-1)}
\right\|_{\mathrm F}^{2}.
\label{eq:app_initial_round_objective}
\end{equation}
Averaging this equality over the clients and using Eq.~\eqref{eq:app_joint_objective} gives
\begin{equation}
\frac{1}{N}
\sum_{i=1}^{N}
\mathcal L_i^{(t)}(\bm W_i^{(t,0)})
=
\Psi^{(t-1)}.
\label{eq:app_initial_joint_objective}
\end{equation}

At the end of local training, let $\bm X_i^{(t)}=\bm X_i^{(t,\tau)}$ and $\bm W_i^{(t)}=\bm W_i^{(t,\tau)}$. Then,
\begin{equation}
\mathcal L_i^{(t)}(\bm W_i^{(t,\tau)})
=
f_i(\bm W_i^{(t)})
+
\frac{\lambda}{2}
\left\|
\bm X_i^{(t)}-\bm X_g^{(t-1)}
\right\|_{\mathrm F}^{2}.
\label{eq:app_preaggregation_objective}
\end{equation}

We next compare the product-level regularization terms before and after server aggregation. Let $\overline{\bm X}^{(t)}=N^{-1}\sum_{i=1}^{N}\bm X_i^{(t)}$. For any matrix $\bm X$, the variance decomposition gives
\begin{equation}
\frac{1}{N}
\sum_{i=1}^{N}
\left\|
\bm X_i^{(t)}-\bm X
\right\|_{\mathrm F}^{2}
=
\frac{1}{N}
\sum_{i=1}^{N}
\left\|
\bm X_i^{(t)}-\overline{\bm X}^{(t)}
\right\|_{\mathrm F}^{2}
+
\left\|
\overline{\bm X}^{(t)}-\bm X
\right\|_{\mathrm F}^{2}.
\label{eq:app_variance_decomposition}
\end{equation}
The first term on the right-hand side is independent of $\bm X$. Therefore, minimizing the left-hand side over matrices of rank at most $r$ is equivalent to finding the best rank-$r$ approximation of $\overline{\bm X}^{(t)}$. By the truncated-SVD aggregation rule,
\begin{equation}
\bm X_g^{(t)}
\in
\operatorname*{arg\,min}_{\operatorname{rank}(\bm X)\leq r}
\frac{1}{N}
\sum_{i=1}^{N}
\left\|
\bm X_i^{(t)}-\bm X
\right\|_{\mathrm F}^{2}.
\label{eq:app_rank_constrained_consensus}
\end{equation}
The previous global product $\bm X_g^{(t-1)}$ is feasible because $\operatorname{rank}(\bm X_g^{(t-1)})\leq r$. Hence, the optimality of $\bm X_g^{(t)}$ gives
\begin{equation}
\frac{1}{N}
\sum_{i=1}^{N}
\left\|
\bm X_i^{(t)}-\bm X_g^{(t)}
\right\|_{\mathrm F}^{2}
\leq
\frac{1}{N}
\sum_{i=1}^{N}
\left\|
\bm X_i^{(t)}-\bm X_g^{(t-1)}
\right\|_{\mathrm F}^{2}.
\label{eq:app_aggregation_nonincrease}
\end{equation}

The task-loss terms $f_i(\bm W_i^{(t)})$ are unchanged by server aggregation. Combining this observation with Eq.~\eqref{eq:app_aggregation_nonincrease} and Eq.~\eqref{eq:app_joint_objective} yields
\begin{equation}
\Psi^{(t)}
\leq
\frac{1}{N}
\sum_{i=1}^{N}
\mathcal L_i^{(t)}(\bm W_i^{(t,\tau)}).
\label{eq:app_postaggregation_bound}
\end{equation}

Averaging Eq.~\eqref{eq:app_local_round_descent} over all clients and using Eqs.~\eqref{eq:app_initial_joint_objective} and \eqref{eq:app_postaggregation_bound} gives
\begin{equation}
\mathbb E[\Psi^{(t)}]\leq\mathbb E[\Psi^{(t-1)}]-\frac{\eta}{N}\sum_{i=1}^{N}\sum_{k=0}^{\tau-1}\mathbb E\left[\left\|\nabla_{\bm B}\mathcal L_i^{(t)}(\bm W_i^{(t,k)})\right\|_{\mathrm F}^{2}+\left\|\nabla_{\bm A}\mathcal L_i^{(t)}(\bm W_i^{(t,k)})\right\|_{\mathrm F}^{2}\right]+\tau M_\lambda\eta^2.
\label{eq:app_one_round_descent}
\end{equation}

\textbf{Telescoping Over Communication Rounds.} Rearranging Eq.~\eqref{eq:app_one_round_descent} gives
\begin{equation}
\frac{\eta}{N}\sum_{i=1}^{N}\sum_{k=0}^{\tau-1}\mathbb E\left[\left\|\nabla_{\bm B}\mathcal L_i^{(t)}(\bm W_i^{(t,k)})\right\|_{\mathrm F}^{2}+\left\|\nabla_{\bm A}\mathcal L_i^{(t)}(\bm W_i^{(t,k)})\right\|_{\mathrm F}^{2}\right]\leq\mathbb E[\Psi^{(t-1)}]-\mathbb E[\Psi^{(t)}]+\tau M_\lambda\eta^2.
\label{eq:app_rearranged_round_descent}
\end{equation}
Summing this inequality over $t=1,\ldots,T$ telescopes the joint objective terms and gives
\begin{equation}
\frac{\eta}{N}\sum_{t=1}^{T}\sum_{i=1}^{N}\sum_{k=0}^{\tau-1}\mathbb E\left[\left\|\nabla_{\bm B}\mathcal L_i^{(t)}(\bm W_i^{(t,k)})\right\|_{\mathrm F}^{2}+\left\|\nabla_{\bm A}\mathcal L_i^{(t)}(\bm W_i^{(t,k)})\right\|_{\mathrm F}^{2}\right]\leq\mathbb E[\Psi^{(0)}]-\mathbb E[\Psi^{(T)}]+T\tau M_\lambda\eta^2.
\label{eq:app_telescoping_bound}
\end{equation}

Since the regularization terms are nonnegative,
$\Psi^{(T)}\geq N^{-1}\sum_{i=1}^{N} f_i(\bm W_i^{(T)})$.
By Assumption~\ref{assumption:conv_lipschitz},
$f_i(\bm W_i^{(T)})\geq f_i(\bm W_0)-L_c\|\bm W_i^{(T)}-\bm W_0\|_{\mathrm F}$.
Using $\bm W_i^{(T)}-\bm W_0=\bm X_i^{(T)}$ and
$\|\bm X_i^{(T)}\|_{\mathrm F}\leq C_A C_B$ from Eq.~\eqref{eq:app_local_product_bound}, we obtain
\begin{align}
\Psi^{(T)}
&\geq
\frac{1}{N}\sum_{i=1}^{N} f_i(\bm W_0)
-
\frac{L_c}{N}\sum_{i=1}^{N}\|\bm X_i^{(T)}\|_{\mathrm F}
\nonumber\\
&\geq
\frac{1}{N}\sum_{i=1}^{N} f_i(\bm W_0)
-
L_c C_A C_B
\nonumber\\
&\geq
\Psi^\star-L_c C_A C_B,
\label{eq:app_joint_objective_lower_bound}
\end{align}
where the last inequality follows from the assumed lower bound evaluated at $\bm W=\bm W_0$.

Therefore,
\begin{align}
\mathbb E[\Psi^{(0)}]-\mathbb E[\Psi^{(T)}]
&\leq
\mathbb E[\Psi^{(0)}]-\Psi^\star
+
L_c C_A C_B
\nonumber\\
&\leq
D+L_c C_A C_B
=
\widetilde D,
\label{eq:app_objective_gap_bound}
\end{align}
Substituting this inequality into Eq.~\eqref{eq:app_telescoping_bound} and dividing both sides by $\eta T\tau$ yields
\begin{equation}
\frac{1}{NT\tau}\sum_{t=1}^{T}\sum_{i=1}^{N}\sum_{k=0}^{\tau-1}\mathbb E\left[\left\|\nabla_{\bm B}\mathcal L_i^{(t)}(\bm W_i^{(t,k)})\right\|_{\mathrm F}^{2}+\left\|\nabla_{\bm A}\mathcal L_i^{(t)}(\bm W_i^{(t,k)})\right\|_{\mathrm F}^{2}\right]\leq\frac{\widetilde D}{\eta T\tau}+M_\lambda\eta.
\label{eq:app_average_gradient_bound}
\end{equation}
This proves Eq.~\eqref{eq:conv_average_gradient_bound} for $0<\eta\leq1$.

Finally, when $\widetilde D\leq M_\lambda T\tau$, choose
\begin{equation}
\eta
=
\sqrt{
\frac{\widetilde D}{M_\lambda T\tau}
}
\leq
1.
\label{eq:app_optimal_stepsize}
\end{equation}
For this choice,
\begin{equation}
\frac{\widetilde D}{\eta T\tau}=M_\lambda\eta=\frac{\sqrt{\widetilde D M_\lambda}}{\sqrt{T\tau}}.
\label{eq:app_balanced_terms}
\end{equation}
Therefore,
\begin{equation}
\frac{1}{NT\tau}\sum_{t=1}^{T}\sum_{i=1}^{N}\sum_{k=0}^{\tau-1}\mathbb E\left[\left\|\nabla_{\bm B}\mathcal L_i^{(t)}(\bm W_i^{(t,k)})\right\|_{\mathrm F}^{2}+\left\|\nabla_{\bm A}\mathcal L_i^{(t)}(\bm W_i^{(t,k)})\right\|_{\mathrm F}^{2}\right]\leq\frac{2\sqrt{\widetilde D M_\lambda}}{\sqrt{T\tau}}.
\label{eq:app_final_rate}
\end{equation}
Therefore, FedPA-LoRA achieves an $\mathcal O((T\tau)^{-1/2})$ averaged factor-space stationarity rate, as stated in Theorem~\ref{theorem:conv_fedpa}.
\hfill$\square$

\subsubsection{Proof of Proposition~\ref{proposition:conv_global_loss}}
\label{appendix:proof_global_loss}

Recall that $\bm W_i^{(t)}=\bm W_0+\bm X_i^{(t)}$ and $\bm W_g^{(t)}=\bm W_0+\bm X_g^{(t)}$. By Assumption~\ref{assumption:conv_lipschitz},
\begin{equation}
f_i(\bm W_g^{(t)})
\leq
f_i(\bm W_i^{(t)})
+
L_c
\left\|
\bm W_g^{(t)}-\bm W_i^{(t)}
\right\|_{\mathrm F}.
\label{eq:app_lipschitz_global_local_weight}
\end{equation}
Because the frozen pre-trained weight $\bm W_0$ is shared by the local and global models,
\begin{equation}
\bm W_g^{(t)}-\bm W_i^{(t)}
=
\bm X_g^{(t)}-\bm X_i^{(t)}.
\label{eq:app_global_local_weight_difference}
\end{equation}
Substituting Eq.~\eqref{eq:app_global_local_weight_difference} into Eq.~\eqref{eq:app_lipschitz_global_local_weight} gives
\begin{equation}
f_i(\bm W_g^{(t)})
\leq
f_i(\bm W_i^{(t)})
+
L_c
\left\|
\bm X_g^{(t)}-\bm X_i^{(t)}
\right\|_{\mathrm F}.
\label{eq:app_lipschitz_global_local}
\end{equation}

By Young's inequality, for any $a\geq0$ and $\lambda>0$,
\begin{equation}
L_c a
\leq
\frac{\lambda}{2}a^2
+
\frac{L_c^2}{2\lambda}.
\label{eq:app_young_scalar}
\end{equation}
Applying Eq.~\eqref{eq:app_young_scalar} with $a=\|\bm X_i^{(t)}-\bm X_g^{(t)}\|_{\mathrm F}$ yields
\begin{equation}
L_c
\left\|
\bm X_i^{(t)}-\bm X_g^{(t)}
\right\|_{\mathrm F}
\leq
\frac{\lambda}{2}
\left\|
\bm X_i^{(t)}-\bm X_g^{(t)}
\right\|_{\mathrm F}^{2}
+
\frac{L_c^2}{2\lambda}.
\label{eq:app_young_global_local}
\end{equation}

Combining Eqs.~\eqref{eq:app_lipschitz_global_local} and \eqref{eq:app_young_global_local} gives
\begin{equation}
f_i(\bm W_g^{(t)})
\leq
f_i(\bm W_i^{(t)})
+
\frac{\lambda}{2}
\left\|
\bm X_i^{(t)}-\bm X_g^{(t)}
\right\|_{\mathrm F}^{2}
+
\frac{L_c^2}{2\lambda}.
\label{eq:app_client_global_loss_bound}
\end{equation}

Averaging Eq.~\eqref{eq:app_client_global_loss_bound} over all clients gives
\begin{equation}
\frac{1}{N}
\sum_{i=1}^{N}
f_i(\bm W_g^{(t)})
\leq
\frac{1}{N}
\sum_{i=1}^{N}
\left[
f_i(\bm W_i^{(t)})
+
\frac{\lambda}{2}
\left\|
\bm X_i^{(t)}-\bm X_g^{(t)}
\right\|_{\mathrm F}^{2}
\right]
+
\frac{L_c^2}{2\lambda}.
\label{eq:app_average_global_loss_bound}
\end{equation}
By the definitions of $F(\bm W)=N^{-1}\sum_{i=1}^{N}f_i(\bm W)$ and $\Psi^{(t)}$ in Eq.~\eqref{eq:app_joint_objective}, Eq.~\eqref{eq:app_average_global_loss_bound} becomes
\begin{equation}
F(\bm W_g^{(t)})
\leq
\Psi^{(t)}
+
\frac{L_c^2}{2\lambda}.
\label{eq:app_global_loss_control}
\end{equation}
This proves Proposition~\ref{proposition:conv_global_loss}.
\hfill$\square$

\subsection{Convergence Analysis: Heterogeneous Setting}
\label{app:convergence_heterogeneous}

We extend the convergence guarantees of FedPA-LoRA to the heterogeneous setting where the local rank $r_i$ and reference rank $R_i$ may differ across clients. The server-side product-space aggregation and global product-guided alignment naturally accommodate such rank heterogeneity. We retain full client participation and $\lambda>0$ from the homogeneous analysis.

\subsubsection{Setup and Truncation Error}

For client $i$ at round $t$, the preserved local LoRA product is $\bm X_i^{(t)}=\bm B_i^{(t)}\bm A_i^{(t)}$, with $\operatorname{rank}(\bm X_i^{(t)})\leq r_i$. The server first forms the average product $\overline{\bm X}^{(t)}=\frac{1}{N}\sum_{i=1}^{N}\bm X_i^{(t)}$. The global update is obtained by projecting this average onto the set of matrices of rank at most $R_g=\max_i R_i$:
\begin{equation}
\bm X_g^{(t)}
\in
\operatorname*{arg\,min}_{\operatorname{rank}(\bm X)\leq R_g}
\left\|
\bm X-\overline{\bm X}^{(t)}
\right\|_{\mathrm F}.
\label{eq:het_global_proj}
\end{equation}
The reference sent to client $i$ for the next round is the best rank-$R_i$ approximation of the global product:
\begin{equation}
\bm X_{g,i}^{(t)}
\in
\operatorname*{arg\,min}_{\operatorname{rank}(\bm X)\leq R_i}
\left\|
\bm X-\bm X_g^{(t)}
\right\|_{\mathrm F}.
\label{eq:het_ref_proj}
\end{equation}
The exact reconstruction in Eq.~\eqref{eq:efficient_global_lora_reconstruction} realizes the global projection without explicitly forming a dense matrix, with optimality given by Eq.~\eqref{eq:optimal_rank_projection}. In round $t$, client $i$ minimizes the regularized local objective
\begin{equation}
\mathcal L_i^{(t)}(\bm W)
=
f_i(\bm W)
+
\frac{\lambda}{2}
\left\|
\bm W-\bm W_0-\bm X_{g,i}^{(t-1)}
\right\|_{\mathrm F}^{2},
\label{eq:het_local_obj}
\end{equation}
starting from the preserved product $\bm W_i^{(t,0)}=\bm W_0+\bm X_i^{(t-1)}$ and performing $\tau$ local factor SGD steps. Let $\bm W_i^{(t,k)}=\bm W_0+\bm B_i^{(t,k)}\bm A_i^{(t,k)}$ denote the weight matrix after $k$ steps, with $(\bm B_i^{(t,0)},\bm A_i^{(t,0)})=(\bm B_i^{(t-1)},\bm A_i^{(t-1)})$ and $(\bm B_i^{(t)},\bm A_i^{(t)})=(\bm B_i^{(t,\tau)},\bm A_i^{(t,\tau)})$.

As in the homogeneous setting, we define the round-wise joint objective using the reconstructed global product:
\begin{equation}
\Psi^{(t)}=\frac{1}{N}\sum_{i=1}^{N}\left[f_i(\bm W_i^{(t)})+\frac{\lambda}{2}\|\bm X_i^{(t)}-\bm X_g^{(t)}\|_{\mathrm F}^{2}\right].
\label{eq:het_Psi}
\end{equation}
Initially, $\bm B_g^{(0)}=\bm 0$, and each client initializes its factors from the global adapter. Thus, $\bm X_i^{(0)}=\bm X_{g,i}^{(0)}=\bm X_g^{(0)}=\bm 0$.

We adopt the four assumptions from the homogeneous analysis (Assumptions~\ref{assumption:conv_smoothness}--\ref{assumption:conv_lipschitz}). The same factor bounds apply to all clients and the initial global factors, regardless of rank.

For $t\geq1$, the client-specific reference is the best rank-$R_i$ approximation of the global product. By the Eckart--Young--Mirsky theorem \citep{eckart1936approximation},
\begin{equation}
\left\|\bm X_g^{(t)}-\bm X_{g,i}^{(t)}\right\|_{\mathrm F}^{2}
=
\sum_{k>R_i}\sigma_k\!\left(\bm X_g^{(t)}\right)^2,
\end{equation}
where $\sigma_k(\bm X_g^{(t)})$ denotes the $k$-th singular value in non-increasing order. We define the reference truncation error $\delta$ as
\begin{equation}
\delta =
\left(
\mathbb{E}\left[
\left(
\sup_{t\geq 0}\max_{i\in\{1,\ldots,N\}}
\left\|\bm X_g^{(t)}-\bm X_{g,i}^{(t)}\right\|_{\mathrm F}
\right)^2
\right]
\right)^{1/2}.
\end{equation}
The uniform factor bounds and the norm-nonincreasing property of truncated SVD imply $0\leq\delta\leq C_AC_B$. In particular, $\delta=0$ if the initial references equal the initial global product and $\min_i R_i\geq\operatorname{rank}(\bm X_g^{(t)})$ for all $t\geq1$.

To handle the mismatch between the per-client truncated references and the global product, we introduce the ideal local objective that uses the full global product as a reference:
\begin{equation}\label{eq:ideal_obj}
\mathcal L_i^{\mathrm{ideal},(t)}(\bm W) = f_i(\bm W) + \frac{\lambda}{2}\bigl\|\bm W - \bm W_0 - \bm X_g^{(t-1)}\bigr\|_{\mathrm F}^{2}.
\end{equation}
Let $\bm\Delta_i^{(t-1)}=\bm X_{g,i}^{(t-1)}-\bm X_g^{(t-1)}$. By the definition of $\delta$, we have $\mathbb E[\|\bm\Delta_i^{(t-1)}\|_{\mathrm F}^{2}]\leq\delta^{2}$. Taking the gradients of both objectives, we obtain
\begin{align}
\nabla_{\bm W}\mathcal L_i^{(t)}(\bm W)
&= \nabla_{\bm W}f_i(\bm W) + \lambda\bigl(\bm W - \bm W_0 - \bm X_{g,i}^{(t-1)}\bigr), \nonumber\\
\nabla_{\bm W}\mathcal L_i^{\mathrm{ideal},(t)}(\bm W)
&= \nabla_{\bm W}f_i(\bm W) + \lambda\bigl(\bm W - \bm W_0 - \bm X_g^{(t-1)}\bigr).
\end{align}
Hence, we have
\begin{equation}
\nabla_{\bm W}\mathcal L_i^{(t)}(\bm W) = \nabla_{\bm W}\mathcal L_i^{\mathrm{ideal},(t)}(\bm W) - \lambda \bm \Delta_i^{(t-1)} ,
\label{eq:grad_bias}
\end{equation}
a relation that will be used repeatedly in the analysis.

\subsubsection{Convergence Result}

\begin{theorem}[FedPA-LoRA, heterogeneous ranks]
\label{thm:het_conv}
Under Assumptions~\ref{assumption:conv_smoothness}--\ref{assumption:conv_lipschitz} of the homogeneous analysis, suppose that $\frac{1}{N}\sum_{i=1}^{N}f_i(\bm W)\geq\Psi^\star$ for all $\bm W$ and $\mathbb E[\Psi^{(0)}]-\Psi^\star\leq D$ for some $D>0$. We use the constants $\widetilde D=D+L_cC_AC_B$ and $M_\lambda$ defined in the homogeneous analysis. Then, for any $0<\eta\leq1$,
\begin{equation}
\frac{1}{NT\tau}\sum_{t=1}^{T}\sum_{i=1}^{N}\sum_{k=0}^{\tau-1}
\mathbb E\Bigl[
\bigl\|\nabla_{\bm B}\mathcal L_i^{(t)}(\bm W_i^{(t,k)})\bigr\|_{\mathrm F}^{2}
+
\bigl\|\nabla_{\bm A}\mathcal L_i^{(t)}(\bm W_i^{(t,k)})\bigr\|_{\mathrm F}^{2}
\Bigr]
\leq
\frac{2\widetilde D}{\eta T\tau}
+
2M_\lambda\eta
+
\lambda^2(C_A^2+C_B^2)\delta^2.
\label{eq:het_bound}
\end{equation}
When $\widetilde D\leq M_\lambda T\tau$, choosing $\eta=\sqrt{\widetilde D/(M_\lambda T\tau)}$ yields
\begin{equation}
\frac{1}{NT\tau}\sum_{t=1}^{T}\sum_{i=1}^{N}\sum_{k=0}^{\tau-1}
\mathbb E\Bigl[
\bigl\|\nabla_{\bm B}\mathcal L_i^{(t)}(\bm W_i^{(t,k)})\bigr\|_{\mathrm F}^{2}
+
\bigl\|\nabla_{\bm A}\mathcal L_i^{(t)}(\bm W_i^{(t,k)})\bigr\|_{\mathrm F}^{2}
\Bigr]
\leq
\frac{4\sqrt{\widetilde D M_\lambda}}{\sqrt{T\tau}}
+
\lambda^2(C_A^2+C_B^2)\delta^2.
\label{eq:het_rate}
\end{equation}
\end{theorem}

Thus, FedPA-LoRA achieves an $\mathcal O((T\tau)^{-1/2})$ averaged factor-space stationarity rate up to an additive $\mathcal O(\lambda^2\delta^2)$ error induced by reference truncation.

\begin{proposition}[Global LoRA loss]
\label{prop:het_global_loss}
Under Assumption~\ref{assumption:conv_lipschitz}, $\bm W_g^{(t)}=\bm W_0+\bm X_g^{(t)}$ satisfies
\begin{equation}
\mathbb E[F(\bm W_g^{(t)})]
\leq
\mathbb E[\Psi^{(t)}]+\frac{L_c^2}{2\lambda},
\qquad
F(\bm W)=\frac{1}{N}\sum_{i=1}^{N}f_i(\bm W).
\label{eq:het_global_loss}
\end{equation}
\end{proposition}

\subsubsection{Preliminary Lemmas}

We begin with two intermediate lemmas that we build on to prove Theorem~\ref{thm:het_conv}.

\begin{lemma}[Boundedness and smoothness]
\label{lem:bounds}
Under the four assumptions of the homogeneous analysis, the following properties hold for all clients, rounds, and local steps.
\begin{enumerate}
\item[(i)] \textbf{Bounded products.}
\[
\|\bm X_i^{(t,k)}\|_{\mathrm F}\leq C_X,\qquad
\|\bm X_{g,i}^{(t)}\|_{\mathrm F}\leq C_g,\qquad
\|\bm X_g^{(t)}\|_{\mathrm F}\leq C_g.
\]
Both constants can be chosen as $C_X=C_g=C_AC_B$.

\item[(ii)] \textbf{Smoothness.}
The objective $\mathcal L_i^{(t)}$ is $L$-smooth in $\bm W$: for any $\bm W,\bm W'$,
\begin{align}
\bigl\|
\nabla_{\bm W}\mathcal L_i^{(t)}(\bm W)
-
\nabla_{\bm W}\mathcal L_i^{(t)}(\bm W')
\bigr\|_{\mathrm F}
\leq
L
\|\bm W-\bm W'\|_{\mathrm F}.
\end{align}
The ideal objective $\mathcal L_i^{\mathrm{ideal},(t)}$ is also $L$-smooth in $\bm W$: for any $\bm W,\bm W'$,
\begin{align}
\bigl\|
\nabla_{\bm W}\mathcal L_i^{\mathrm{ideal},(t)}(\bm W)
-
\nabla_{\bm W}\mathcal L_i^{\mathrm{ideal},(t)}(\bm W')
\bigr\|_{\mathrm F}
\leq
L
\|\bm W-\bm W'\|_{\mathrm F},
\end{align}
where $L=L_s+\lambda$.

\item[(iii)] \textbf{Bounded gradients.} For either $\mathcal L_i^{(t)}$ or $\mathcal L_i^{\mathrm{ideal},(t)}$, the stochastic weight-space gradients evaluated at $\bm W_i^{(t,k)}$ are bounded by $G_\lambda$, as defined in the homogeneous analysis. Consequently, the corresponding stochastic factor gradients satisfy
\begin{equation}
\|\nabla_{\bm B}\mathcal L(\bm W_i^{(t,k)};\xi_i^{(t,k)})\|_{\mathrm F}^{2}+\|\nabla_{\bm A}\mathcal L(\bm W_i^{(t,k)};\xi_i^{(t,k)})\|_{\mathrm F}^{2}\leq G_\lambda^2(C_A^2+C_B^2),
\end{equation}
where $\mathcal L$ denotes either objective.
\end{enumerate}
\end{lemma}

\begin{proof}
\textbf{Bounded products.} From the factor bounds (Assumption~\ref{assumption:conv_factor_regularity}), we have for any $\bm W = \bm W_0 + \bm B\bm A$ with $\|\bm A\|_{\mathrm F}\le C_A$ and $\|\bm B\|_{\mathrm F}\le C_B$,
\begin{equation}
\|\bm X\|_{\mathrm F} = \|\bm B\bm A\|_{\mathrm F} \le \|\bm B\|_{\mathrm F}\|\bm A\|_{\mathrm F} \le C_A C_B =: C_X.
\end{equation}
Thus, all local products are bounded by $C_X$.

For the global products $\bm X_g^{(t)}$ (with $t\ge 1$), by Eq.~\eqref{eq:het_global_proj}, $\bm X_g^{(t)}$ is the best rank-$R_g$ approximation of the average product $\overline{\bm X}^{(t)}$. Since the Frobenius norm of a truncated SVD does not exceed the norm of the original matrix (indeed, for any matrix $\bm A$, $\|\bm A\|_{\mathrm F}^{2} = \sum_{k} \sigma_k^2(\bm A)$, and truncating the SVD removes only nonnegative terms from this sum), we have
\begin{equation}
\|\bm X_g^{(t)}\|_{\mathrm F} \le \|\overline{\bm X}^{(t)}\|_{\mathrm F} \le \frac{1}{N}\sum_{i=1}^{N}\|\bm X_i^{(t)}\|_{\mathrm F} \le C_X.
\end{equation}
For the references $\bm X_{g,i}^{(t)}$, by Eq.~\eqref{eq:het_ref_proj}, each is the best rank-$R_i$ approximation of $\bm X_g^{(t)}$. Again, truncation cannot increase the Frobenius norm, so
\begin{equation}
\|\bm X_{g,i}^{(t)}\|_{\mathrm F} \le \|\bm X_g^{(t)}\|_{\mathrm F} \le C_X.
\end{equation}
At initialization, $\bm B_g^{(0)}=\bm 0$ implies $\bm X_g^{(0)}=\bm X_{g,i}^{(0)}=\bm 0$ for every client $i$. Hence, the product bounds also hold at $t=0$.

Therefore, for all rounds $t\geq0$ and clients $i$,
\begin{equation}
\|\bm X_i^{(t,k)}\|_{\mathrm F}\le C_X,\qquad
\|\bm X_g^{(t)}\|_{\mathrm F}\le C_g,\qquad
\|\bm X_{g,i}^{(t)}\|_{\mathrm F}\le C_g.
\end{equation}

\textbf{Smoothness.} The local objective $\mathcal L_i^{(t)}(\bm W) = f_i(\bm W) + \frac{\lambda}{2}\|\bm W - \bm W_0 - \bm X_{g,i}^{(t-1)}\|_{\mathrm F}^{2}$ is $L_s$-smooth in $f_i$ and quadratic in the regularizer, hence $L = (L_s+\lambda)$-smooth. The ideal objective $\mathcal L_i^{\mathrm{ideal},(t)}$ differs only by the reference matrix, which does not affect the Hessian, so it is also $L$-smooth.

\textbf{Bounded gradients.} For either the actual or ideal objective, the stochastic weight-space gradient is
\begin{equation}
\nabla_{\bm W}\mathcal L(\bm W;\xi)=\nabla_{\bm W}f_i(\bm W;\xi)+\lambda(\bm W-\bm W_0-\bm X_{\mathrm{ref}}),
\end{equation}
where $\mathcal L$ denotes either objective and $\bm X_{\mathrm{ref}}$ is $\bm X_{g,i}^{(t-1)}$ or $\bm X_g^{(t-1)}$, respectively. Since both reference products are bounded by $C_g$, Assumption~\ref{assumption:conv_stochastic_gradients} yields
\begin{equation}
\|\nabla_{\bm W}\mathcal L(\bm W;\xi)\|_{\mathrm F}\leq G_f+\lambda(C_X+C_g)=G_f+2\lambda C_AC_B=G_\lambda.
\end{equation}
By the chain rule, the stochastic factor gradients satisfy
\begin{equation}
\nabla_{\bm B}\mathcal L(\bm W;\xi)=\nabla_{\bm W}\mathcal L(\bm W;\xi)\bm A^\top,\qquad
\nabla_{\bm A}\mathcal L(\bm W;\xi)=\bm B^\top\nabla_{\bm W}\mathcal L(\bm W;\xi).
\end{equation}
Therefore, using the factor bounds,
\begin{equation}
\|\nabla_{\bm B}\mathcal L(\bm W;\xi)\|_{\mathrm F}^{2}+\|\nabla_{\bm A}\mathcal L(\bm W;\xi)\|_{\mathrm F}^{2}\leq G_\lambda^2(C_A^2+C_B^2).
\end{equation}
This completes the proof.
\end{proof}

\begin{lemma}[One-step descent]
\label{lem:descent}
Consider a single local SGD step starting from $(\bm B,\bm A)$ with $\bm W=\bm W_0+\bm B\bm A$, where the factors are updated using a mini-batch $\xi$ on the actual objective $\mathcal L_i^{(t)}$:
\begin{equation}
\bm B'=\bm B-\eta\nabla_{\bm B}\mathcal L_i^{(t)}(\bm W;\xi),\qquad \bm A'=\bm A-\eta\nabla_{\bm A}\mathcal L_i^{(t)}(\bm W;\xi).
\end{equation}
Let $\bm W'=\bm W_0+\bm B'\bm A'$. Under Assumptions~\ref{assumption:conv_smoothness}--\ref{assumption:conv_factor_regularity}, for any $0<\eta\leq1$,
\begin{equation}
\begin{aligned}
\mathbb E[\mathcal L_i^{\mathrm{ideal},(t)}(\bm W')]
&\leq\mathbb E[\mathcal L_i^{\mathrm{ideal},(t)}(\bm W)]\\
&\quad-\frac{\eta}{2}\mathbb E\left[\|\nabla_{\bm B}\mathcal L_i^{\mathrm{ideal},(t)}(\bm W)\|_{\mathrm F}^{2}+\|\nabla_{\bm A}\mathcal L_i^{\mathrm{ideal},(t)}(\bm W)\|_{\mathrm F}^{2}\right]\\
&\quad-\frac{\eta}{2}\mathbb E\left[\|\nabla_{\bm B}\mathcal L_i^{(t)}(\bm W)\|_{\mathrm F}^{2}+\|\nabla_{\bm A}\mathcal L_i^{(t)}(\bm W)\|_{\mathrm F}^{2}\right]\\
&\quad+\frac{\eta\lambda^2}{2}(C_A^2+C_B^2)\delta^2+M_\lambda\eta^2.
\end{aligned}
\label{eq:descent_step}
\end{equation}
\end{lemma}

\begin{proof}
By the $(L_s+\lambda)$-smoothness of $\mathcal L_i^{\mathrm{ideal},(t)}$ established in Lemma~\ref{lem:bounds}(ii), we have
\begin{equation}
\mathcal L_i^{\mathrm{ideal},(t)}(\bm W')\leq\mathcal L_i^{\mathrm{ideal},(t)}(\bm W)+\left\langle\nabla_{\bm W}\mathcal L_i^{\mathrm{ideal},(t)}(\bm W),\Delta\bm W\right\rangle_{\mathrm F}+\frac{L_s+\lambda}{2}\|\Delta\bm W\|_{\mathrm F}^{2},
\end{equation}
where $\Delta\bm W=\bm W'-\bm W$. We bound the inner-product and squared-increment terms separately.

\medskip
\noindent\textbf{Step 1: Expected inner product.}
Expanding the product of the updated LoRA factors yields
\begin{equation}
\Delta\bm W=-\eta(\bm g_B\bm A+\bm B\bm g_A)+\eta^2\bm g_B\bm g_A,
\end{equation}
where $\bm g_B=\nabla_{\bm B}\mathcal L_i^{(t)}(\bm W;\xi)$ and $\bm g_A=\nabla_{\bm A}\mathcal L_i^{(t)}(\bm W;\xi)$ are the stochastic factor gradients. By unbiasedness,
\begin{equation}
\mathbb E_\xi[\bm g_B]=\nabla_{\bm B}\mathcal L_i^{(t)}(\bm W)=:\bm G_B,\qquad \mathbb E_\xi[\bm g_A]=\nabla_{\bm A}\mathcal L_i^{(t)}(\bm W)=:\bm G_A.
\end{equation}
We also define the ideal factor gradients as $\bm G_B^{\mathrm{ideal}}=\nabla_{\bm B}\mathcal L_i^{\mathrm{ideal},(t)}(\bm W)$ and $\bm G_A^{\mathrm{ideal}}=\nabla_{\bm A}\mathcal L_i^{\mathrm{ideal},(t)}(\bm W)$. By the chain rule,
\begin{equation}
\left\langle\nabla_{\bm W}\mathcal L_i^{\mathrm{ideal},(t)}(\bm W),\bm G_B\bm A+\bm B\bm G_A\right\rangle_{\mathrm F}=\langle\bm G_B^{\mathrm{ideal}},\bm G_B\rangle_{\mathrm F}+\langle\bm G_A^{\mathrm{ideal}},\bm G_A\rangle_{\mathrm F}.
\end{equation}
Using the identity $2\langle\bm P,\bm Q\rangle_{\mathrm F}=\|\bm P\|_{\mathrm F}^2+\|\bm Q\|_{\mathrm F}^2-\|\bm P-\bm Q\|_{\mathrm F}^2$, we obtain
\begin{equation}
\begin{aligned}
&-\langle\bm G_B^{\mathrm{ideal}},\bm G_B\rangle_{\mathrm F}-\langle\bm G_A^{\mathrm{ideal}},\bm G_A\rangle_{\mathrm F}\\
&=-\frac{1}{2}\left(\|\bm G_B^{\mathrm{ideal}}\|_{\mathrm F}^2+\|\bm G_A^{\mathrm{ideal}}\|_{\mathrm F}^2+\|\bm G_B\|_{\mathrm F}^2+\|\bm G_A\|_{\mathrm F}^2\right)+\frac{1}{2}\left(\|\bm G_B^{\mathrm{ideal}}-\bm G_B\|_{\mathrm F}^2+\|\bm G_A^{\mathrm{ideal}}-\bm G_A\|_{\mathrm F}^2\right).
\end{aligned}
\end{equation}
By Eq.~\eqref{eq:grad_bias} and the chain rule,
\begin{equation}
\bm G_B^{\mathrm{ideal}}-\bm G_B=\lambda\bm\Delta_i^{(t-1)}\bm A^\top,\qquad \bm G_A^{\mathrm{ideal}}-\bm G_A=\lambda\bm B^\top\bm\Delta_i^{(t-1)}.
\end{equation}
Therefore, using the factor bounds,
\begin{equation}
\|\bm G_B^{\mathrm{ideal}}-\bm G_B\|_{\mathrm F}^2+\|\bm G_A^{\mathrm{ideal}}-\bm G_A\|_{\mathrm F}^2\leq\lambda^2(C_A^2+C_B^2)\|\bm\Delta_i^{(t-1)}\|_{\mathrm F}^2.
\end{equation}
For the second-order interaction term, the boundedness of the ideal weight-space gradient and the stochastic factor gradients gives
\begin{equation}
\left|\left\langle\nabla_{\bm W}\mathcal L_i^{\mathrm{ideal},(t)}(\bm W),\mathbb E_\xi[\bm g_B\bm g_A]\right\rangle_{\mathrm F}\right|\leq C_AC_BG_\lambda^3.
\end{equation}
Combining these results yields
\begin{equation}
\begin{aligned}
\mathbb E_\xi\left[\left\langle\nabla_{\bm W}\mathcal L_i^{\mathrm{ideal},(t)}(\bm W),\Delta\bm W\right\rangle_{\mathrm F}\right]
&\leq-\frac{\eta}{2}\left(\|\bm G_B^{\mathrm{ideal}}\|_{\mathrm F}^2+\|\bm G_A^{\mathrm{ideal}}\|_{\mathrm F}^2+\|\bm G_B\|_{\mathrm F}^2+\|\bm G_A\|_{\mathrm F}^2\right)\\
&\quad+\frac{\eta\lambda^2}{2}(C_A^2+C_B^2)\|\bm\Delta_i^{(t-1)}\|_{\mathrm F}^2+\eta^2C_AC_BG_\lambda^3.
\end{aligned}
\label{eq:inner_prod_bound}
\end{equation}

\medskip
\noindent\textbf{Step 2: Expected squared increment.}
Expanding $\|\Delta\bm W\|_{\mathrm F}^2$ and using $\|\bm X+\bm Y+\bm Z\|_{\mathrm F}^2\leq3(\|\bm X\|_{\mathrm F}^2+\|\bm Y\|_{\mathrm F}^2+\|\bm Z\|_{\mathrm F}^2)$, we obtain
\begin{equation}
\|\Delta\bm W\|_{\mathrm F}^2\leq3\eta^2\left(\|\bm g_B\bm A\|_{\mathrm F}^2+\|\bm B\bm g_A\|_{\mathrm F}^2\right)+3\eta^4\|\bm g_B\bm g_A\|_{\mathrm F}^2.
\end{equation}
Taking expectation over $\xi$ and applying Lemma~\ref{lem:bounds}(iii) together with $\|\bm A\|_{\mathrm F}\leq C_A$, $\|\bm B\|_{\mathrm F}\leq C_B$, and $\eta\leq1$, we obtain
\begin{equation}
\frac{L_s+\lambda}{2}\mathbb E_\xi[\|\Delta\bm W\|_{\mathrm F}^2]\leq\frac{3(L_s+\lambda)\eta^2}{2}\left[(C_A^4+C_B^4)G_\lambda^2+C_A^2C_B^2G_\lambda^4\right].
\label{eq:squared_increment_bound}
\end{equation}

\medskip
\noindent\textbf{Step 3: Combine and conclude.}
Combining the bounds from Steps 1 and 2 in the smoothness inequality, the second-order terms are bounded by $\eta^2M_\lambda$, where $M_\lambda$ is identical to that in the homogeneous analysis. Finally, taking full expectation and applying $\mathbb E[\|\bm\Delta_i^{(t-1)}\|_{\mathrm F}^2]\leq\delta^2$ establishes \eqref{eq:descent_step}.
\end{proof}

\subsubsection{Proof of Theorem~\ref{thm:het_conv}}

\begin{proof}
Local factor preservation ensures that $\bm W_i^{(t,0)}=\bm W_i^{(t-1)}$. Summing the one-step descent bound in Lemma~\ref{lem:descent} over $k=0,\ldots,\tau-1$ and dropping the nonpositive terms involving the ideal factor gradients yields
\begin{equation}
\begin{aligned}
\mathbb E[\mathcal L_i^{\mathrm{ideal},(t)}(\bm W_i^{(t)})]
&\leq\mathbb E[\mathcal L_i^{\mathrm{ideal},(t)}(\bm W_i^{(t-1)})]\\
&\quad-\frac{\eta}{2}\sum_{k=0}^{\tau-1}\mathbb E\left[\|\nabla_{\bm B}\mathcal L_i^{(t)}(\bm W_i^{(t,k)})\|_{\mathrm F}^2+\|\nabla_{\bm A}\mathcal L_i^{(t)}(\bm W_i^{(t,k)})\|_{\mathrm F}^2\right]\\
&\quad+\frac{\tau\eta\lambda^2}{2}(C_A^2+C_B^2)\delta^2+\tau M_\lambda\eta^2.
\end{aligned}
\label{eq:tau_descent_proof}
\end{equation}

To establish descent across communication rounds, we use the joint objective $\Psi^{(t)}$ defined in Eq.~\eqref{eq:het_Psi}. By the definition of the ideal objective and local factor preservation,
\begin{equation}
\Psi^{(t-1)}=\frac{1}{N}\sum_{i=1}^{N}\mathcal L_i^{\mathrm{ideal},(t)}(\bm W_i^{(t-1)}).
\label{eq:Psi_tm1_identity}
\end{equation}

By Eqs.~\eqref{eq:app_variance_decomposition} and \eqref{eq:het_global_proj}, $\bm X_g^{(t)}$ minimizes the average squared distance to the client products among matrices of rank at most $R_g$. Since $\operatorname{rank}(\bm X_g^{(t-1)})\leq R_g$, its optimality gives
\begin{equation}
\frac{1}{N}\sum_{i=1}^{N}\|\bm X_i^{(t)}-\bm X_g^{(t)}\|_{\mathrm F}^2\leq\frac{1}{N}\sum_{i=1}^{N}\|\bm X_i^{(t)}-\bm X_g^{(t-1)}\|_{\mathrm F}^2.
\label{eq:proj_opt}
\end{equation}
Since $\mathcal L_i^{\mathrm{ideal},(t)}(\bm W_i^{(t)})=f_i(\bm W_i^{(t)})+\frac{\lambda}{2}\|\bm X_i^{(t)}-\bm X_g^{(t-1)}\|_{\mathrm F}^2$, Eq.~\eqref{eq:proj_opt} implies
\begin{equation}
\Psi^{(t)}\leq\frac{1}{N}\sum_{i=1}^{N}\mathcal L_i^{\mathrm{ideal},(t)}(\bm W_i^{(t)}).
\label{eq:Psi_ineq_proof}
\end{equation}

Averaging Eq.~\eqref{eq:tau_descent_proof} over all clients and combining Eqs.~\eqref{eq:Psi_tm1_identity} and \eqref{eq:Psi_ineq_proof}, we obtain
\begin{equation}
\begin{aligned}
\mathbb E[\Psi^{(t)}]
&\leq\mathbb E[\Psi^{(t-1)}]\\
&\quad-\frac{\eta}{2N}\sum_{i=1}^{N}\sum_{k=0}^{\tau-1}\mathbb E\left[\|\nabla_{\bm B}\mathcal L_i^{(t)}(\bm W_i^{(t,k)})\|_{\mathrm F}^2+\|\nabla_{\bm A}\mathcal L_i^{(t)}(\bm W_i^{(t,k)})\|_{\mathrm F}^2\right]\\
&\quad+\frac{\tau\eta\lambda^2}{2}(C_A^2+C_B^2)\delta^2+\tau M_\lambda\eta^2.
\end{aligned}
\label{eq:Psi_descent_proof}
\end{equation}

Summing Eq.~\eqref{eq:Psi_descent_proof} over $t=1,\ldots,T$ and telescoping the joint objective terms yields
\begin{equation}
\begin{aligned}
&\frac{\eta}{2N}\sum_{t=1}^{T}\sum_{i=1}^{N}\sum_{k=0}^{\tau-1}
\mathbb E\left[\|\nabla_{\bm B}\mathcal L_i^{(t)}(\bm W_i^{(t,k)})\|_{\mathrm F}^2+\|\nabla_{\bm A}\mathcal L_i^{(t)}(\bm W_i^{(t,k)})\|_{\mathrm F}^2\right]\\
&\quad\leq\mathbb E[\Psi^{(0)}]-\mathbb E[\Psi^{(T)}]
+\frac{T\tau\eta\lambda^2}{2}(C_A^2+C_B^2)\delta^2+T\tau M_\lambda\eta^2.
\end{aligned}
\label{eq:Psi_telescoping_proof}
\end{equation}

As in the homogeneous analysis, the nonnegative regularization terms, Lipschitz continuity, and factor bounds give
$\Psi^{(T)}\geq \frac{1}{N}\sum_{i=1}^{N}f_i(\bm W_0)-L_cC_AC_B\geq\Psi^\star-L_cC_AC_B$.
Combining this with the initial objective gap yields
\begin{equation}
\mathbb E[\Psi^{(0)}]-\mathbb E[\Psi^{(T)}]
\leq D+L_cC_AC_B=\widetilde D.
\label{eq:Psi_gap_proof}
\end{equation}
Substituting Eq.~\eqref{eq:Psi_gap_proof} into Eq.~\eqref{eq:Psi_telescoping_proof} and dividing by $\eta T\tau/2$ establishes Eq.~\eqref{eq:het_bound}. Finally, when $\widetilde D\leq M_\lambda T\tau$, choosing $\eta=\sqrt{\widetilde D/(M_\lambda T\tau)}\leq1$ gives
\begin{equation}
\frac{2\widetilde D}{\eta T\tau}+2M_\lambda\eta
=\frac{4\sqrt{\widetilde D M_\lambda}}{\sqrt{T\tau}},
\end{equation}
which establishes Eq.~\eqref{eq:het_rate}.
\end{proof}

\subsubsection{Proof of Proposition~\ref{prop:het_global_loss}}

\begin{proof}
By the $L_c$-Lipschitz continuity of $f_i$ and $\bm W_g^{(t)}-\bm W_i^{(t)}=\bm X_g^{(t)}-\bm X_i^{(t)}$, we have
\begin{equation}
\begin{aligned}
f_i(\bm W_g^{(t)})
&\leq f_i(\bm W_i^{(t)})+L_c\|\bm X_i^{(t)}-\bm X_g^{(t)}\|_{\mathrm F}\\
&\leq f_i(\bm W_i^{(t)})+\frac{\lambda}{2}\|\bm X_i^{(t)}-\bm X_g^{(t)}\|_{\mathrm F}^2+\frac{L_c^2}{2\lambda},
\end{aligned}
\end{equation}
where the second inequality follows from Young's inequality. Averaging over all clients and using the definition of $\Psi^{(t)}$ yields
\begin{equation}
F(\bm W_g^{(t)})
\leq\Psi^{(t)}+\frac{L_c^2}{2\lambda}.
\end{equation}
Taking expectations completes the proof.
\end{proof}

\section{Convergence with Partial Participation}
\label{app:partial_participation}

We extend the homogeneous convergence analysis to partial client participation, where only $K$ of the $N$ clients are active at each round.

\begin{theorem}[Convergence under partial participation]
\label{thm:partial_participation}
We retain the assumptions and initial-gap conditions of Theorem~\ref{theorem:conv_fedpa}, except for full participation. At each round $t$, a subset $\mathcal S_t\subseteq\{1,\ldots,N\}$ of size $K$ is sampled uniformly without replacement, independently of the past and mini-batch noise.

For each client $i$, the step-indexed factors $(\bm B_i^{(t,k)},\bm A_i^{(t,k)})$ describe the local trajectory under the fixed global reference $\bm X_g^{(t-1)}$, which is realized if $i\in\mathcal S_t$ and virtual otherwise. For $i\notin\mathcal S_t$, the actual model remains unchanged, so $\bm W_i^{(t)}=\bm W_i^{(t-1)}$ and $\bm X_i^{(t)}=\bm X_i^{(t-1)}$. The server reconstructs $\bm X_g^{(t)}$ as the best rank-$r$ approximation of the sampled average $\hat{\bm X}^{(t)}=\frac{1}{K}\sum_{i\in\mathcal S_t}\bm X_i^{(t)}$.

We use $\widetilde D=D+L_cC_AC_B$ and $M_\lambda$ from Theorem~\ref{theorem:conv_fedpa}, and define $C_X=C_AC_B$ as a bound on the local and global product norms. Then, with $B=6\lambda C_X^2$, for any $0<\eta\leq1$,
\begin{equation}
\frac{1}{NT\tau}\sum_{t=1}^{T}\sum_{i=1}^{N}\sum_{k=0}^{\tau-1}\mathbb E\!\left[\|\nabla_{\bm B}\mathcal L_i^{(t)}(\bm W_i^{(t,k)})\|_{\mathrm F}^{2}+\|\nabla_{\bm A}\mathcal L_i^{(t)}(\bm W_i^{(t,k)})\|_{\mathrm F}^{2}\right]\leq\frac{N\widetilde D}{\eta KT\tau}+M_\lambda\eta+\frac{(N-K)B}{\eta K\tau}.
\label{eq:partial_bound}
\end{equation}

When $K=N$, the additional term vanishes, recovering the full-participation bound in Theorem~\ref{theorem:conv_fedpa}.
\end{theorem}

\begin{proof}
Throughout the proof, expectations at round $t$ are conditioned on the state at the start of the round, with total expectations obtained by the tower property. Recall the joint objective
\begin{equation}
\Psi^{(t)}
=
\frac{1}{N}\sum_{i=1}^{N}
\left[
f_i(\bm W_i^{(t)})
+
\frac{\lambda}{2}
\left\|\bm X_i^{(t)}-\bm X_g^{(t)}\right\|_{\mathrm F}^{2}
\right],
\end{equation}
and define the sampled joint objective
\begin{equation}
\Psi_{\mathcal S_t}^{(t)}
=
\frac{1}{K}\sum_{i\in\mathcal S_t}
\left[
f_i(\bm W_i^{(t)})
+
\frac{\lambda}{2}
\left\|\bm X_i^{(t)}-\bm X_g^{(t)}\right\|_{\mathrm F}^{2}
\right].
\end{equation}

Since $\bm X_g^{(t)}$ minimizes the sampled consensus penalty among rank-$r$ products and $\bm X_g^{(t-1)}$ is feasible, the factor-space descent bound from the homogeneous analysis applies to the participating clients. Uniform sampling ensures that the expected previous-round sampled objective equals the full joint objective, while the expected sampled factor-gradient average equals the full-client average over the corresponding virtual trajectories. Taking expectations yields
\begin{equation}
\mathbb E[\Psi_{\mathcal S_t}^{(t)}]
\leq
\mathbb E[\Psi^{(t-1)}]
-
\frac{\eta}{N}\sum_{i=1}^{N}\sum_{k=0}^{\tau-1}\mathbb E\!\left[\|\nabla_{\bm B}\mathcal L_i^{(t)}(\bm W_i^{(t,k)})\|_{\mathrm F}^{2}+\|\nabla_{\bm A}\mathcal L_i^{(t)}(\bm W_i^{(t,k)})\|_{\mathrm F}^{2}\right]
+
\tau M_\lambda\eta^2.
\label{eq:sampled_descent_expectation}
\end{equation}

Now relate $\Psi^{(t)}$ to $\Psi_{\mathcal S_t}^{(t)}$. For any realization of $\mathcal S_t$, we have
\begin{equation}
\Psi^{(t)}
=
\frac{K}{N}\Psi_{\mathcal S_t}^{(t)}
+
\frac{1}{N}\sum_{i\notin \mathcal S_t}
\left[
f_i(\bm W_i^{(t)})
+
\frac{\lambda}{2}\left\|\bm X_i^{(t)}-\bm X_g^{(t)}\right\|_{\mathrm F}^{2}
\right],
\end{equation}
\begin{equation}
\Psi^{(t-1)}
=
\frac{K}{N}\Psi_{\mathcal S_t}^{(t-1)}
+
\frac{1}{N}\sum_{i\notin \mathcal S_t}
\left[
f_i(\bm W_i^{(t-1)})
+
\frac{\lambda}{2}\left\|\bm X_i^{(t-1)}-\bm X_g^{(t-1)}\right\|_{\mathrm F}^{2}
\right].
\end{equation}

Since non-participating clients retain their factors, their task-loss terms cancel when subtracting the two decompositions. Taking expectations yields
\begin{align}
\mathbb E[\Psi^{(t)}-\Psi^{(t-1)}]
=
\frac{K}{N}\,\mathbb E\bigl[\Psi_{\mathcal S_t}^{(t)}-\Psi_{\mathcal S_t}^{(t-1)}\bigr]
+
\mathbb E\left[
\frac{1}{N}\sum_{i\notin \mathcal S_t}
\frac{\lambda}{2}
\left(
\bigl\|\bm X_i^{(t-1)}-\bm X_g^{(t)}\bigr\|_{\mathrm F}^{2}
-
\bigl\|\bm X_i^{(t-1)}-\bm X_g^{(t-1)}\bigr\|_{\mathrm F}^{2}
\right)
\right].
\label{eq:partial_difference}
\end{align}

Let $\Delta_g^{(t)}=\bm X_g^{(t)}-\bm X_g^{(t-1)}$. Then
\begin{align}
\frac{\lambda}{2}\left(\|\bm X_i^{(t-1)}-\bm X_g^{(t)}\|_{\mathrm F}^2-\|\bm X_i^{(t-1)}-\bm X_g^{(t-1)}\|_{\mathrm F}^2\right)
&=
\frac{\lambda}{2}\left(-2\langle \bm X_i^{(t-1)}-\bm X_g^{(t-1)},\Delta_g^{(t)}\rangle+\|\Delta_g^{(t)}\|_{\mathrm F}^2\right) \nonumber\\
&\le
\frac{\lambda}{2}\left(\|\bm X_i^{(t-1)}-\bm X_g^{(t-1)}\|_{\mathrm F}^2+\|\Delta_g^{(t)}\|_{\mathrm F}^2+\|\Delta_g^{(t)}\|_{\mathrm F}^2\right) \nonumber\\
&=
\frac{\lambda}{2}\|\bm X_i^{(t-1)}-\bm X_g^{(t-1)}\|_{\mathrm F}^2+\lambda\|\Delta_g^{(t)}\|_{\mathrm F}^2,
\end{align}
where the inequality is the elementary bound $-2\langle \bm u,\bm v\rangle\le\|\bm u\|_{\mathrm F}^2+\|\bm v\|_{\mathrm F}^2$.

Since truncating singular values cannot increase the Frobenius norm,
\begin{equation}
\|\bm X_g^{(t)}\|_{\mathrm F}
\le
\|\hat{\bm X}^{(t)}\|_{\mathrm F}
\le
\frac{1}{K}\sum_{i\in\mathcal S_t}\|\bm X_i^{(t)}\|_{\mathrm F}
\le
C_X,
\end{equation}
and the same bound holds for $\bm X_g^{(t-1)}$; at initialization, $\|\bm X_g^{(0)}\|_{\mathrm F}\le C_X$ because all clients start from the same rank-$r$ factors. Hence,
\begin{equation}
\|\Delta_g^{(t)}\|_{\mathrm F}
\le
\|\bm X_g^{(t)}\|_{\mathrm F}
+
\|\bm X_g^{(t-1)}\|_{\mathrm F}
\le
2C_X,
\qquad
\lambda\|\Delta_g^{(t)}\|_{\mathrm F}^2
\le
4\lambda C_X^2.
\end{equation}

Therefore, the reference-shift term for each non-participating client is bounded by
\begin{equation}
\frac{\lambda}{2}\left\|\bm X_i^{(t-1)}-\bm X_g^{(t-1)}\right\|_{\mathrm F}^{2}
+
4\lambda C_X^2.
\end{equation}

Averaging over the $N-K$ non-participating clients and taking expectations, the last term of Eq.~\eqref{eq:partial_difference} is bounded by
\begin{equation}
\mathbb E\left[
\frac{1}{N}\sum_{i\notin \mathcal S_t}
\frac{\lambda}{2}\left\|\bm X_i^{(t-1)}-\bm X_g^{(t-1)}\right\|_{\mathrm F}^{2}
\right]
+
\left(1-\frac{K}{N}\right)4\lambda C_X^2.
\end{equation}

By uniform sampling, the first expectation equals $\left(1-\frac{K}{N}\right)\mathbb E[R^{(t-1)}]$, where $R^{(t-1)}=\frac{1}{N}\sum_{i=1}^{N}\frac{\lambda}{2}\|\bm X_i^{(t-1)}-\bm X_g^{(t-1)}\|_{\mathrm F}^{2}$ denotes the regularization component of $\Psi^{(t-1)}$. Since $R^{(t-1)}\leq2\lambda C_X^2$, the last term of Eq.~\eqref{eq:partial_difference} is bounded by
\begin{equation}
\left(1-\frac{K}{N}\right)2\lambda C_X^2
+
\left(1-\frac{K}{N}\right)4\lambda C_X^2
=
\left(1-\frac{K}{N}\right)6\lambda C_X^2.
\end{equation}

Substituting Eq.~\eqref{eq:sampled_descent_expectation} into Eq.~\eqref{eq:partial_difference}, using $\mathbb E[\Psi_{\mathcal S_t}^{(t-1)}]=\mathbb E[\Psi^{(t-1)}]$ by uniform sampling, and applying the reference-shift bound above, we obtain
\begin{equation}
\begin{aligned}
\mathbb E[\Psi^{(t)}]
\leq{}& \mathbb E[\Psi^{(t-1)}]
-\frac{\eta K}{N^2}\sum_{i=1}^{N}\sum_{k=0}^{\tau-1}
\mathbb E\bigl[\|\nabla_{\bm B}\mathcal L_i^{(t)}(\bm W_i^{(t,k)})\|_{\mathrm F}^{2}
+\|\nabla_{\bm A}\mathcal L_i^{(t)}(\bm W_i^{(t,k)})\|_{\mathrm F}^{2}\bigr]\\
&+\frac{K}{N}\tau M_\lambda\eta^2
+\left(1-\frac{K}{N}\right)6\lambda C_X^2.
\end{aligned}
\end{equation}

Since the regularization terms are nonnegative, $\Psi^{(T)}\geq N^{-1}\sum_{i=1}^{N} f_i(\bm W_i^{(T)})$. By Assumption~\ref{assumption:conv_lipschitz}, $f_i(\bm W_i^{(T)})\geq f_i(\bm W_0)-L_c\|\bm W_i^{(T)}-\bm W_0\|_{\mathrm F}$. Using $\bm W_i^{(T)}-\bm W_0=\bm X_i^{(T)}$ and $\|\bm X_i^{(T)}\|_{\mathrm F}\leq C_A C_B$, we obtain $\Psi^{(T)}\geq\Psi^\star-L_c C_A C_B$. Hence, $\mathbb E[\Psi^{(0)}]-\mathbb E[\Psi^{(T)}]\leq D+L_c C_A C_B=\widetilde D$. Summing the preceding descent inequality over $t=1,\ldots,T$ and dividing by $\eta KT\tau/N$ yields Eq.~\eqref{eq:partial_bound}. This completes the proof.
\end{proof}

When $K=N$, the additional term vanishes, recovering the full-participation bound. For $K<N$, the bound includes an additional participation-dependent error term of $(N-K)B/(\eta K\tau)$.

\section{Complexity Analysis}
\label{app:complexity}

We analyze the per-layer server-side complexity of FedPA-LoRA and representative federated LoRA methods under homogeneous ranks, with $N$ clients, a $d\times d$ weight matrix, $r=R_g$, and $r_{\mathrm{tot}}=Nr\ll d$. Memory assumes all client uploads are retained during aggregation. We also report client-side wall-clock measurements covering local optimization and method-specific operations.

\subsection{Server-Side Complexity Analysis}
\label{app:server_complexity}

\subsubsection{Asymptotic Complexity}
\label{app:server_asymptotic_complexity}

\begin{table}[t]
\centering
\caption{Dominant per-layer server-side computation and memory complexities under the homogeneous-rank setting with $Nr\ll d$. The randomized variant assumes a sketch dimension $\ell=R_g+p=\mathcal O(r)$.}
\label{tab:server_complexity}
{\small
\setlength{\tabcolsep}{3.5pt}
\renewcommand{\arraystretch}{1.15}
\begin{tabular}{@{}lllcc@{}}
\toprule
\textbf{Method}
& \textbf{Client Upload}
& \textbf{Server Operation}
& \textbf{Computation}
& \textbf{Memory} \\
\midrule

FedIT
& $\bm B_i,\bm A_i$
& Two-Factor Averaging
& $\mathcal{O}(Ndr)$
& $\mathcal{O}(Ndr)$ \\

FedRot-LoRA
& $\widetilde{\bm B}_i,\widetilde{\bm A}_i$
& Two-Factor Averaging
& $\mathcal{O}(Ndr)$
& $\mathcal{O}(Ndr)$ \\

FedDPA-LoRA
& $\bm B_{g,i},\bm A_{g,i}$
& Two-Factor Averaging
& $\mathcal{O}(Ndr)$
& $\mathcal{O}(Ndr)$ \\

FFA-LoRA
& $\bm B_i$
& Single-Factor Averaging
& $\mathcal{O}(Ndr)$
& $\mathcal{O}(Ndr)$ \\

RoLoRA
& $\bm B_i$ or $\bm A_i$
& Single-Factor Averaging
& $\mathcal{O}(Ndr)$
& $\mathcal{O}(Ndr)$ \\

FedSA-LoRA
& $\bm A_i$
& Single-Factor Averaging
& $\mathcal{O}(Ndr)$
& $\mathcal{O}(Ndr)$ \\

\midrule

FlexLoRA
& $\bm B_i,\bm A_i$
& Dense Product Aggregation and SVD
& $\mathcal{O}(d^3)$
& $\mathcal{O}(d^2)$ \\

FedPA-LoRA
& $\bm B_i,\bm A_i$
& Reduced QR and Core SVD
& $\mathcal{O}(N^2dr^2)$
& $\mathcal{O}(Ndr)$ \\

FedPA-LoRA (Randomized)
& $\bm B_i,\bm A_i$
& Factored Randomized Reconstruction
& $\mathcal{O}(Ndr^2)$
& $\mathcal{O}(Ndr)$ \\

\bottomrule
\end{tabular}
}
\end{table}

Table~\ref{tab:server_complexity} compares the dominant per-layer server-side computation and memory complexities, including FedPA-LoRA with exact reduced-QR reconstruction and randomized reconstruction with sketch dimension $\ell=R_g+p$.

FedIT, FedRot-LoRA, and FedDPA-LoRA average both LoRA factors, whereas FFA-LoRA, FedSA-LoRA, and RoLoRA aggregate a single factor per round. These factor-wise methods require $\mathcal{O}(Ndr)$ server-side computation. FlexLoRA instead constructs the dense $d\times d$ product-space aggregate and performs SVD, requiring $\mathcal{O}(Nd^2r+d^3)$ computation and $\mathcal{O}(d^2+Ndr)$ memory. FedPA-LoRA avoids the dense aggregate by applying reduced QR factorizations to the $d\times Nr$ concatenated factors and performing SVD only on the resulting $Nr\times Nr$ core matrix. Its full computation and memory complexities are $\mathcal{O}(N^2dr^2+N^3r^3+Ndr^2)$ and $\mathcal{O}(Ndr+N^2r^2)$, respectively. Under the typical low-rank regime $Nr\ll d$, the dominant computation and memory costs are $\mathcal{O}(d^3)$ and $\mathcal{O}(d^2)$ for FlexLoRA, compared with $\mathcal{O}(N^2dr^2)$ and $\mathcal{O}(Ndr)$ for FedPA-LoRA. Despite avoiding dense reconstruction, FedPA-LoRA recovers the same optimal rank-$r$ approximation as truncated SVD of the dense product-space aggregate.

The exact reconstruction retains a quadratic dependence on $N$ through $r_{\mathrm{tot}}=Nr$. We therefore consider a randomized extension with sketch dimension $\ell=R_g+p$, which directly accumulates the sketch and projected matrix from the client factors without additional concatenated intermediates. Its computation and memory complexities are $\mathcal O(Ndr\ell+d\ell^2)$ and $\mathcal O(Ndr+d\ell)$, respectively. For $\ell=\mathcal O(r)$, the computation reduces from $\mathcal O(N^2dr^2)$ to $\mathcal O(Ndr^2)$, while the dominant memory remains $\mathcal O(Ndr)$ with lower intermediate overhead.

\subsubsection{Randomized Approximation}
\label{app:randomized_server_reconstruction}

The randomized variant uses a sketch of dimension $\ell=R_g+p$ to reduce the quadratic dependence on the number of clients $N$. Specifically, it samples $\bm\Omega^{(t)}\in\mathbb R^{d\times\ell}$ and directly accumulates $\bm Y^{(t)}=\frac{1}{N}\sum_{i=1}^{N}\bm B_i^{(t)}(\bm A_i^{(t)}\bm\Omega^{(t)})$, avoiding the dense product aggregate and additional concatenated factors.

\begin{remark}[Randomized Approximation Guarantee]
\label{remark:randomized_reconstruction}

When $\ell=R_g+p=\mathcal O(r)$, the randomized reconstruction avoids both the dense $d\times d$ aggregate and the full $r_{\mathrm{tot}}$-dimensional basis. Its dominant computation is reduced from $\mathcal O(N^2dr^2)$ to $\mathcal O(Ndr^2)$, at the cost of replacing the exact optimality guarantee in Eq.~\eqref{eq:optimal_rank_projection} with a randomized approximation guarantee. Applying the randomized range-finder bound of \citet[Theorem 1.1]{halko2011finding} followed by rank-$R_g$ truncation gives
\begin{equation}
\mathbb{E}\left\|\bm B_g^{(t)}\bm A_g^{(t)}-\Delta\bm W_{\mathrm{ideal}}^{(t)}\right\|_2 \leq \left[2+4\sqrt{\frac{R_g+p}{p-1}}\sqrt d\right]\sigma_{R_g+1}\left(\Delta\bm W_{\mathrm{ideal}}^{(t)}\right).
\notag
\end{equation}
Here, $\sigma_{R_g+1}(\cdot)$ denotes the $(R_g+1)$-th singular value, and the bound holds for $p\geq2$.

When $\ell=R_g+p=r_{\mathrm{tot}}$, the sketch captures the column space of $\Delta\bm W_{\mathrm{ideal}}^{(t)}$ almost surely in exact arithmetic. The resulting reconstruction therefore recovers the optimal rank-$R_g$ approximation obtained by the reduced-QR and core-SVD procedure, while its dominant computation becomes $\mathcal O(dr_{\mathrm{tot}}^2)=\mathcal O(N^2dr^2)$. Thus, the sketch dimension $\ell$ controls the trade-off between computational cost and approximation accuracy.

\end{remark}

\subsubsection{Wall-Clock Evaluation}
\label{app:server_wallclock_evaluation}

\begin{table}[t]
\centering
\caption{Average per-round server-side aggregation wall-clock time under the homogeneous-rank setting. For randomized reconstruction, the sketch dimension is $\ell=R_g+p$. RoBERTa-Large uses $N=10$ clients with $r=R_g=4$ and $r_{\mathrm{tot}}=Nr=40$, whereas Llama 3-8B uses $N=6$ clients with $r=R_g=8$ and $r_{\mathrm{tot}}=48$.}
\label{tab:server_wallclock}
{\small
\setlength{\tabcolsep}{7pt}
\renewcommand{\arraystretch}{1.15}
\begin{tabular}{@{}lcccc@{}}
\toprule
\multirow{2}{*}{\textbf{Server Operation}}
& \multicolumn{2}{c}{\textbf{RoBERTa-Large}}
& \multicolumn{2}{c}{\textbf{Llama 3-8B}} \\
\cmidrule(lr){2-3}
\cmidrule(l){4-5}
& \textbf{Sketch Dim.} $\bm{\ell}$
& \textbf{Time (s)}
& \textbf{Sketch Dim.} $\bm{\ell}$
& \textbf{Time (s)} \\
\midrule

Two-Factor Averaging
& -- & $0.0298$
& -- & $0.0199$ \\

Single-Factor Averaging
& -- & $0.0206$
& -- & $0.0109$ \\

Dense Product Aggregation and SVD
& -- & $4.7340$
& -- & $152.3957$ \\

Reduced QR and Core SVD
& -- & $0.0840$
& -- & $0.1169$ \\

\midrule

Randomized Reconstruction ($p=0$)
& $4$ & $0.0473$
& $8$ & $0.0596$ \\

Randomized Reconstruction ($p=8$)
& $12$ & $0.0493$
& $16$ & $0.0638$ \\

Randomized Reconstruction ($p=16$)
& $20$ & $0.0529$
& $24$ & $0.0706$ \\

Randomized Reconstruction ($p=32$)
& $36$ & $0.0615$
& $40$ & $0.0901$ \\

\bottomrule
\end{tabular}
}
\end{table}

To complement the asymptotic analysis, we measure the average server-side aggregation time per round for RoBERTa-Large and Llama 3-8B under homogeneous ranks. RoBERTa-Large uses $N=10$ and $R_g=4$, while Llama 3-8B uses $N=6$ and $R_g=8$, with all methods evaluated in the same environment. As shown in Table~\ref{tab:server_wallclock}, factor-wise averaging incurs the lowest server-side cost, whereas dense product aggregation and SVD become substantially more expensive as the model dimension increases. FedPA-LoRA avoids dense $d\times d$ decomposition by performing SVD only on the smaller core matrix. Randomized reconstruction further reduces the server-side cost, with wall-clock time generally increasing as the sketch dimension $\ell$ approaches the full concatenated rank $r_{\mathrm{tot}}$.

We additionally include $p=0$ as an empirical minimum-sketch configuration. Although randomized reconstruction remains computationally well defined in this setting, the expectation bound in Remark~\ref{remark:randomized_reconstruction} does not apply because it requires $p\geq2$.

\subsection{Client-Side Complexity Analysis}
\label{app:client_wallclock_evaluation}

\begin{table}[t]
\centering
\caption{Average client-side wall-clock time in seconds, first averaged across participating clients within each communication round and then across rounds. The first round is excluded to avoid one-time initialization overhead.}
\label{tab:client_wallclock}
{\small
\setlength{\tabcolsep}{4pt}
\renewcommand{\arraystretch}{1.15}
\begin{tabular}{@{}llcc@{}}
\toprule
\textbf{Local Update Strategy}
& \textbf{Representative Methods}
& \textbf{RoBERTa-Large}
& \textbf{Llama 3-8B} \\
\midrule

Single-Factor LoRA Update
& FFA-LoRA, RoLoRA
& $9.2586$
& $3.2965$ \\

Full LoRA Update
& FedIT, FlexLoRA, FedSA-LoRA
& $9.3373$
& $3.4130$ \\

LoRA Update + Alignment
& FedRot-LoRA
& $9.3876$
& $3.7048$ \\

Global and Local LoRA Update
& FedDPA-LoRA
& $10.3189$
& $3.6641$ \\

\midrule

Product-Regularized LoRA Update
& FedPA-LoRA
& $9.5431$
& $4.5183$ \\

Efficient Product-Regularized LoRA Update
& FedPA-LoRA
& $9.3678$
& $3.6205$ \\

\bottomrule
\end{tabular}
}
\end{table}

\subsubsection{Asymptotic Complexity}

Direct evaluation of the product regularizer requires forming the dense $d\times d$ products $\bm B_i\bm A_i$ and $\bm B_g\bm A_g$. Instead, we use
\begin{equation}
\left\|\bm B_i\bm A_i-\bm B_g\bm A_g\right\|_{\mathrm F}^{2} = \left\langle\bm B_i^\top\bm B_i,\bm A_i\bm A_i^\top\right\rangle_{\mathrm F} + \left\langle\bm B_g^\top\bm B_g,\bm A_g\bm A_g^\top\right\rangle_{\mathrm F} - 2\left\langle\bm B_i^\top\bm B_g,\bm A_i\bm A_g^\top\right\rangle_{\mathrm F}.
\label{eq:efficient_product_regularization}
\end{equation}
This reformulation is exactly equivalent to the original regularizer and operates only on $r\times r$ matrices. The second term depends only on the global reference and can therefore be cached once per communication round, while the remaining terms are computed from the current local factors. This reduces the per-step computation from $\mathcal{O}(d^2r)$ to $\mathcal{O}(dr^2)$ and the intermediate memory from $\mathcal{O}(d^2)$ to $\mathcal{O}(r^2)$, without forming the dense $d\times d$ products.

\subsubsection{Wall-Clock Evaluation}

We measure the client-side wall-clock time of representative local update strategies on RoBERTa-Large and Llama 3-8B. The time is first averaged across participating clients within each communication round and then across rounds. We compare single-factor and full LoRA updates, FedRot-LoRA with alignment, FedDPA-LoRA with dual-adapter optimization, and the original and efficient product-regularized variants of FedPA-LoRA.

As shown in Table~\ref{tab:client_wallclock}, single-factor LoRA updates incur the lowest client-side cost, while FedRot-LoRA and FedDPA-LoRA require additional alignment or dual-adapter optimization. The original product regularizer becomes more costly as model size increases due to dense product computations. In contrast, the efficient formulation remains close to a standard full LoRA update and faster than FedRot-LoRA and FedDPA-LoRA on both models. Being mathematically equivalent to the original regularizer, it retains the same product-space objective with substantially lower computation and intermediate memory.

\section{Adaptive Regularization Strength}
\label{app:adaptive_regularization}

The regularization weight $\lambda$ in Eq.~\eqref{eq:product_guided_objective} is shared across clients and communication rounds, and is selected offline through grid search. However, Figure~\ref{fig:lambda_results} suggests that an appropriate value may vary with data heterogeneity. Near-IID settings may require weaker regularization than highly non-IID settings, whereas stronger regularization may be needed under non-IID data to control client drift. This motivates client- and round-specific weights $\lambda_i^{(t)}$ determined from quantities available locally at each client. We discuss three possible directions as illustrative alternatives rather than a definitive adaptive rule.

\subsection{Task-loss-based adaptation}
The local task loss reflects how strongly a client still needs to adapt to its data. One possible rule is
\begin{equation}
\lambda_{i,\mathrm{loss}}^{(t)}
=
\frac{
\lambda_0
}{
1+
\alpha
f_i\left(
\bm W_i^{(t,0)};
\xi_i^{(t,0)}
\right)
},
\label{eq:loss_based_lambda}
\end{equation}
where $\xi_i^{(t,0)}$ denotes the first mini-batch sampled at client $i$ in round $t$, $\lambda_0>0$ is a base regularization weight and $\alpha>0$ controls its sensitivity to the initial local loss. Since the task loss is nonnegative, the resulting weight naturally satisfies $0 < \lambda_{i,\mathrm{loss}}^{(t)} \leq \lambda_0$. A large task loss weakens the global pull and prioritizes local adaptation, whereas a smaller loss restores stronger drift control. This quantity is readily available during training, although its behavior may depend on the loss scale and the sampled mini-batch.

\subsection{Gradient-aware drift adaptation}
The task-gradient magnitude provides a direct estimate of local update strength. In LoRA, the trainable factors $\bm B_i$ and $\bm A_i$ satisfy $\nabla_{\bm B_i}f_i=\nabla_{\bm W}f_i\bm A_i^\top$ and $\nabla_{\bm A_i}f_i=\bm B_i^\top\nabla_{\bm W}f_i$ on the initial mini-batch $\xi_i^{(t,0)}$. Since these gradients are scaled by the complementary factors, we use the normalized surrogate
\begin{equation}
\widehat g_{f,i}^{(t)} = \left[\frac{\left\|\nabla_{\bm B_i}f_i\left(\bm W_i^{(t,0)};\xi_i^{(t,0)}\right)\right\|_{\mathrm F}^{2}}{\left(\|\bm A_i^{(t,0)}\|_{\mathrm F}+\varepsilon\right)^{2}}+\frac{\left\|\nabla_{\bm A_i}f_i\left(\bm W_i^{(t,0)};\xi_i^{(t,0)}\right)\right\|_{\mathrm F}^{2}}{\left(\|\bm B_i^{(t,0)}\|_{\mathrm F}+\varepsilon\right)^{2}}\right]^{1/2}.
\label{eq:grad_surrogate}
\end{equation}
Here, $\varepsilon>0$ is a small constant for numerical stability. Combining this estimate with the current product mismatch yields
\begin{equation}
\lambda_{i,\mathrm{grad}}^{(t)} = \operatorname{clip}\!\left(\alpha\frac{\widehat g_{f,i}^{(t)}}{\left\|\bm X_i^{(t-1)}-\bm X_{g,i}^{(t-1)}\right\|_{\mathrm F}+\varepsilon},\lambda_{\min},\lambda_{\max}\right).
\label{eq:grad_based_lambda}
\end{equation}
Here, $\bm X_i^{(t-1)}=\bm B_i^{(t-1)}\bm A_i^{(t-1)}$, $\bm X_{g,i}^{(t-1)}$ denotes the rank-$R_i$ global reference product in Eq.~\eqref{eq:product_guided_objective}, $\alpha>0$ is a scaling factor, and $\lambda_{\min}$ and $\lambda_{\max}$ define the clipping range, with $\operatorname{clip}(x,a,b)=\min(\max(x,a),b)$. This rule adjusts the regularization strength according to the estimated local update magnitude relative to the product-space deviation from the global reference, although it may be sensitive to small factor norms and mini-batch variability.

\subsection{Gradient-direction-based adaptation}
A third possibility is to consider the directional relationship between the task and regularization gradients. Let $\bm g_{f,i}^{(t)}$ and $\bm g_{r,i}^{(t)}$ denote the task and product-regularization factor gradients, respectively, concatenated across all LoRA layers after normalization by the corresponding complementary factor norms. We define
\begin{equation}
\lambda_{i,\mathrm{dir}}^{(t)}
=
\lambda_0
\left(
1-
\frac{
\left\langle
\bm g_{f,i}^{(t)},
\bm g_{r,i}^{(t)}
\right\rangle
}{
\left\|
\bm g_{f,i}^{(t)}
\right\|_2
\left\|
\bm g_{r,i}^{(t)}
\right\|_2
+
\varepsilon
}
\right).
\label{eq:direction_based_lambda}
\end{equation}
The normalized directional similarity measures the agreement between task optimization and product-space regularization. Greater agreement leads to a smaller regularization weight, whereas weaker or opposing directions increase the weight to place greater emphasis on limiting client drift. The stabilizing term $\varepsilon$ prevents division by zero and keeps the resulting weight within $(0,2\lambda_0)$.

Overall, these alternatives provide illustrative directions for adaptive regularization rather than a definitive rule. They introduce global hyperparameters such as $\lambda_0$ and $\alpha$ while adapting the effective regularization strength to the local client state, including task loss and gradient information. Such client-specific adaptation may reduce the need for heterogeneity-specific tuning. In our preliminary exploration, we use fixed constants $\lambda_0=1$, $\alpha=1$, $\lambda_{\min}=0$, and $\lambda_{\max}=10$ across heterogeneity levels, with the clipping bounds serving as safety limits. Whether these settings transfer robustly across different heterogeneous conditions remains an empirical question for future work.

\end{document}